%% file: achddou21-supp.tex
\documentclass[wcp]{jmlr}

\usepackage{longtable}

\usepackage{booktabs}

\makeatletter
\let\Ginclude@graphics\@org@Ginclude@graphics 
\makeatother

\jmlrvolume{157}
\jmlryear{2021}
\jmlrworkshop{ACML 2021}

\usepackage{amssymb}
\usepackage{bbold}
\usepackage{booktabs} 
\usepackage{tikz} 
\usepackage{todonotes}
\usepackage{comment}

\usepackage{mathtools, bm}
\usepackage{amssymb, bm}
\hypersetup{breaklinks=true}

\newcommand{\1}{\mathbb{1}} 

\newcommand{\Po}{\mathbb{P}}

\DeclareMathOperator*{\argmax}{arg\,max}

\newcommand{\E}{\mathbb{E}} 
\def\Fn{\hat{F}_t}

\newtheorem{assumption}{Assumption}

\newtheorem{algo}{Algorithm}

\makeatletter
\newtheorem{repeatthm@}{Theorem}
\newenvironment{repeatthm}[1]{%
    \def\therepeatthm@{\ref{#1}}
    \repeatthm@
}
{\endrepeatthm@}
\makeatother

\makeatletter
\newtheorem{repeatlem@}{Lemma}
\newenvironment{repeatlem}[1]{%
    \def\therepeatlem@{\ref{#1}}
    \repeatlem@
}
{\endrepeatlem@}
\makeatother
\title[Fast Rate Learning in Stochastic First Price Bidding]{Fast Rate Learning in Stochastic First Price Bidding}

  \author{\Name{Juliette Achddou} \Email{juliette.achdou@gmail.com}\\
  \addr DIENS, INRIA, Université PSL, 1000mercis Group
  \AND
  \Name{Olivier Cappé} \Email{olivier.cappe@cnrs.fr}\\
  \addr DIENS, CNRS, INRIA, Université PSL, 
  \AND
  \Name{Aurélien Garivier}
  \Email{aurelien.garivier@ens-lyon.fr}\\
  \addr{UMPA,CNRS, INRIA, ENS Lyon}
 }

\editors{Vineeth N Balasubramanian and Ivor Tsang}

\begin{document}

\maketitle

\begin{abstract}
  First-price auctions have largely replaced traditional bidding
  approaches based on Vickrey auctions in programmatic advertising.
  As far as learning is concerned, first-price auctions are more
  challenging because the optimal bidding strategy does not only
  depend on the value of the item but also requires some knowledge of
  the other bids.
  They have already given rise to several works in sequential
  learning, %
  many of which consider models for which the value of the buyer or
  the opponents' maximal bid is chosen in an adversarial manner. Even
  in the simplest settings, this gives rise to algorithms whose regret
  grows as $\sqrt{T}$ with respect to the time horizon $T$.
  Focusing on the case where the buyer plays against a stationary
  stochastic environment, we show how to achieve significantly lower
  regret: when the opponents' maximal bid distribution is known we
  provide an algorithm whose regret can be as low as $\log^2(T)$; in
  the case where the distribution must be learnt sequentially, a
  generalization of this algorithm can achieve $T^{1/3+ \epsilon}$
  regret, for any $\epsilon>0$.
  To obtain these results, we introduce two novel ideas that can be of
  interest in their own right. First, by transposing results obtained
  in the posted price setting, we provide conditions under which the
  first-price bidding utility is locally quadratic around its
  optimum. Second, we leverage the observation that, on small
  sub-intervals, the concentration of the variations of the empirical
  distribution function may be controlled more accurately than by
  using the classical Dvoretzky-Kiefer-Wolfowitz inequality.
  Numerical simulations confirm that our algorithms converge much
  faster than alternatives proposed in the literature for various bid
  distributions, including for bids collected on an actual
  programmatic advertising platform.
\end{abstract}





\author{author names withheld}

\editors{Under Review for ACML 2021}

\maketitle

\section{Introduction}\label{sec:intro}
We consider the problem of setting a bid in repeated first-price
auctions.  First-price auctions are widely used in practice, partly
because they constitute the most natural and simple type of
auctions. In particular, they have been largely adopted in the field
of programmatic advertising, where they have progressively replaced
second-price auctions \citep{Slu17,slefo}.  This recent transition
took place for various reasons.  First, whereas second-price auctions
have the advantage of being dominant-strategy incentive-compatible and
hence allow for simple bidding strategies
\citep{vickrey1961counterspeculation}, they were made obsolete by the
widespread use of \textit{header bidding}
, a technology that puts different ad-exchange platforms in
competition. With this technology, every participating ad-exchange has
to provide the winning bid of the auction organized on its platform; a
second-level auction is then organized between all the winners to
determine which bidder earns the right of displaying its
banner. Second price auctions would hence jeopardize the fairness of
the attribution of the placement at sale with header bidding.
Second, sellers have benefited from the transition, since many bidders
continued to bid as in second-price auctions and despite the automated
implementation of so-called \textit{bid shading} by demand-side
platforms, meant to adjust their bids to this new situation
\citep{Slu19}.  The transition to first price auctions raises
questions for advertisers who need new bidding strategies.  In
general, bidders participating in auctions in the context of
programmatic advertising do not know the bidding strategies of the
other contestants in advance, or anything about the valuations that
other bidders attribute to the advertisement slot.  Not only do they
have to learn other bidders' behavior on the go, but they also need to
understand how valuable the placement is for their own use (how many
clicks or actions the display of their ad on this placement will lead
to), which is usually not the same for all bidders.

In this work, we model the problem faced by a single bidder in
repeated stochastic first-price auctions, that is, when the
contestants' bids are drawn from a stationary distribution. We
consider that the learner's bids will not influence the others'
bidding strategies.  This approximation is sensible in contexts where
the major part of the stakeholders do not have an elaborate bidding
strategy.  More precisely, many stakeholders never modify their bids
or do so at a very low frequency.  Moreover, the poll of bidders is
very large and each bidder only participates in a fraction of the
auctions, 
which argues in favor of the assumption that the influence of one
bidder on the rest of the participants can be neglected.

\paragraph{Model} We consider that similar items are sold in $T$
sequential first price auctions. For $t=1,\ldots,T$, the auction
mechanism unfolds in the following way. First, the bidder submits her
bid $B_t$ for the item that is of unknown value $V_t$. The other
players submit their bids, the maximum of which is called $M_t$.  If
$M_t\leq B_t$ (which includes the case of ties), the bidder observes
and receives $V_t$ and pays $B_t$.  If $B_t< M_t$, the bidder loses
the auction and does not observe $V_t$.

We make the following additional assumptions: $\{V_t\}_{t\geq 1}$ are
independent and identically distributed random variables in the unit
interval $[0,1]$; their expectation is denoted by
$v:=\mathbb{E}(V_t)$.  The $\{M_t\} _{t\geq 1}$ are independent and
identically distributed random variables in the unit interval $[0,1]$
with a cumulative distribution function (CDF) $F$, independent from the
$\{V_t\}_{t\geq 1}$. When applicable, we denote by $f=F'$ the associated
probability density function.

Due to the stochastic nature of the setting, we study the first-price utility of the bidder:
$U_{v,F}(b) := \E\big[(V_t- b)\1\{M_t\leq b\}\big] = (v-b)F(b)$. The
(pseudo-)regret is defined as
\begin{align*}
R_T^{v,F} &= T \max_{b \in [0,1]} U_{v,F}(b)- \sum_{t=1}^T \E[U_{v,F}(B_t)]\;.
\end{align*}
We denote by
$b^*_{v,F} = \max\big\{\argmax_{b\in[0,1]}U_{v,f}(b)\big\}$ the
(highest) optimal bid. In the rest of the paper, we will abuse notation and speak about regret although rigorously this quantity should be termed pseudo-regret. Note that the outer max is required as the
utility may have multiple maxima (see Section~\ref{sec:properties}
below): in that case, we define the optimal bid as the one that has
the largest winning rate. In the sequel, we exclude the particular
case where $F(b^*_{v,F})=0$, since in this hopeless situation the
contestants always bid above the value of the item and the best
strategy is not to bid at all ($B_t\equiv 0$): we thus assume that
$F(b^*_{v,F})>0.$

In Section~\ref{sec:known_F}, we will first assume that $F$ is known
to the learner.  This setting bears some similarities with the case of
second-price auctions considered by
\citep{weed2016online,achddou2021efficient}: the truthfulness of
second-price auctions makes it sufficient for the bidder to learn the
value of $v$ and the valuation of the item is the only parameter to
estimate in that case. However, an important feature of the
second-price auction mechanism is that the utility of the bidder is
quadratic in $v$ under very mild assumptions on the bidding
distribution $F$. In the case of first-price auctions, the utility is
no longer guaranteed to be unimodal, neither is the optimal bid
$b^*_{v,F}$ a regular function of $v$.

We treat the case, in Section~\ref{sec:unknown_F}, where the CDF $F$
of the opponents' maximal bid is initially unknown to the learner,
assuming that the maximal bid $M_t$ is observed for each auction. Note
that in this more realistic setting, the bidder could not infer the
optimal bid $b^*_{v,F}$ even if she had perfect knowledge of the item
value $v$. The bidder consequently needs to estimate $F$ and $v$
simultaneously, which makes it a clearly harder task. This second
setting bears some similarities with the task of fixing a price in the
posted price problem~\citep{huang2018making, kleinberg2003value,
  bubeck2017multi, cesa2019dynamic}, in which a seller needs to
estimate the distribution of the valuations of buyers, in order to set
the optimal price in terms of her revenue. However, in contrast to the
posted-price setting, there is an additional unknown parameter $v$
that also impacts the utility function.

In both of these settings, the learner is faced with a structured
continuously-armed bandit problem with censored feedback. Indeed, the
bidder only observes the reward associated with the chosen bid, but
she observes the value only when she wins.  This introduces a specific
exploitation/exploration dilemma, where exploitation is achieved by
bidding close to one of the optimal bids but exploration requires that
the bids are not set too low. This structure seems to call for
algorithms that bid above the optimal bid with high probability, as in
\citep{weed2016online,achddou2021efficient} for the second-price case,
but we will see in the following that it is not necessarily
true. 

\paragraph{Related Works}
A major line of research in the field of online learning in repeated
auctions is devoted to fixing a reserve price for second-price
auctions or a selling price in posted price auctions, see
\citep{nedelec2020learning} for a general survey.
In the posted price setting, arbitrarily bad distributions of bids
give rise to very hard optimization problems
\citep{roughgarden2016ironing}.  That is why regularity assumptions
are often used, like e.g. the \textit{monotonic hazard rate} (MHR)
condition.  Most notably, \cite{huang2018making, cole2014sample,
  dhangwatnotai2015revenue} use this assumption to bound the sample
complexity of finding the monopoly price.  Regarding online learning
in the posted price
setting, 
\cite{kleinberg2003value} and \cite{cesa2019dynamic} introduce
algorithms for the stochastic case, respectively in the cases where the
distribution of the prices are continuous and discrete.
\cite{bubeck2017multi} study the adversarial
counterpart. \cite{blum2004online, cesa2014regret} study online
strategies that aim at setting the optimal reserve price in
second-price auctions while learning the distribution of the buyer's
bids. \cite{cesa2014regret} assume that bidders are symmetric, but
that the bids distribution is not necessarily MHR. They introduce an
optimistic algorithm based on two ideas. Firstly they observe that
exploitation is achieved by submitting a price smaller than the
optimal reserve price, and secondly they use the fact that the utility
can be bounded in infinite norm, thanks to the
Dvoretzky-Kiefer-Wolfowitz (DKW) inequality \citep{massart1990tight}.

The problem of learning in repeated auctions from the point of view of
the buyer was originally addressed in the setting of second-price
auctions. For the stochastic setting, \cite{weed2016online} propose an
algorithm that overbids with high probability, and that is shown to
have a regret of the order of $\log^2{T}$ under mild assumptions on
the distribution of the bids. They also provide algorithms for the
adversarial case, that have a regret scaling in
$\sqrt{T}$. \cite{achddou2021efficient} extend their work by proposing
tighter optimistic strategies that show better worst case
performances. They also analyze non-overbidding strategies, proving
that such strategies can perform well on a large class of second-price
auctions instances. \cite{flajolet2017real} consider the contextual
set-up where the value associated to an item is linear with respect to
a context vector associated to the item, and revealed before each
action.

Learning in repeated stochastic first price auctions is a difficult
problem that has given rise to a number of very different though
equally interesting modelizations. 
\cite{feng2020convergence} consider auctions in which the values of
all the bidders are revealed as a context before each turn, proving
that the bids of bidders who use no regret contextual learning
strategies in first price auctions converge to Bayes Nash
equilibria. \cite{han2020optimal} also consider the case where the
values are assumed to be revealed as an element of context before each
auction takes place and the highest bid among others' bids
is only shown to the learner when she loses.  This setting
interestingly introduces a censoring structure that is opposed to the
one we consider: in this context, exploitation is achieved by not
bidding too high. \cite{han2020optimal} provide new algorithms for
this setting which have a regret of the order of $\sqrt{T}$.  A
setting somewhat closer to ours is studied by
\cite{feng2018learning}. This work deals with the setting of a bid in
an adversarial fashion, when the other bids are revealed at each
time step and the value is revealed only upon winning an
auction. However the proposed algorithm is based on a discretization
of the bidding space which relies on the prior knowledge of the
smallest gap between two distinct bids. 
With this knowledge, the proposed algorithm achieves
an adversarial regret of the order of $\sqrt{T}$.

\paragraph{Contributions}
The highlights of Sections~\ref{sec:properties}--\ref{sec:unknown_F}
are the following. In Section~\ref{sec:properties} we stress the
hardness of the first-price bid optimization task, showing that in
general it necessarily leads to high minimax regret rates. We however
transplant ideas introduced in the case of posted prices to exhibit
natural assumptions ensuring that the first-price utility is smooth,
paving the way for faster learning. In Section~\ref{sec:known_F}, we
consider the case where the learner can assume knowledge of $F$ and
propose a new UCB-type algorithm called UCBid1 for learning the
optimal bid with low regret. UCBid1 is adaptive to the difficulty of
the problem in the sense that its regret is $O(\sqrt{T})$ in difficult
cases, but comes down to $O(\log^2 T)$ when the first-price utility is
smooth. We also provide lower-bound results suggesting that these
rates are nearly optimal. In Section~\ref{sec:unknown_F}, we consider
the more general setting where $F$ is initially unknown to the
learner. By leveraging the structure of the first-price bidding
problem, we are able to propose an algorithm, termed UCBid1+, which is
a direct generalization of UCBid1. Interestingly, this algorithm is
not optimistic anymore: it does not submit bids which are with high
probability above the (unknown) optimal bid. However, it can still be
proved to achieve a regret rate of $O(\sqrt{T})$ in the most general
case and, more importantly, a regret rate upper bounded by
$O(T^{1/3+ \epsilon})$ for every $\epsilon>0$ when the first-price
utility satisfies the regularity assumptions mentioned in Section
\ref{sec:properties}.  The latter result relies on an original proof
notably based on the use of a local concentration inequality on the
empirical CDF. All the proofs corresponding to these three sections
are presented in appendix. Section~\ref{sec:experiments} closes the paper with numerical
simulations where we compare the proposed algorithms with
continuously-armed bandit strategies and tailored strategies from the
literature, both using simulated and real-world data.

\section{Properties of Stochastic First-Price auctions}\label{sec:properties}

\begin{figure}[htb] \centering
\includegraphics[width=.44\textwidth]{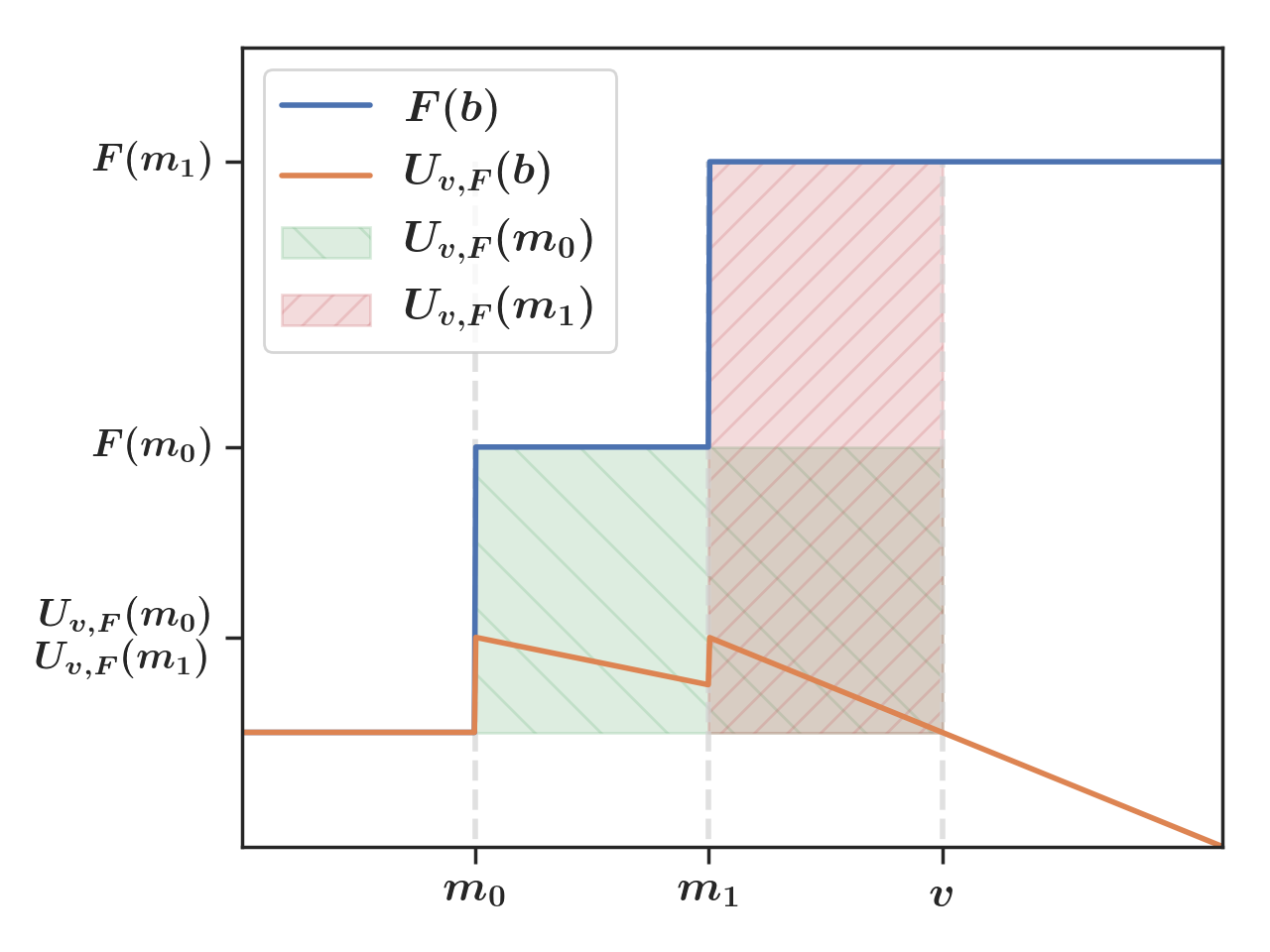}
\caption{An example with two maximizers}\label{fig:2_maxima_dis}
\end{figure}

There are two important difficulties with first price auctions. The
first one lies in the fact that the utility can have multiple
maximizers (or multiple modes with arbitrarily close values) and thus
lead to arbitrarily hard optimization problems.  To illustrate this,
we provide in Figure \ref{fig:2_maxima_dis} an example of value $v$
and discrete distribution, supported on two values $m_0,m_1$, that
leads to a utility having two global maximizers. Note that the utility
$U_{v,F}(b)$ is the area of the rectangle with vertices
$(b,F(b)), (b,0),(v,F(b)), (v,0)$. This observation makes it easy to
build examples with multiple maxima. Discrete examples like the one in
Figure \ref{fig:2_maxima_dis} are intuitive because the utility is
decreasing between two successive points of the support, but there
also exist similar cases with continuous distributions (see for
example Appendix \ref{app:par_2max}). This example also shows that
there exist combinations of bids distributions and values for which the
utility is not regular around its maximum.

The second difficulty comes from the fact that the mapping from $v$ to
the largest maximizer, $\psi_F: v \mapsto b^*_{v,F}$ may also lack
regularity. Indeed, keeping the distribution in Figure
\ref{fig:2_maxima_dis} but setting the value to $v' = v + \Delta$,
with a positive $\Delta$ (resp. to $v'=-\Delta$) yields that the set
of maximizers is $\{m_1\}$ (resp. $\{m_0\}$).
Even though $\psi_F$ can not be proved to be regular in all
generality, it always holds that $\psi_F$ is increasing. This is
intuitive: the optimal bid grows with the private valuation.

\begin{lemma}\label{lem:psi_F}
  For any cumulative distribution $F$, $\psi_F: v \mapsto b^*_{v,F}$
  is non decreasing.
\end{lemma}

The two aforementioned difficulties contribute to making the problem
at hand particularly hard. In the following theorem, we show that any
algorithm is bound to have a worst case regret growing at least like
$\sqrt{T}$.

\begin{theorem}\label{th:lower_bound_gen}
  Let $\mathcal{C}$ denote the class of cumulative distribution
  functions on $[0,1]$. Any strategy, whether it assumes knowledge of
  $F$ or not, must satisfy
  \begin{align*}
  \liminf_{T \rightarrow \infty} \frac{\max_{v \in [0,1], F \in \mathcal{C}}R_T^{v,F}}{\sqrt{T}}
  &\geq   \frac{1}{64}, 
  \end{align*}
\end{theorem}

Theorem \ref{th:lower_bound_gen} corresponds to Theorem 6
in~\cite{han2020optimal}. For completeness, we prove it in Appendix
\ref{sec:app_lower_bound}. The proof relies on specifically hard
instances of CDF that are perturbations of the example of Figure
\ref{fig:2_maxima_dis}.
It illustrates the complexity of bidding in first-price auctions, when $F$ and $v$ are arbitrary. This complexity stems from specifically hard instances of $F$ and $v$.
We present a natural assumption that avoids these pathological cases.

\begin{assumption} \label{ass:pseudo-mhr} $F$ is continuously differentiable and is strictly log-concave.
\end{assumption}

This assumption is reminiscent of the monotonic hazard rate (MHR)
condition (see~e.g. \cite{cole2014sample}), that appears in the
analysis of the posted price problem.  While MHR requires
${f}/{(1-F)}$ to be increasing, Assumption \ref{ass:pseudo-mhr}
requires ${f}/{F}$ to be decreasing.  In particular, this condition is
satisfied by truncated exponentials and Beta distributions with $f$ of
the form $C x^{\alpha-1}$ where $\alpha>1$ or $C (1-x)^{\beta -1}$
where $\beta>1$, or Beta distributions in which
$\alpha + \beta< \alpha \beta$ (see Lemma \ref{lem:beta_distr} in
Appendix \ref{sec:app_properties}). Assumption \ref{ass:pseudo-mhr}
plays roughly the same role for first price auctions than MHR for the
posted price setting. It guarantees in particular that there is a
unique optimal bid. Note that if $F$ satisfies Assumption
\ref{ass:pseudo-mhr}, $F$ is increasing, and admits an inverse which
we denote by $F^{-1}$.

\begin{lemma} \label{lem:unique_max} Under Assumption \ref{ass:pseudo-mhr}, for any $v\in[0,1]$ the mapping $b\mapsto U_{v,F}(b)$ has a unique maximizer.
\end{lemma}

As does the MHR assumption for the posted-prices setting, Assumption
\ref{ass:pseudo-mhr} ensures that the utility is strictly concave
\emph{when expressed as a function of the quantile} $q=F(b)$
\emph{associated with the bid} $b$. Another important consequence of
Assumption \ref{ass:pseudo-mhr} is that the mapping from $v$ to the
optimal bid $b^*_{v,F}$ is guaranteed to be regular.

\begin{lemma}\label{lem:lipschitz} If Assumption \ref{ass:pseudo-mhr} is satisfied and $f$ is continuously differentiable, then $\psi_F: v \mapsto b^*_{v, F}$ is Lipschitz continuous with a Lipschitz constant 1.
\end{lemma} 

Indeed, if $f$ is continuously differentiable and if $f$ does not
vanish on $[0,1[$ (which is implied by Assumption
\ref{ass:pseudo-mhr}), $\psi_F$ is invertible and it inverse $\phi_F$
writes $\phi_F : b \mapsto b+ {F(b)}/{f(b)}$. Assumption
\ref{ass:pseudo-mhr} ensures that $\phi_F$ admits a derivative that is
lower-bounded by $\phi_F'(b) > 1$.

Assumption \ref{ass:pseudo-mhr} also implies the important property
that the probability of winning the auction at the optimal bid
$F(b^*_{v,F})$ cannot be arbitrarily small when compared to $F(v)$.

\begin{lemma}\label{lem:bound_F_b*} If  Assumption \ref{ass:pseudo-mhr} is satisfied, then
$$F(b^*_{v,F})\geq \frac{F(v)}{e}\;.$$
\end{lemma}

We conclude this section by additional properties that are essential
for obtaining low regret rates: the utility is second-order regular,
when expressed as a function of the quantiles. Let $W_{v,F}$ denote
the utility expressed as a function of the quantile,
$W_{v,F}: q\mapsto U_{v,F}(F^{-1}(q))$, and let
$q^*_{v,F}:= F(b^*_{v,F})$ be its maximizer. Under Assumption
\ref{ass:pseudo-mhr}, the deviations of $W_{v,F}$ from its maximum are
lower-bounded by a quadratic function.

\begin{lemma}\label{lem:quadratic} Under Assumption \ref{ass:pseudo-mhr}, for any $q\in [0, 1]$, $$W_{v,F}(q^*_{v,F}) - W_{v,F}(q) \geq \frac{1}{4}(q^*_{v,F} - q)^2 W_{v,F}(q^*_{v,F}).$$
\end{lemma}

This property relies, among other arguments, on the observation
that \begin{equation*}W'_{v,F}(q) =v- \phi_F(F^{-1}(q))=
  \phi_F(F^{-1}(q^*_{v,F}))- \phi_F(F^{-1}(q))\end{equation*} and that
$\phi_F'$ is lower-bounded by $1$ under Assumption
\ref{ass:pseudo-mhr} (see discussion of Lemma~\ref{lem:lipschitz}
above). Similarly, in order to obtain a quadratic lower bound on
$W_{v,F}(q)$, one needs to show that $\phi_F'$ may be upper
bounded. This is the purpose of the following regularity assumption.

\begin{assumption}\label{ass:lambda_pseudo-mhr} $F$ admits a density $f$ such that $c_f<f(b)<C_f, \forall b \in[b^*_{v,F} - \Delta,b^*_{v,F} + \Delta]$ and
$\phi_F: b \mapsto b+ {F(b)}/{f(b)}$ admits a derivative that is upper-bounded by a constant $\lambda \in \mathbb{R}^+$  on $[b^*_{v,F},b^*_{v,F} + \Delta]$.
\end{assumption}

Assumption~\ref{ass:lambda_pseudo-mhr} holds, in particular, when $F$
is twice differentiable, $f$ is lower-bounded by a positive constant
and $f'$ is upper-bounded by a positive constant on a neighborhood of
$b^*_{v,F}$.  Note that in the field of auction theory, it is common to assume that the utility is approximately quadratic around the maximum, which is a far stronger assumption, as stated in  \citep{nedelec2020learning} (see \citep{kleinberg2003value} for example). Assumption~\ref{ass:lambda_pseudo-mhr} implies the following lower bound for the
utility expressed as a function of the quantiles.

\begin{lemma}\label{lem:sub_quadratic} Under Assumption
  \ref{ass:lambda_pseudo-mhr}, for any
  $q \in [q^*_{v,F}, q^*_{v,F}+ C_f \Delta]$,
$$W_{v,F}(q^*_{v,F}) - W_{v,F}(q) \leq \frac{1}{c_f}\lambda (q^*_{v,F} - q)^2 .$$
\end{lemma}



\section{Known Bid Distribution} \label{sec:known_F}
In this section we address the online learning task in the setting
where the bid distribution $F$ is known to the learner from the
start. In order to set the bid $B_t$ at time $t$, the available
information consists in
$N_{t}:= \sum_{s=1}^{t-1} \1 \{ M_s \leq B_s\}$, the number of
observed values before time $t$, and
$\hat{V}_t:= \frac{1}{N_{t}}\sum_{s=1}^{t-1} V_s \1(M_s\leq B_s)$ the
average of those values. Let
$\epsilon_t := \sqrt{{\gamma \log(t-1)}/{2 N_t}}$ denote a confidence
bonus depending on a parameter $\gamma>0$ to be specified below.

\begin{algo}[UCBid1]
Initially set $B_1 = 1$ and, for $t \geq 2$, bid according to
$$ B_t = \max\Big\{ \argmax_{b \in [0,1]}(\hat{V}_t + \epsilon_t -b)F(b)\Big\}.$$ 
\end{algo}

This algorithm, strongly inspired by UCB-like methods designed for
second-price auctions by \cite{weed2016online, achddou2021efficient},
is a natural approach to first-price auctions. The idea behind this
kind of method is that one should rather overestimate the optimal bid,
so as to guarantee a sufficient rate of observation.  As an UCB-like
algorithm, UCBid1 submits an (high probability) upper bound
$\psi_F(\hat{V}_t + \epsilon_t)$ of $b^*_{v,F}$, thanks to Lemma
\ref{lem:psi_F} and since $\psi_F$ is non decreasing. In practice, the
algorithm requires a line search at each step as the utility
maximization task is usually non-trivial, as discussed in
Section~\ref{sec:intro}.

In the most general case,  the regret of UCBid1 admits an upper bound of the order of $\sqrt{T\log(T)}$.

\begin{theorem}\label{th:FPUCBID_general}
When $\gamma>1$, the regret of UCBid1 is upper-bounded as
$$R_T^{v,F} \leq\frac{ \sqrt {2 \gamma}}{F(b^*_{v,F})} \sqrt{T \log T}  + O(\log T)\;.$$
\end{theorem}

Note that $\sqrt{T}$ is the order of the regret of UCB strategies
designed for second-price auctions in the absence of regularity
assumptions on $F$ \citep{weed2016online}. However, under the
regularity assumptions introduced in Section \ref{sec:properties}, it
is possible to achieve faster learning rates.

\begin{theorem}\label{th:FPUCBID_pseudo_mhr}
If   $F$ satisfies Assumption \ref{ass:pseudo-mhr} and \ref{ass:lambda_pseudo-mhr}, then, for any $\gamma>1$,
$$
 R_T^{v,F}\leq  \frac{2\gamma \lambda C_f^2}{F(b^*_{v,F}) c_f}\log^2(T)  + O(\log T) .
$$
\end{theorem}

The $\log^2(T)$ rate of the regret comes from the Lipschitz nature of
$\psi_F$, that makes it possible to bound the gap $B_t-b^*_{v,F}$, and
from the obervation that the utility is quadratic around its optimum.
This explains the similarity with the order of the regret of UCBID in
\citep{weed2016online}, when the distribution of the bids admits a
bounded density.  Indeed, in second-price-auctions, when the
distribution of the bids admits a bounded density, the utility is
locally quadratic around its maximum and the equivalent of $\psi_F$ is
the identity, meaning that the optimal bid is just the value $v$ of
the item.  The presence of the multiplicative constant
$1/F(b^*_{v,F})$ is also expected: it is the average time between two
successive observations under the optimal
policy. 
 This similarity between the structures of second and first price
auctions under Assumptions \ref{ass:pseudo-mhr} and
\ref{ass:lambda_pseudo-mhr} also suggest that the constants in the
regret may be further improved by using a tighter confidence interval
for $v$ based on Kullback-Leibler divergence, proceeding as in
\citep{achddou2021efficient}.


Under Assumption \ref{ass:pseudo-mhr}, the regret of any optimistic
strategy can be shown to satisfy the following lower bound.

\begin{theorem}\label{th:parametric_lower_bound}
Consider all environments where $V_t$ follows a Bernoulli distribution with expectation $v$ and $F$ satisfies Assumption \ref{ass:pseudo-mhr} and is such that $\phi' \leq \lambda$, and there exists $c_f$ and $C_f$ such that $0<c_f<f(b)<C_f, ~ \forall b \in [0,1]$.
If a strategy is such that, for all such environments,
$R_T^{v,F}\leq O(T^{a})$, for all $a>0$, 
and there exists $\gamma >0$ such that $\Po(B_t< b^*)<t^{-\gamma}$,
 then this strategy must satisfy:
\begin{align*}
& \liminf_{T \rightarrow \infty}  \frac{R_T^{v,F}}{\log T}\geq c_f^2 \lambda ^2\left(\frac{v(1-v)(v- b^*_{v,F})}{32} \right).
\end{align*}
\end{theorem}

The first assumption, $R_T^{v,F}\leq O(T^{a})$, is a common
consistency constraint that is used when proving the lower bound of
\cite{lai1985asymptotically} in the well-established theory of
multi-armed bandits. The second assumption, $\Po(B_t< v)<t^{-\gamma}$,
restricts the validity of the lower bound to the class of strategies
that overbid with high probability. By construction, this assumption
is satisfied for UCBid1.

Note that there is a gap between the rates $\log T$ in the lower bound
(Theorem~\ref{th:parametric_lower_bound}) and $\log^2T$ in the
performance bound of UCBid1 (Theorem~\ref{th:FPUCBID_pseudo_mhr}),
which we believe is mostly due to the mathematical difficulty of the
analysis. The $v(1-v)$ factor may be interpreted as an upper bound on
the variance of the value distribution with expectation $v$. Theorem
\ref{th:parametric_lower_bound} displays a dependence on $v$ of the
order of $v^2$ when $v$ tends to 0. However this has to be put in
perspective with the fact that the value of the optimal utility
$U_{v,F}(b^*_{v,F})$ is also quadratic in $v$, when $v$ tends to zero
under the assumptions of Theorem~\ref{th:parametric_lower_bound} (from
Lemma~\ref{lem:quadratic}).

\section{Unknown Bid Distribution}\label{sec:unknown_F}
We now turn to the more realistic, but harder, setting where both the
parameter $v$ and and the function $F$ need to be estimated
simultaneously. For this setting, we propose the following algorithm,
which is a natural adaptation of UCBid1, simply plugging in the empirical
CDF in place of the unknown $F$.

It may come as a surprise that we do not add any optimistic bonus to
the estimate $\hat{F}_t$: it is not necessary to be optimistic about
$F$ since the observation $M_t$ drawn according to $F$ is observed at
each time step whatever the bid submitted.

\begin{algo}[UCBid1+]\label{th:V-opt}
Submit a bid equal to $1$ in the first round, then bid: $$ B_t = \max\Big\{  \argmax_{b \in [0,1]}(\hat{V}_t + \epsilon_t -b)\hat{F}_t(b)\Big\},$$ 
where $\hat{F}_t(b) := \frac{1}{t-1}\sum_{s=1}^{t-1} \1\{M_s<b\}$ and $\epsilon_t := \sqrt{{\gamma \log(t-1)}/{2 N_t}}.$
\end{algo}

Although $B_t$ produced by Algorithm~\ref{th:V-opt} could, in
principle, be arbitrarily small, it is possible to show that there is
no extinction of the observation process. Indeed, after a time that
only depends on $v$ and $F$, $F(B_t)$ is guaranteed to be higher than
a strictly positive fraction of $F(b^*_{v,F})$ with high probability
(see Lemma~\ref{lem:bound_F} in Appendix
\ref{sec:app_unknown_F}). This result implies that the number of
successful auctions $N_t$ asymptotically grows at a linear rate (with
high probability), making it possible to bound the expected difference
between $\hat{V}_t + \epsilon_t$ and $v$. Combined with the DKW
inequality \citep{massart1990tight}, this allows to bound the difference between the utility
and $(\hat{V}_t + \epsilon_t -b)\hat{F}_t(b)$ in infinite norm and
hence the difference between $B_t$ and $b^*_{v,F}$. Putting all the
pieces together (see the complete proof in
Appendix~\ref{sec:app_unknown_F}) yields the following upper bound on
the regret of UCBid1+.
 
\begin{theorem}\label{th:UCBid1+_gen}
 UCBid1+ incurs a regret bounded by
$$R_T^{v,F}
\leq 12 \sqrt{\frac{\gamma v}{U_{v,F}(b^*_{v,F})}}  \sqrt{T \log T}   + O(\log T),
$$
provided that $\gamma>2$.
\end{theorem}

Note that computing the bid $B_t$ for UCBid1+ is easy, as
$(\hat{V}_t + \epsilon_t -b)\hat{F}_t(b)$ necessarily lies among the
observed bids because this function is linearly decreasing
between observed bids. More precisely,
$(\hat{V}_t + \epsilon_t -b)\hat{F}_t(b)) =
\hat{F}_t(M^{(i)})(\hat{V}_t + \epsilon_t-b),$ for
$b \in [M^{(i)}, M^{(i+1)}[$, where $M^{(i)}$ is the i-th order
statistic of the observed bids (obtained by sorting the bids in
ascending order). However, as there is no obvious way to update $B_t$
sequentially, this results in a complexity of UCBid1+ that grows
quadratically with the time horizon $T$.
 
The proof of Theorem \ref{th:UCBid1+_gen} relies on the DKW
inequality to bound the difference between $B_t$ and $b^*$. This happens
to be very conservative and a little misleading in practice. Indeed,
what really matters is the local behavior of the empirical utility,
and hence, of $\hat{F}_t$ around $b^*$. As illustrated by Figure
\ref{fig:zoom_cdf}, locally, $\hat{F}_t$ is roughly a translation of
$F$ plus a negligible perturbation which can be bounded in infinite
norm. This intuition is formalized in~Lemma
\ref{lem:local_concentration_inequality}, a localized version of
the DKW inequality. The fact that $\hat{F}_t$ is locally almost parallel to $F$
imposes a constraint on $B_t$ that may be used to bound its distance from
$b^*$, yielding an improved regret rate under Assumptions
\ref{ass:pseudo-mhr} and \ref{ass:lambda_pseudo-mhr}, as shown by
Theorem \ref{th:fast_rate}.

\begin{lemma}
\label{lem:local_concentration_inequality}
For any $a,b\in [0,1]$, if $F$ is increasing,
\begin{multline*}\sup_{a\leq x \leq b}|\hat{F}_t(x) - F(x) - (\hat{F}_t(a)- F(a))| \\\leq \sqrt{\frac{2(F(b)- F(a))\log
\left( \frac{e \sqrt{t}}{ \eta \sqrt{2(F(b)- F(a))}}\right)
}
{t}} + \frac{\log(\frac{t}{2(F(b)- F(a) \eta^2 })}{6 t},
\end{multline*}
with probability $1-\eta$.
\end{lemma}

\begin{figure}[htb] \centering
\includegraphics[width=.95\textwidth]{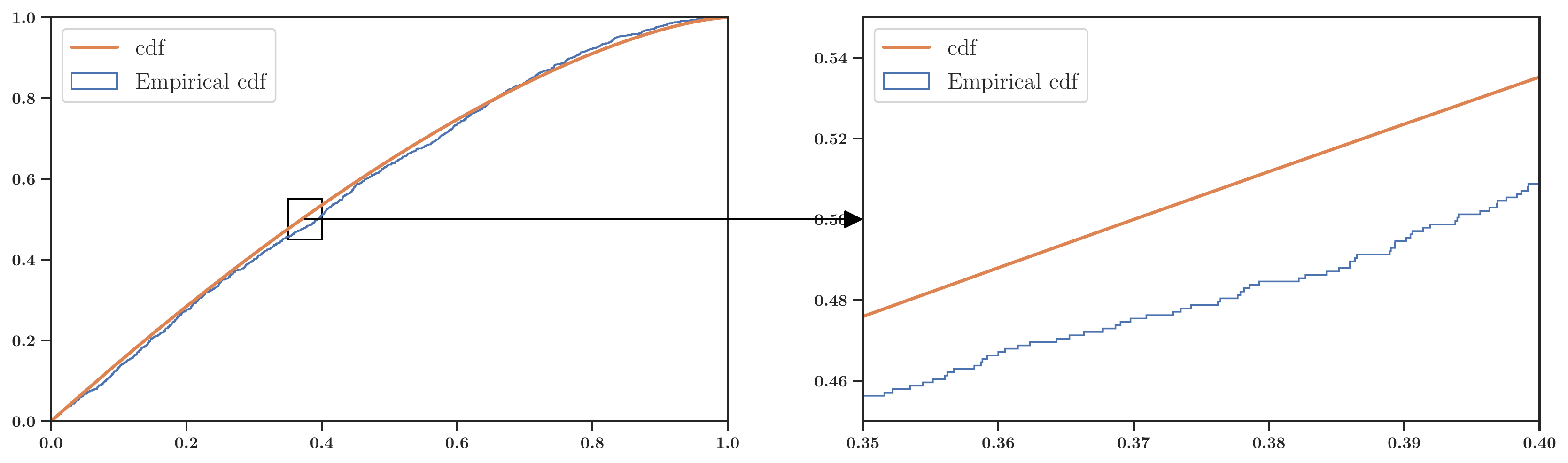}
\caption{Local behavior of the empirical CDF}
\label{fig:zoom_cdf}
\end{figure}

\begin{theorem}\label{th:fast_rate}
If $F$ satisfies Assumptions \ref{ass:pseudo-mhr} and \ref{ass:lambda_pseudo-mhr}, UCBid1+ incurs a regret bounded by
$$R_T^{v,F}
\leq O\big(T^{1/3+\epsilon}\big),
$$
for any $\epsilon>0$,
provided that $\gamma>2$.
\end{theorem}

UCBid1+ thus retains the adaptivity of UCBid1. In general, its regret is of the
order of $\sqrt{T}$ (omitting logarithmic terms), matching the lower bound of Theorem
\ref{th:lower_bound_gen}. But it is reduced to $T^{1/3+\epsilon}$,
for any $\epsilon>0$, in the smooth case defined by Assumptions
\ref{ass:pseudo-mhr} and \ref{ass:lambda_pseudo-mhr}.  In practice,
the improvement over other $\sqrt{T}$-regret algorithms is huge, as shown
in~
the next section. 

\section{Numerical simulations}\label{sec:experiments}

\subsection{Benchmark Algorithms}
\paragraph{Methods pertaining to black box optimization.}
Sequential black box optimization algorithms, also known as
continuously-armed bandits
\citep{kleinberg2008multi,bubeck2011x,munos2011optimistic,valko2013stochastic},
are algorithms designed to find the optimum of an unknown function by
receiving noisy evaluations of that function at points that are chosen
sequentially by the learner. They rely on prior assumptions on the
smoothness of the unknown function. For first-price bidding, we may
consider that the reward $(v-B_t) \1(M_t\leq B_t)$ is a noisy
observation of the utility $U_{v,F}(B_t)$, with a noise bounded by
$1$. Moreover, when $F$ admits a density $f$ and $f(b)<C_f$, then
$-1<U_{v,F}'(b)= (v-x) f(b) -F(b)< C_f$, which implies that $U_{v, F}$
is Lipschitz with constant $\max(1, C_f)$. As a consequence, all
black-box optimization algorithms that consider an objective function
with Lispchitz regularity may be used for learning in stochastic first
price auctions.  HOO \citep{bubeck2011x} has a parameter $\rho$
related to the level of smoothness of the objective function which we
can set to $1/2$, corresponding to the observation that the
first-price utility is Lipschitz under the assumptions discussed
above. This immediately leads to a first baseline approach with
$O(\sqrt{T \log T})$ regret rate.  Setting the parameter related to
the Lipschitz constant of HOO so that it is larger than $C_f$ is not
possible in practice without prior knowledge on $F$. More generally,
knowing the smoothness is considered a challenge most of the time in
black-box optimization, so that several methods have been introduced
that are adaptive to the smoothness, e.g. stoSOO
\citep{valko2013stochastic}.

\paragraph{UCB on a smartly chosen discretization.}
\cite{combes2014unimodal} prove that when the reward function is
unimodal, a discretization based on the smoothness level of this
function suffices to achieve a regret of the order of $\sqrt{T}$. If
$F$ satisfies Assumption \ref{ass:pseudo-mhr}, $U_{v,F}$ is unimodal,
as shown by the proof of Lemma \ref{lem:unique_max}. Hence, using the right discretization while applying UCB,
one can achieve a $O(\sqrt T)$ regret. In particular if the utility is
quadratic, the advised discretization is a grid of $O(T^{1/4})$
values.

\paragraph{O-UCBID1.}
We also implement the following algorithm, that is reminiscent of the
method used by \citep{cesa2014regret} to learn reserve prices.

\begin{algo}[O-UCBid1] Submit a bid equal to $1$ in the first round, then bid:
$$B_t= \max\{b \in [0,\hat{V}_t+ \epsilon_t], \hat{U}_t(b) \geq \max_{b \in [0,1]}\hat{U}_t(b) - 2 \epsilon_t\},$$ where
$\hat{U}_t(b)=(\hat{V}_t-b)\hat{F}_t(b)$.
\end{algo}

This algorithm overbids with high probability, by construction. Thanks
to the DKW inequality, one can control the
difference between the true bid distribution $F$ and its empirical
version $\hat{F}_t$ in infinite norm. Because we observe $M_t$ at
each round, $\|F-\hat{F}_t\|_{\infty}$ is at most $\epsilon_t$ with
high probability. It is easy to show that
$\|U_{v,F}-\hat{U}_t\|_{\infty}$ is bounded by a multiple of
$\epsilon_t$ showing that $B_t$ is (again with high probability)
larger than the unknown optimal bid $b^*_{v,F}$. O-UCBid1 is very
close to the method used by \citep{cesa2014regret} to set a reserve
price in second-price auctions. While in first-price auctions, a
bidder needs to overbid in order to favor exploration, sellers in
second-price auctions are encouraged to offer a lower price than the
optimal one, as they can only observe the second highest bid if their
reserve price is set lower than the latter. The approach of
\cite{cesa2014regret} requires successive stages as sellers in
second-price auctions can only observe the second-price and need to
estimate the distribution of all bids based on this information. In
our setting, we have direct access to the opponents' highest bid and
successive stages are not required any longer.  We prove that the
regret incurred by O-UCBid1 is of the order of $\log T \sqrt{T}$ when
$\gamma>1$, which makes it an interesting baseline algorithm, that has
guarantees similar to those of black box optimization algorithm,
without the need of knowing the smoothness or the horizon. We refer to
Theorem \ref{th:optiBID} in Appendix \ref{sec:app_unknown_F} for
further details.

\paragraph{Methods for discrete distributions}  We run UCBid1+ on discrete examples. In this case, we compare it to UCB on a discretization of $[0,1]$ and to WinExp, a generalization of Exp3 for the problem of learning to bid \citep{feng2018learning}.

\subsection{Experiments On Simulated Data}
\begin{figure}[htb] \centering
\includegraphics[width=.9\textwidth]{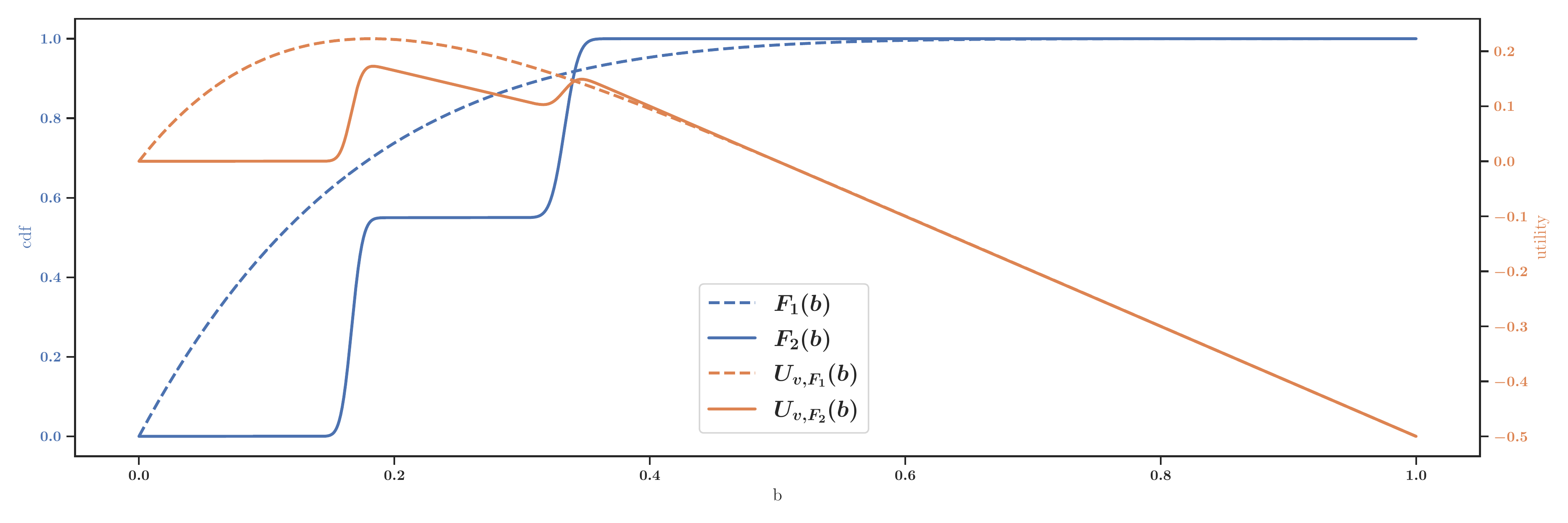}
\caption{Two choices of $F$; associated utilities for $v=1/2$.}
\label{fig:set_utility}
\end{figure}

\begin{figure}[htb] \centering
\subfigure[Regret plots under the first instance of the problem]{
\includegraphics[width=0.45\textwidth]{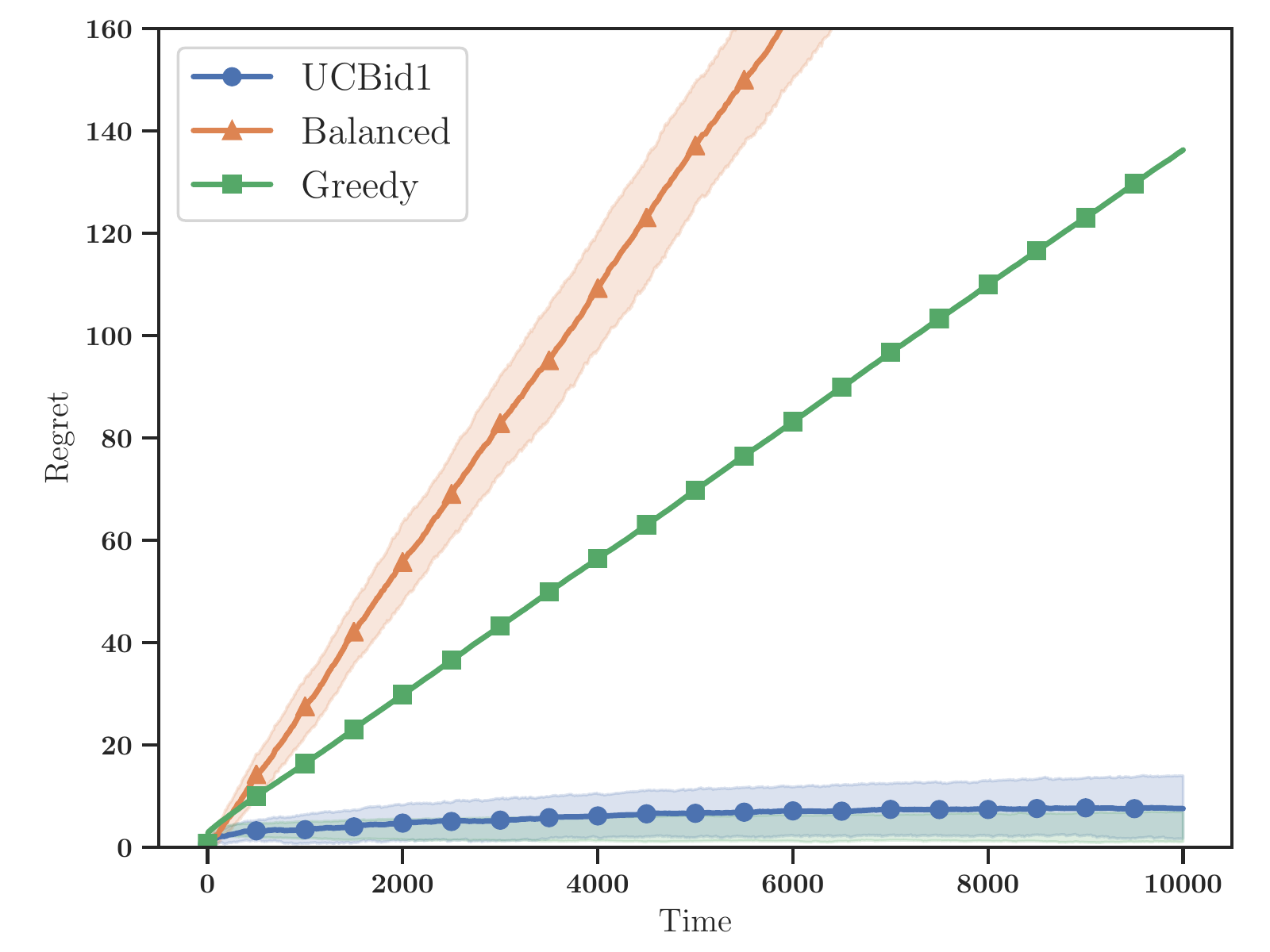}
\label{fig:known_F_1}}
\hfill
\subfigure[Regret plots under the second instance of the problem]{
\includegraphics[width=.45\textwidth]{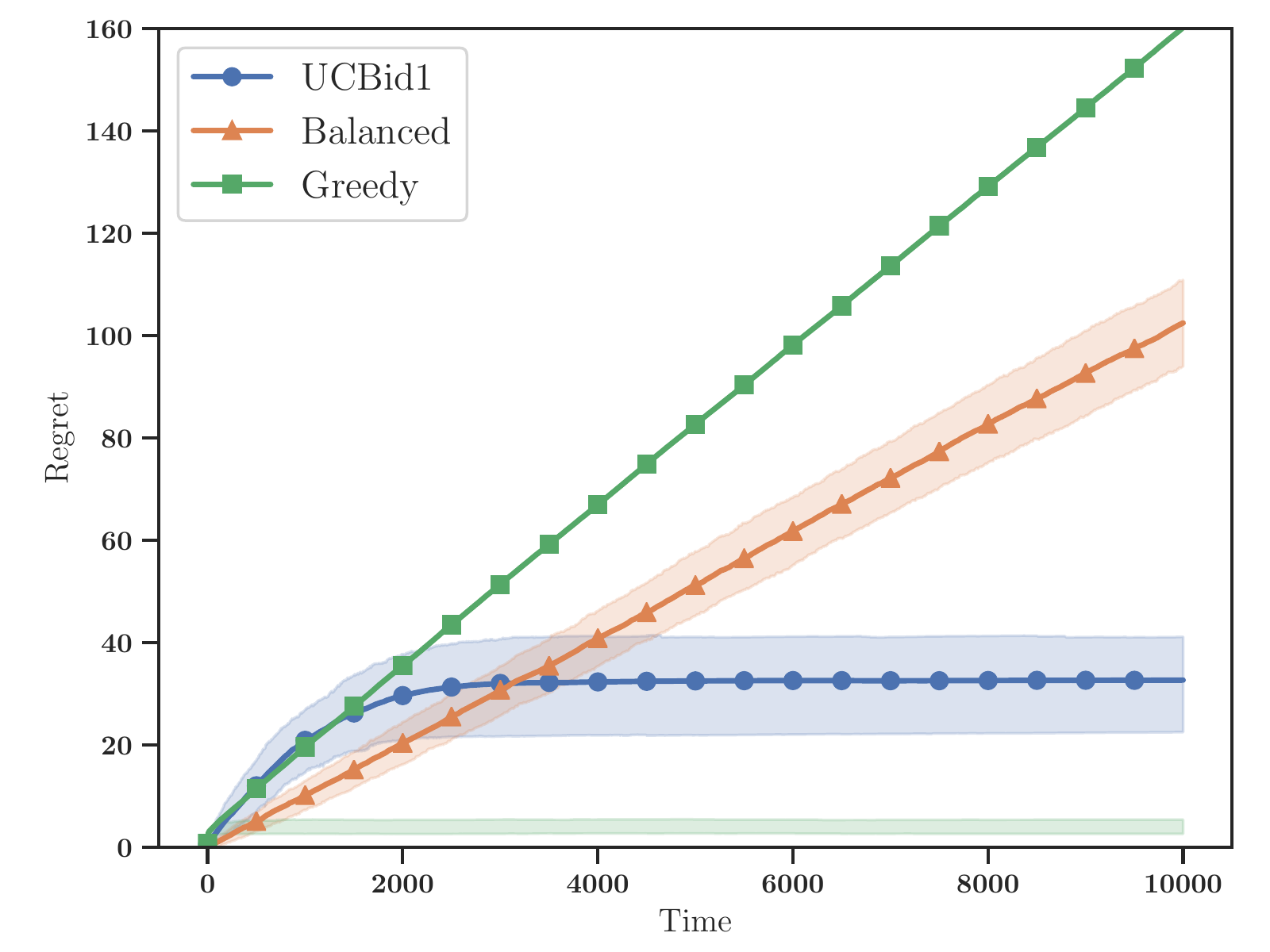}
\label{fig:known_F_2}}
\caption{Regret plots for known $F$}
\end{figure}

\begin{figure}[htb] \centering
\subfigure[Regret plots under the first instance of the problem]{\includegraphics[width=.44\textwidth]{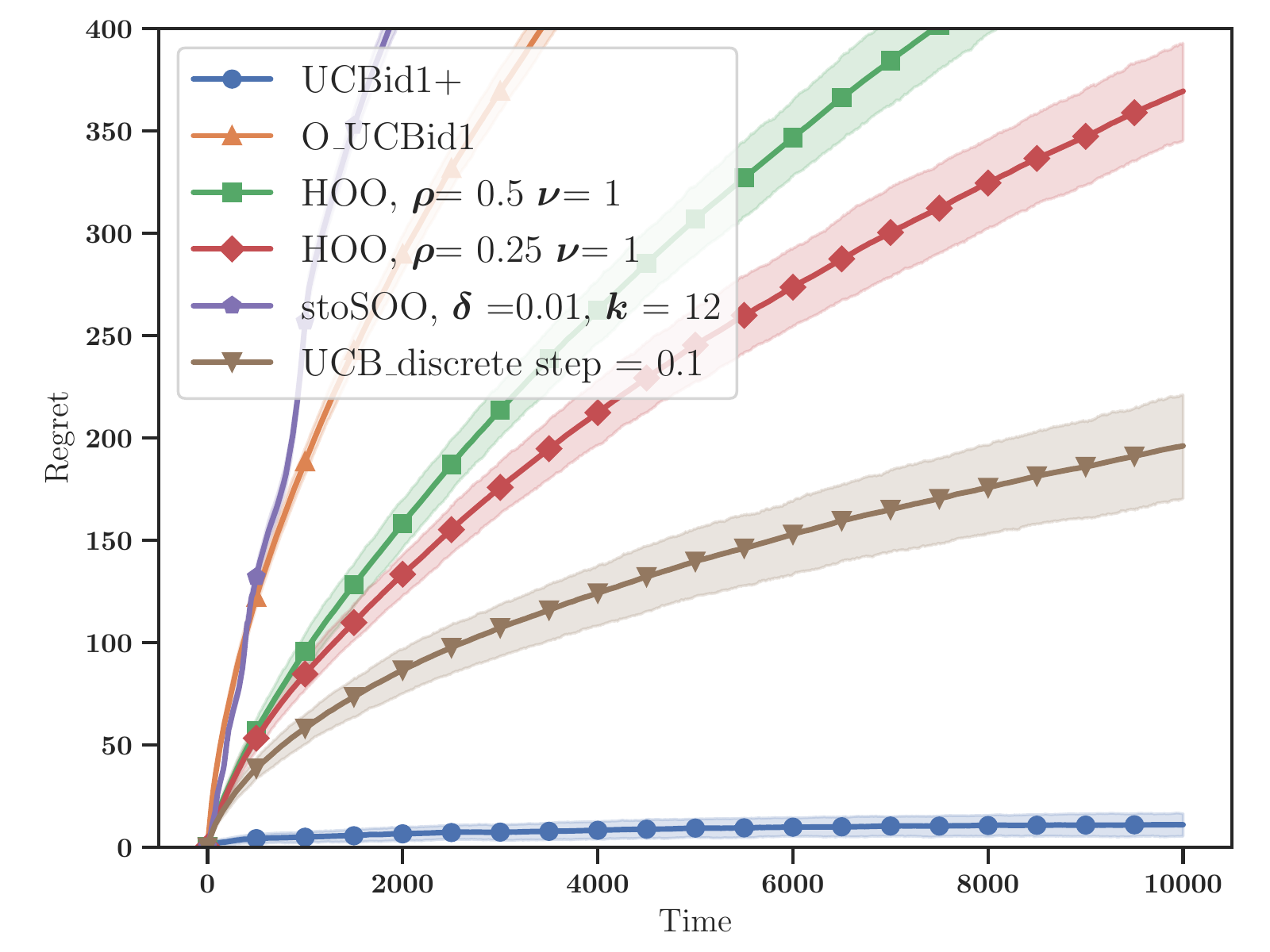}\label{fig:beta_unknown_F}}
\subfigure[Regret plots under the second instance of the problem]{
\includegraphics[width=.44\textwidth]{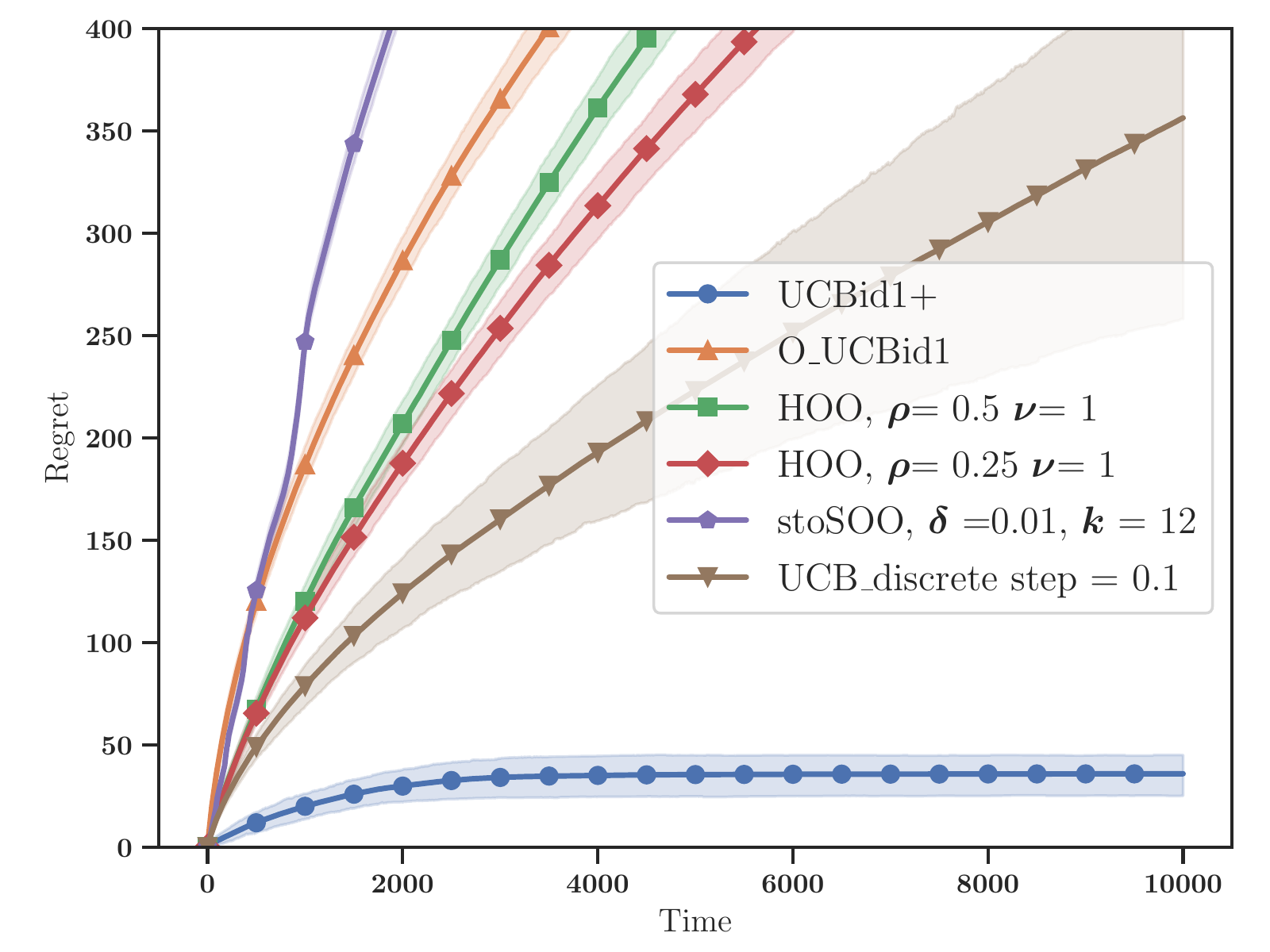}\label{fig:mix_unknown_F}}
\caption{Regret plots for unknown $F$}
\label{fig:unknown_F}
\end{figure}

In this section we focus on two particular instances of the first
price auction learning problem. The first instance is characterized by
a value distribution set to a Bernoulli distribution of average $0.5$,
and a distribution of the highest contestants' bids set to a
Beta(1,6). The second instance only differs by the distribution of the
highest contestants' bids, which is set to a mixture of two Beta
distributions:
$0.55 \times Beta(500, 2500)+ 0.45 \times Beta(1000, 2000)$. This
distribution is very close to that used in the proof of Theorem
\ref{th:lower_bound_gen}, but is continuous. The cumulative
distribution and the matching utility of each instance are plotted on
Figure \ref{fig:set_utility}. Both distributions are smooth but the
first one satisfies Assumption \ref{ass:pseudo-mhr}, while it is not
clear that the second one does.

Figures \ref{fig:known_F_1} and \ref{fig:known_F_2} show the regret of
various strategies when $F$ is known. The first (respectively second)
figure represents the regrets of these strategies under the first
(respectively second) instance of the problem described above. The
horizon is set to $10000$ and the results of 720 Monte Carlo trials
are aggregated. The plots represent the average regret over time
(shaded areas correspond to the interquartile range).  The strategy
termed Greedy is a naive strategy that bids
$\max \argmax \hat{U}_t(b)$, whenever it has made more than three
observations. It shows a linear regret, which comes from the fact that
when it only observes value samples equal to zero during the first
three observations, it bids $0$ indefinitely, and thus incurs the
regret $U_{v,F}(b^*_{v,F})-U_{v,F}(0)$ at each time step. Observing
only $0$ three times in a row is not very likely: the third quartile
is very small, but the consequences are so terrible that the average
is many orders of magnitude higher. The strategy termed Balanced
consists in bidding the median of the highest contestants' bids. It
guarantees that the learner is able to win half of the rounds. As
expected, this strategy, which does not adapt to the instance at hand,
shows poor performances in both cases. However, it is a better
solution than bidding $0$ or $1$. Finally, we also plot the regret of
UCBid1. Note that in order to implement UCBid1 we would have to
compute $\argmax_{b \in [0,1]}(\hat{V}_t + \epsilon_t -b)F(b)$ at each
round; instead we only use an approximation of this quantity by
computing the argmax of the function over a grid of $10000$
values. UCBid1 outperforms the naive baseline strategies in both
cases. Under the more complex second instance of the problem, it shows
a larger regret than under the first one. However, even in this more
complex case, the rate of growth of the regret stays very low.

In Figure \ref{fig:unknown_F}, we analyze the regrets of different
algorithms when $F$ is unknown. In this setting, we compare UCB on a
discretization of $[0,1]$ with 10 arms, HOO \citep{bubeck2011x} with
various parameters, O-UCBid1 and UCBid1+ with $\gamma =1$ and stoSOO
\citep{valko2013stochastic} with the parameters recommended in the
latter paper. For efficiency reasons, we also do not allow the tree
built by HOO and stoSOO to have a depth larger than $\log_2{T}$. The
various versions of HOO, UCB, as well as stoSOO show regret plots that
could correspond to a $\sqrt{T}$ behavior. UCBid1+ shows a
dramatically improved regret plot compared to the black box
optimization strategies.

\begin{figure}[htb] \centering
\subfigure[Utility with $v=1/2$]{\includegraphics[width=.44\textwidth]{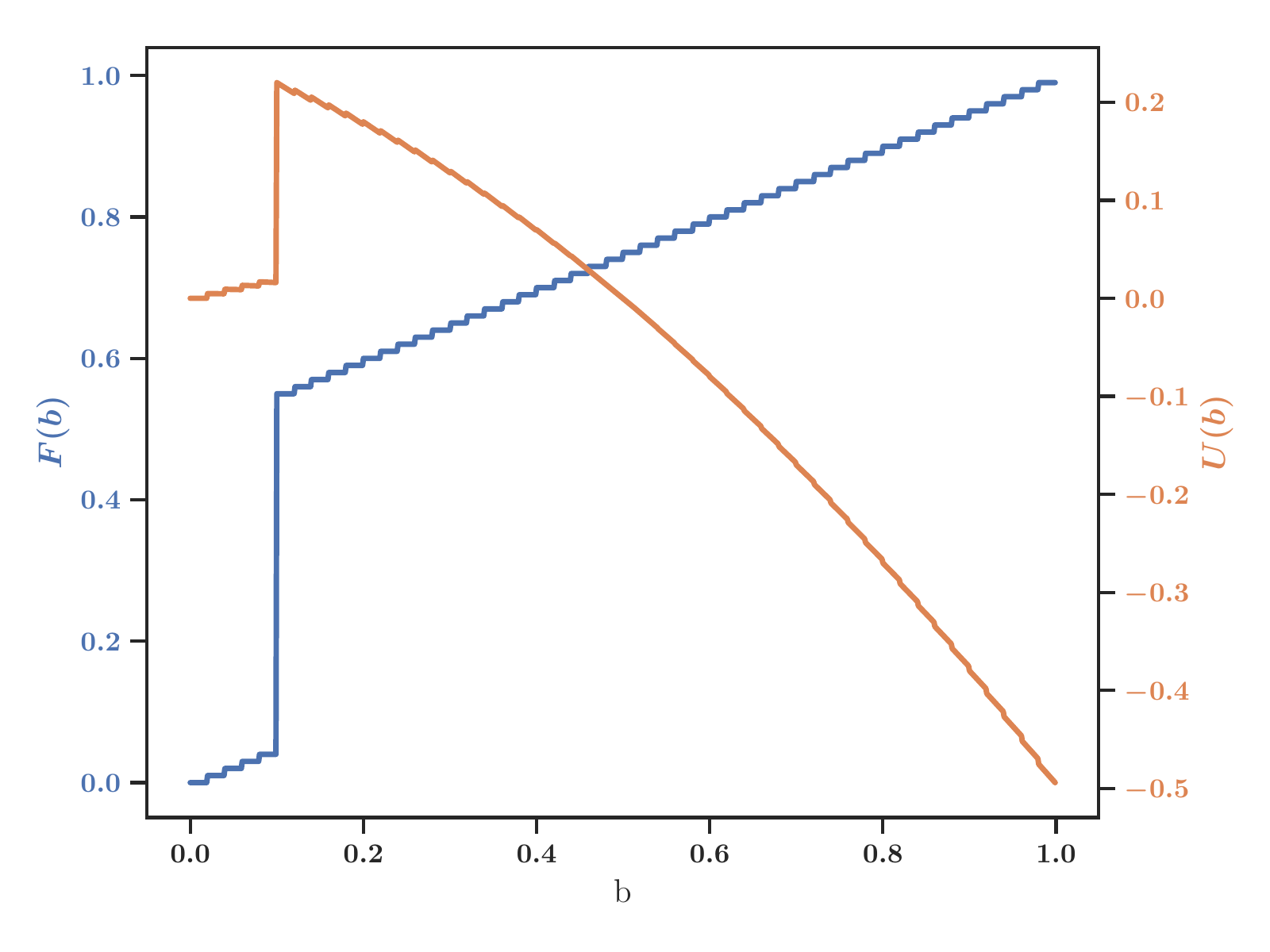}\label{discrete_utility}}
\subfigure[Regret plots]{
\includegraphics[width=.44\textwidth]{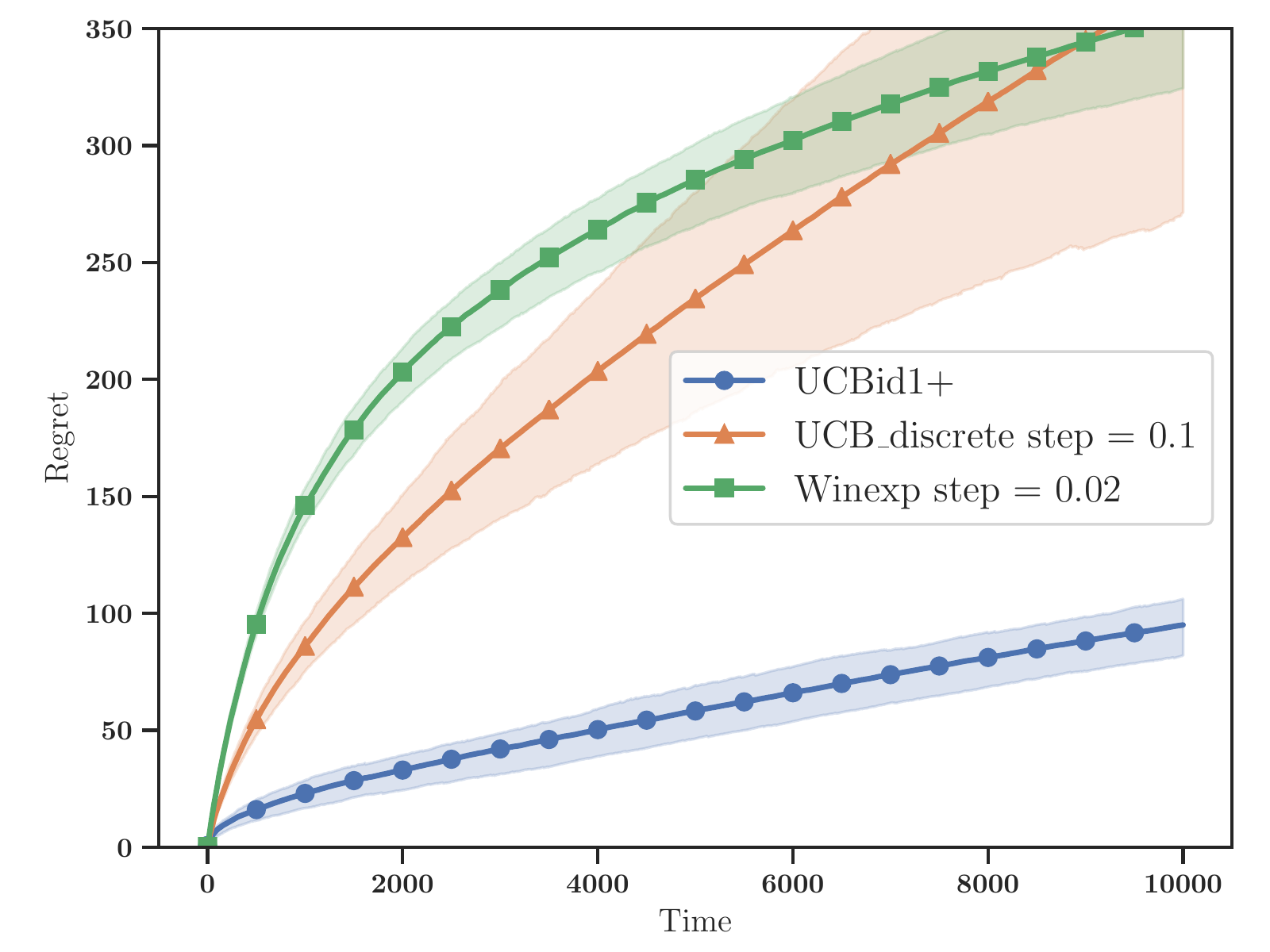}\label{fig:discrete_regret}}
\caption{An example with discrete bids}
\label{fig:discrete_example}
\end{figure}

Figure \ref{fig:discrete_example} shows a different example where the
distribution of bids is discrete with a probability mass of $0.51$ on
$0.1$ and equal probability masses on
$i/50, \forall i \in [1\ldots 4, 6, \ldots, 50]$. We compare UCBid1+
with UCB, having operated a discretization into 10 arms and with
Winexp with a discretization into 50 arms. UCBid1+ again yields a regret at
least 5 times smaller than the other algorithms.
In addition, it is important to stress that UCBid1+ and O-UCBid1 are anytime
algorithms, while all the alternatives shown on Figures~\ref{fig:unknown_F}
and~\ref{fig:discrete_example} require, at least, the knowledge of the
time horizon.

\begin{figure}[htb] \centering
\subfigure[Utility with $v=1/2$]{\includegraphics[width=.44\textwidth]{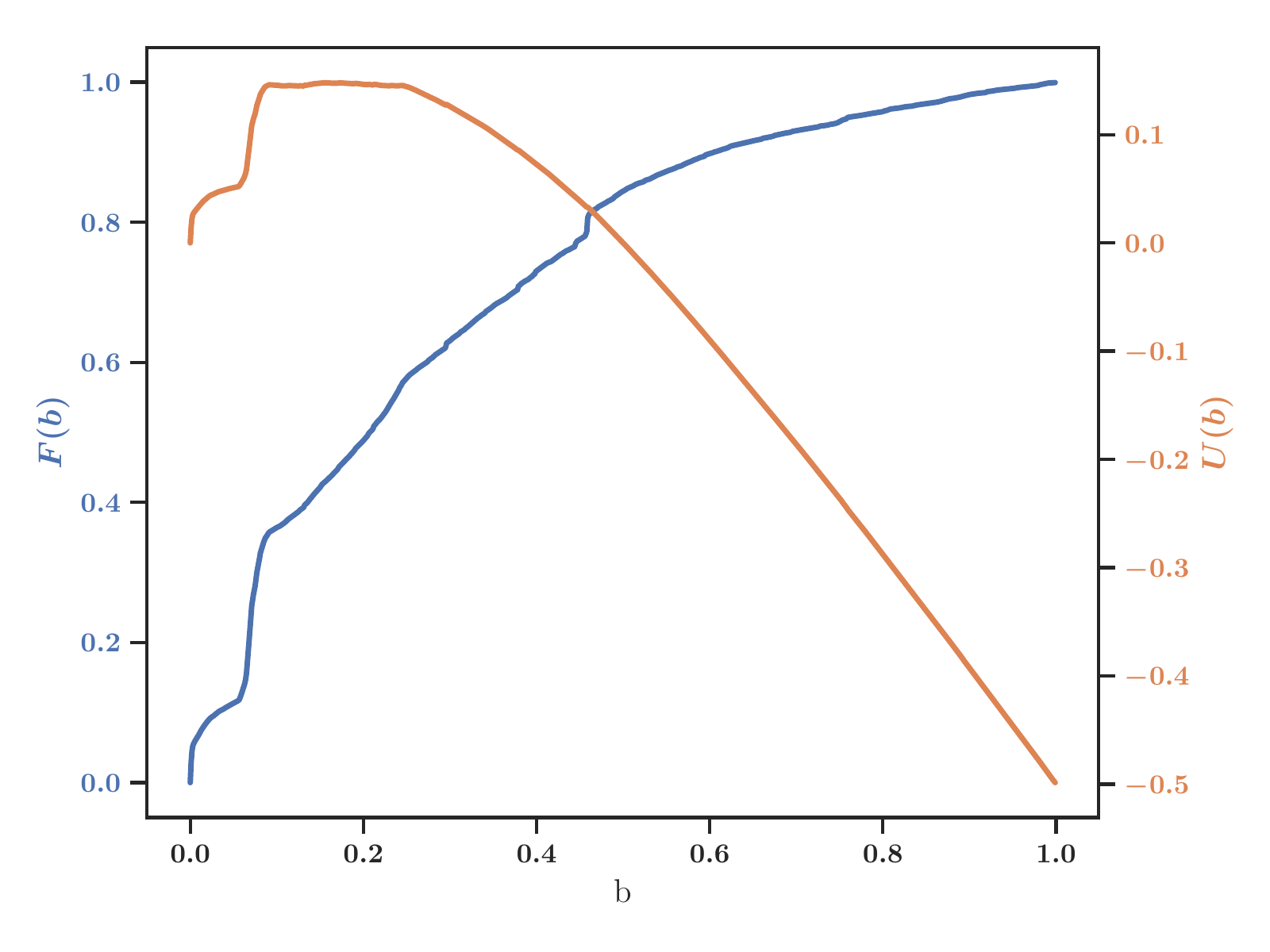}\label{real_data_utility}}
\subfigure[Regret plots with bidding data]{
\includegraphics[width=.44\textwidth]{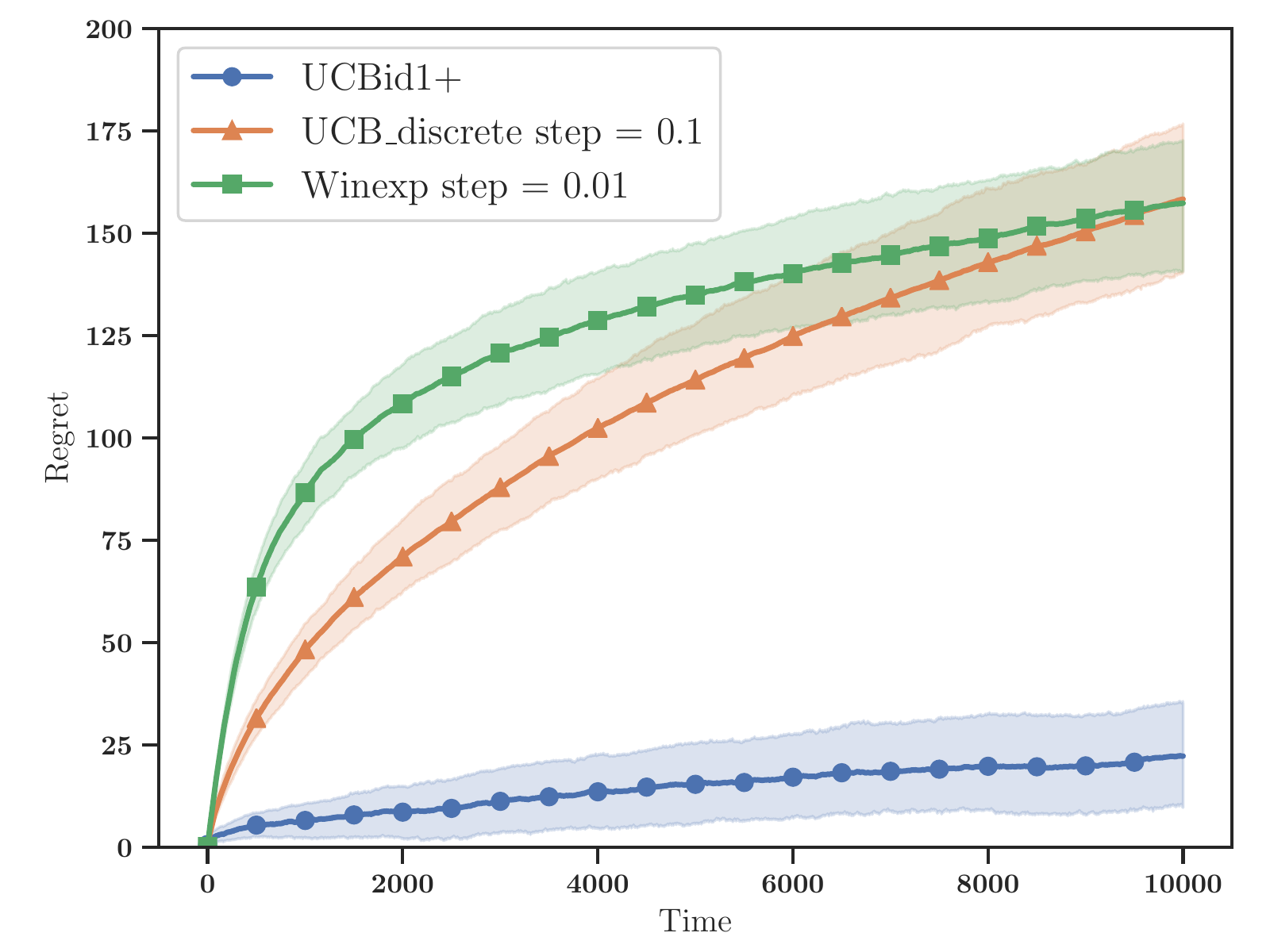}\label{fig:real_data_regret}}
\caption{Experiment with real bidding data}
\end{figure}

\subsection{Experiments On a Real Bidding Dataset}

We also experiment on a real-world bidding dataset representing the
highest bids from the contestants of one advertiser on a certain
campaign. Thanks to Numberly, a media trading agency, Adverline, an advertising network, and Xandr, a supply and demand-side platform, we collected a set of 56607 bids that were made
on a specific placement on Adverline's inventory on
auctions that Numberly participated to, for a specific
campaign. We keep only the bids smaller than the 90\% quantile and we
normalize them to get data between 0 and 1 (see Figure
\ref{fig:rw_histo} in Appendix \ref{sec:app_expe} for a
histogram). The regret plots are represented in Figure
\ref{fig:real_data_regret}. As earlier, with discrete simulated data,
we compare UCBid1+ with UCB, having operated a discretization into 10
arms and with Winexp with a discretization into 100
arms. Unsurprisingly, the regret plots are similar to those with
simulated data, since the distributions at hand are similar. UCBid1+
still largely outperforms the baseline algorithms.
\acks{We would like to thank Adverline for accepting to provide us with the bidding data on their inventories and Xandr for making this data transaction possible. We are very grateful to them for their support on this project.
Aurélien Garivier acknowledges the support of the Project IDEXLYON of the University of Lyon, in the framework of the Programme Investissements d’Avenir (ANR-16-IDEX-0005), and Chaire SeqALO (ANR-20-CHIA-0020).}
{\small
\bibliography{biblio}}
\input{appendix}

\end{document}

%% file: appendix.tex
\clearpage
\onecolumn
\appendix
{\Large\textbf{Supplementary Material}}\vspace*{1em}\\
\paragraph{Outline.} We prove in Appendix \ref{sec:app_properties} all the results pertaining to Section \ref{sec:properties} apart from Theorem \ref{th:lower_bound_gen}, which is  proved separately in Appendix \ref{sec:app_lower_bound}. In Appendix \ref{sec:app_prelim}, we introduce preliminary results necessary to analyze the regrets of the algorithms presented in main body of the paper. Appendix \ref{sec:app_known_F} contains all the proofs of the results of Section \ref{sec:known_F}, while the theorems of Section \ref{sec:unknown_F} are proved in  Appendix \ref{sec:app_unknown_F}. A figure related to Section \ref{sec:experiments} is presented in  Appendix \ref{sec:app_expe}.

\paragraph{Notation.}
\begin{itemize}
\item In the following we write $U$ instead of $U_{v,F}$ (respectively $W$ instead of $W_{v,F}$; $b^*$ instead of $b^*_{v,F}$; $q^*$ instead of $q^*_{v,F}$ and $R_T$ instead of $R_T^{v,F}$) when there is no ambiguity.
\item  $b(q)$ denotes $F^{-1}(q)$.
\item $\hat{V}(n):=1/n\sum_{s=1}^n V(s)$ is the mean of the $n$ first
observed values.
\item  We set  $V'_s = V_s$ if $ M_s\leq B_s,$ and $V'_s = \emptyset$
otherwise.
\item We set $\mathcal{F}_t = \sigma((M_s,V'_s)_{s\leq t})$ be the
  $\sigma$-algebra generated by the the bid maxima and the values observed up
  to time $t$. 
\item $S_t :=(V_t-b^*)\1(M_t < b^*) - (V_t-B_t)\1(M_t < B_t)$  represents the instantaneous regret.
\end{itemize}

\section{Properties of first-price auctions}\label{sec:app_properties}
\subsection{General properties}

\begin{repeatlem}{lem:psi_F} 
For any cumulative distribution function $F$, $\psi_F$ is non decreasing.
\end{repeatlem}
\begin{proof}
Let $0<v_1<v_2<1$. We have $U_{v_2,F}(b_{v_2,F}^*)-U_{v_2,F}(b_{v_1,F}^*)\geq 0$ and  $U_{v_1,F}(b_{v_1,F}^*)-U_{v_1,F}(b_{v_2,F}^*)\geq 0$, by definition of $b_{v_1,F}^*$ and $b_{v_2,F}^*$.

By summing these two inequalities,
$U_{v_2,F}(b_{v_2,F}^*) -U_{v_1,F}(b_{v_2,F}^*)-(U_{v_2,F}(b_{v_1,F}^*)- U_{v_1,F}(b_{v_1,F}^*))\geq 0.$
Hence 
$$(v_2-v_1)(F(b^*_{v_2,F})-F(b^*_{v_1,F}))\geq 0.$$ 
We then prove the result by contradiction, by assuming that $b_{v_1,F}^*>b_{v_2,F}^*$. Then $F(b_{v_1,F}^*)=F( b_{v_2,F}^*)$, since $F$ is non decreasing. In this case, $$U_{v_1, F}(b_{v_1,F}^*)= (v_1- b_{v_1}^*)F(b_{v_1,F}^*)< (v_1- b_{v_2}^*)F(b_{v_2,F}^*)= U_{v_1, F}(b_{v_2,F}^*).$$
This is impossible, since $b^*_{v_1,F}$ is an optimizer of $U_{v_1, F}$. In conclusion, $b_{v_1,F}^*\leq b_{v_2,F}^*$

\end{proof}

\subsection{Properties under regularity assumptions}

\begin{repeatlem}{lem:unique_max}If Assumption \ref{ass:pseudo-mhr} is satisfied, then for any $v\in[0,1]$, $U_{v,F}$ has a unique maximizer.
\end{repeatlem}

\begin{proof}
If $F$ satisfies Assumption \ref{ass:pseudo-mhr} then $\frac{f}{F}$ is decreasing and $\phi_F: b \mapsto b + \frac{F(b)}{f(b)}$ is increasing and $f$ does not vanish on $]0,1[$.\\
The derivative of $U$ is $U'(b) = \left(v-b- \frac{F(b)}{f(b)}\right)f(b)$. So $U'(b)=0$ if and only if $v= b +  \frac{F(b)}{f(b)}$. Since $\phi_F$ is increasing, this can only be satisfied by a single $b \in [0,1]$. Also, since $f$ does not vanish, $U$ is unimodal (increasing then decreasing).\\
\end{proof}

\begin{lemma}\label{lem:strongly_concave}
If Assumption \ref{ass:pseudo-mhr} is satisfied, then $W_{v,F}$ is strongly concave.
\end{lemma}
If $F$ satisfies Assumption \ref{ass:pseudo-mhr} then $\frac{f}{F}$ is decreasing and $\phi_F: b \mapsto b + \frac{F(b)}{f(b)}$ is increasing and $f$ does not vanish on $]0,1[$.\\
The derivative of $U$ is $U'(b) = \left(v-b- \frac{F(b)}{f(b)}\right)f(b)$.
The derivative of $W$ is $W'(q) = \left(v-b- \frac{F(F^{-1}(q))}{f(F^{-1}(q))}\right)= v- \phi_F'(F^{-1}(q))$, since $\phi_F$ is increasing. Consequently,  $U'$ is decreasing, and $U'$ is strongly concave.

\begin{repeatlem}{lem:lipschitz}
 If Assumption \ref{ass:pseudo-mhr} is satisfied and $f$ is differentiable, then $\psi_F: v \mapsto b^*(v, F)$ is Lipschitz continuous with a Lipschitz constant 1.
\end{repeatlem}
\begin{proof}
If $b^*$ is the optimum of the utility $U$, then it satisfies $(v-b^*)f(b^*)-F(b^*)=0$.
It satisfies $$\phi_F(b^*) :=b^* +\frac{F(b^*)}{f(b^*)}=v.$$
Since $\phi_F'(b^*) >1$ thanks to Assumption \ref{ass:pseudo-mhr}, $\phi_F$ is invertible and $(\phi_F)^{-1}= \psi_F$ is  Lipschitzian with constant $1$ .
\end{proof}

\begin{repeatlem}{lem:bound_F_b*} If  Assumption \ref{ass:pseudo-mhr} is satisfied, then $$F(b^*)\geq e^{-1}F(v)$$
\end{repeatlem}
\begin{proof}
We know that $b^*<v$ and 
$$\log \left(\frac{F(v)}{F(b^*)}\right) = \int_{b^*}^v \frac{f(u)}{F(u)}du.$$
Hence $$\frac{F(v)}{F(b)} = \exp \left(\int_{b^*}^v \frac{f(u)}{F(u)} du\right).$$
Since $\frac{f(u)}{F(u)}$ is decreasing, thanks to Assumption \ref{ass:pseudo-mhr}, 
 $$\frac{F(v)}{F(b)} \leq \exp \left((v-b^*)\frac{f(b^*)}{F(b^*)} \right).$$
 We have $v-b^* = \frac{F(b^*)}{f(b^*)}$, by definition of $b^*$. Hence $ \exp \left(v-b^*)\frac{f(b^*)}{F(b^*)}\right)=\exp(1)$
 and $$F(b^*)\geq \exp(-1)F(v).$$ 
 \end{proof}

\begin{repeatlem}{lem:quadratic} If  Assumption \ref{ass:pseudo-mhr} is satisfied, for any $0 \leq q' \leq 1$, $$W(q^*) - W(q') \leq \frac{1}{4}(q^* - q')^2 W(q^*)$$
\end{repeatlem}

\begin{proof}
Note that this proof is an adaptation of the proof of Lemma 3.2 in \cite{huang2018making}.
In this proof, we denote by $b(q)$ $ F^{-1}(q)$.

First of all, let us observe that $U'(b) = (v-\phi_F(b)) f(b)$.
We have $W'(q)= v- \phi_F(F^{-1}(q)).$

Assumption \ref{ass:pseudo-mhr} implies that $\phi_F'(b) >1, ~ \forall b \in [0,1]$.

To prove Lemma \ref{lem:quadratic}, we will apply case-based reasoning.
There are three cases depending on the relation between $q'$ and $q^*$: $q' > q^*$, $q' = q^*$, and $q' < q^*$.
The second case, i.e., $q' = q^*$, is trivial.

\medskip

First, consider the case when $q' > q^*$.
It holds 
\[
W(q^*) - W(q') = \int^{q'}_{q^*} - W'(q) dq = \int^{q'}_{q^*}  \Big(\phi_F(b(q)) -v\Big) dq. 
\]
We therefore need to bound $\phi_F(b(q), \forall q \in [q^*, q'].$
By definition of $q^*$, for any $q$ s.t.\ $q^* \le q \le q'$, we have $$q (v- b(q)) \le q^* (v-b(q^*)).$$
By rewriting this equation, 
\begin{equation} \label{eq:optimality_imp} b(q)\geq \frac{q v - q^* v + q^* b(q^*)}{q}= v\left(\frac{q-q^*}{q}\right) + \frac{q^*}{q}b(q^*)
\end{equation}
Secondly, by the intermediate value theorem, there exists $b \in [b(q^*), b(q)]$, such that 
\[
\phi_F(b(q)) - \phi_F(b(q^*)) =  \phi_F'(b)\Big( b(q) - b(q^*)\Big) \geq b(q) - b(q^*),
\]
for any $q^* \le q \le q'$, where the second inequality follows from Assumption \ref{ass:pseudo-mhr} that $\tfrac{d \phi_F(b)}{db} \ge 1$ and $F$ being increasing thanks to Assumption \ref{ass:pseudo-mhr}.
This in turn yields 
\[\phi_F(b(q)) \geq v + b(q) - b(q^*),
\]
since by definition, $W'(q^*) = \phi_F(b(q^*))=v$.
Combining with Inequality \ref{eq:optimality_imp}, we get that
\[
\phi_F(b(q)) -v \ge  v(\frac{q-q^*}{q}) + \frac{q^* }{q}b(q^*) - b(q^*)\ge (v-b(q^*))(\frac{q-q^*}{q})= \frac{W(q^*)}{q^*} (\frac{q-q^*}{q})
\]
Therefore, we get that
\begin{align*}
W(q^*) - W(q') &= \int^{q'}_{q^*} - W'(q) dq = \int^{q'}_{q^*}  \Big(\phi_F(b(q)) -v\Big) dq\ge\frac{W(q^*)}{q^*} \int^{q'}_{q^*} \frac{q - q^*}{q}  dq \\
&
\ge \frac{W(q^*)}{q^*}\int_{\frac{q'+q^*}{2}}^{q'} \frac{q - q^*}{q} dq,
\end{align*}
since $\frac{q - q^*}{q} \ge 0$ for any $q' \le q \le q^*$.
Moreover, for any $q \ge \tfrac{q' + q^*}{2}$, we have $\frac{q - q^*}{q} = 1 - \frac{q^*}{q}\ge 1 - \frac{2 q^*}{q' + q^*} \ge \frac{q' - q^*}{q' + q^*}$.
Hence, we can derive the following inequality
\[
W(q^*) - W(q') \ge \int^{q'}_{\frac{q' + q^*}{2}} \frac{q' - q^*}{q' + q^*} \frac{W(q^*)}{q^*} dq = \frac{(q' - q^*)^2}{2(q' + q^*)} \frac{W(q^*)}{q^*} = \frac{(q' - q^*)^2}{2 q^*(q' + q^*)} W(q^*) \enspace.
\]
The lemma then follows from the fact that $0 \le q', q^* \le 1$.

The second case, $q' > q^*$ has to be treated a little differently than the first, partly because we now need to upper bound $b(q)$ instead of lower-bounding it. We achieve this by using the concavity of $W$ (proved in Lemma \ref{lem:strongly_concave}).

By concavity of the revenue curve, for any $q' \le q \le q^*$, we have
\[
W(q) \ge \frac{q - q'}{q^* - q'}  W(q^*) + \frac{q^* - q}{q^* - q'} W(q') \enspace,
\]
because $W$ lies above the segment that connects $(q', W(q'))$ and $(q^*, W(q^*))$, between $q'$ and $q^*$. Hence
\[
(v-b(q))q \ge \frac{q - q'}{q^* - q'} (v-b(q^*))q^* + \frac{q^* - q}{q^* - q'} (v-b(q'))q' \ge qv - b(q^*) q^* \frac{q - q'}{q^* - q'} -b(q') q' \frac{q^* - q}{q^* - q'}  ,
\]

And
\[
-qb(q) \ge    \frac{q^* q'}{(q^* - q')}\Big(b(q^*)-b(q') \Big) +   q\frac{q' b(q') - q^*b(q^*)}{q^* - q'}  ,
\]
which yields 
\[
qb(q) \le    \frac{q^* q'}{(q^*- q' )}\Big(b(q')-b(q^*) \Big) +   q\frac{ q^*b(q^*) -q' b(q') }{q^* - q'}  ,
\]
Dividing both sides by $q$, we have
\begin{equation}\label{eq:concave}
b(q) \le    \frac{q^* q'}{q(q^*- q' )}\Big(b(q')-b(q^*) \Big) +   \frac{ q^*b(q^*) -q' b(q') }{q^* - q'}  ,
\end{equation}
Further, by the intermediate value theorem, there exists $b \in [b(q^*), b(q)]$, such that 
\[
\phi_F(b(q)) - \phi_F(b(q^*)) =  \phi_F'(b)\Big( b(q) - b(q^*)\Big) ,
\]
for any $q^* \le q \le q'$.
Further, by Assumption \ref{ass:pseudo-mhr} that $\tfrac{d \phi_F(b)}{db} \ge 1$, and because $b$ is increasing thanks to Assumption \ref{ass:pseudo-mhr}, for any $q' \le q \le q^*$,
\[
\phi_F(b(q)) - \phi_F(b(q^*))\leq b(q) - b(q^*)\]
and 
\[\phi_F(b(q)) \leq v + b(q) - b(q^*) ~=~ v+ b(q) - b(q^*),\]

Combining with Inequality \ref{eq:concave}, we get that
\begin{eqnarray*}
\phi_F(b(q)) & \le & v + \frac{q^* q'}{q(q^*- q' )}\Big(b(q')-b(q^*) \Big) +   \frac{ q^*b(q^*) -q' b(q') }{q^* - q'} -b(q*) \\
& = & v+ \frac{q' (q^* - q)}{q (q^* - q')} \big( b(q') - b(q^*) \big) ~
\le ~  v+ \frac{q' (q^* - q)}{q^* (q^* - q')} \big( b(q') - b(q^*) \big) \enspace,
\end{eqnarray*}
where the last inequality is due to $q \le q^*$ and $b(q') - b(q^*)<0$.
Hence, we have
\begin{align}
W(q^*) - W(q') & =  \int^{q^*}_{q'} W'(q) dq \notag \\& = ~ \int^{q^*}_{q'}v - \phi_F(b(q)) dq \notag \\& \ge ~ \int^{q^*}_{q'} \frac{q' (q^* - q)}{q^* (q^* - q')} \big( b(q^*) - b(q') \big) dq \notag \\
& =  \frac{q'}{2 q^*} (q^* - q') \big( b(q^*) -b(q') 
\big). \label{eq:manyprepeak1}
\end{align}
On the other hand, we have
\begin{equation}
\label{eq:manyprepeak2}
W(q^*) - W(q') =  (q^*-q')v +q'b(q')-q^* b(q^*).
\end{equation}
Taking the 
linear
combination $\frac{2 q^*}{3 q^* - q'} \cdot
\ref{eq:manyprepeak1} + \frac{q^* - q'}{3 q^* - q'} \cdot
\ref{eq:manyprepeak2}$, we have
%
\begin{align*}
W(q^*) - W(q') &\ge  v \frac{(q^* -q')^2}{3q^* - q')} - \frac{(q^* - q')^2}{3 q^* - q'} b(q^*)\\& = 
\frac{1}{q^* (3 q^* - q')} (q^* - q')^2 W(q^*) 
\\&\ge \frac{1}{3} (q^* -
q')^2 W(q^*) \enspace,
\end{align*}
where the last inequality holds because $0 \le q^*, q' \le 1$.
\end{proof}

\begin{repeatlem}{lem:sub_quadratic} If  Assumption \ref{ass:lambda_pseudo-mhr} is satisfied, for any $F^{-1}(b^*)\leq q' \leq F^{-1}(b^*+ \Delta) \leq b^*+ C_f \Delta)$, $$W(q^*) - W(q') \leq \frac{1}{c_f}\lambda (q^* - q')^2 ,$$
\end{repeatlem}

\begin{proof}
\begin{align*}
W(q^*) - W(q') = \int^{q'}_{q^*} - W'(q) dq = \int^{q'}_{q^*}  \Big(\phi_F(b(q)) -v\Big) dq. 
\end{align*}
 by the intermediate value theorem, there exists $b\in [b(q^*), b(q)]$, such that 
\[
\phi_F(b(q)) - \phi_F(b(q^*)) =  \phi_F'(b)\Big( b(q) - b(q^*)\Big) \geq \lambda( b(q) - b(q^*)),
\]

so that $\phi_F(b(q)) -v \leq  \lambda( b(q)-b(q^*))$
when $q^* \leq q \leq q'$ and $\phi_F(b(q)) -v \geq  \lambda (b(q)-b(q^*))$
when $q'\leq q \leq q^*$. Since $f$ is bounded from below by $c_f$, and since by the intermediate value theorem  $\exists u \in [q, q^*], ~ b(q)-b(q^*) = b'(u)(q-q^*)\geq \frac{1}{f(u)}(q-q^*)$,
this yields 
\[
W(q^*) - W(q') \leq \lambda \frac{1}{c_f}(q'-q^*)^2
\]
in both cases. 
\end{proof}

\begin{lemma}\label{lem:beta_distr}
Beta distributions such that $$\alpha + \beta< \alpha \beta$$ satisfy Assumption \ref{ass:pseudo-mhr}.
\end{lemma}
\begin{proof}
The density of a Beta distribution satisfies 
$$f(x) = \frac{x^{\alpha -1}(1-x)^{\beta -1}}{B(\alpha, \beta)}$$
And $$f'(x) = \frac{(\alpha -1) x^{\alpha -2}(1-x)^{\beta -1} - (\beta -1) x^{\alpha -1}(1-x)^{\beta -2}}{B(\alpha, \beta)},$$

where   $B(\alpha, \beta) = \frac{\Gamma(\alpha+\beta)}{\Gamma(\alpha)\Gamma(\beta)}$ when $\Gamma$ denotes the Gamma function.
$F$ satisfies assumption \ref{ass:pseudo-mhr} if and only if
$\left(\frac{f}{F}\right)'(x) =\frac{F(x)f'(x)-f^2(x)}{F^2(x)}<0$, $\forall x \in]0,1[$, which is equivalent to:
\begin{align*}
f'(x) F(x) - f^2(x)<0 ,~\forall x \in]0,1[ &\iff \frac{f'(x)}{f(x)} F(x) < f(x) ,~\forall x \in]0,1[)\\
& \iff F(x) B(\alpha, \beta) \left[(\alpha -1)(1-x) - (\beta -1)x \right]< x^{\alpha}(1-x)^{\beta}, \\
& ~~~~\forall x \in]0,1[.
\end{align*}

Therefore we study the function $G:x \mapsto F(x) B(\alpha, \beta) \left[(\alpha -1)(1-x) - (\beta -1)x \right]- x^{\alpha}(1-x)^{\beta}.$ 
First of all, we observe that $G(0) = 0$.
Next, we note that 
\begin{align*}G'(x) = &- F(x) ( \alpha + \beta -2)B(\alpha, \beta) + ((\alpha-1)- (\alpha + \beta -2)x)x^{\alpha -1}(1-x)^{\beta -1}\\
&- \left(\left( \alpha (1-x) - \beta x \right) x^{\alpha -1}(1-x)^{\beta-1}\right)
\end{align*}
and $G'(0) = 0$. Now, we compute the second derivative of $G$:
\begin{align*}
G''(x) =& -( \alpha + \beta - 2 ) x^{\alpha-1}(1-x)^{\beta-1} + \left( (\alpha -1) - (\alpha + \beta -2) x)\right)^2x^{\alpha-2}(1-x)^{\beta-2} \\
&- (\alpha + \beta -2) x^{\alpha -1}(1-x)^{\beta -1} - \left( \alpha - (\alpha + \beta)x\right) ((\alpha -1)-\\
& (\alpha + \beta -2)x)x^{\alpha -2}(1-x)^{\beta -2} + ( \alpha + \beta)x^{\alpha -1}(1-x)^{\beta-1}
\end{align*}
The sign of $G''(x)$
is the same as that of $P(x)=-\left( \left( \alpha + \beta\right) -4\right) (x(1-x)) + (-1+2x) \left( (\alpha-1) - (\alpha + \beta-2)x\right)$.

By simplifying, we get $P(x) = - (\alpha + \beta)x^2+ 2 \alpha x -(\alpha-1)$.
This polynomial is always negative because its maximum is $P(\frac{\alpha}{\alpha + \beta}) = - \frac{\alpha ^2}{\alpha + \beta} + 2 \frac{\alpha^2}{\alpha + \beta} - \alpha +1 = \alpha^2(\frac{2}{\alpha+ \beta}-1) - \alpha +1 =  \frac{\alpha^2}{\alpha +\beta} - \alpha + 1 =\frac{\alpha + \beta - \alpha \beta}{\alpha + \beta}$.

Since $G''(x)<0, \forall x \in [0,1]$ and $G'(0)=0$, then $G'(x)<0, ~ \forall x \in [0,1]$. 
Similarly, $G'(x)<0, \forall x \in [0,1]$ and $G(0)=0$, implies  $G'(x)<0, ~ \forall x \in [0,1]$, which in turn implies that $F$ satisfies Assumption \ref{ass:pseudo-mhr}.
\end{proof}

\subsection{Continuous distribution 
leading to a utility with two global maximizers }\label{app:par_2max}
Consider a distribution which cumulative distribution function $F$ is piece-wise linear on $[0,v]$ at least.
We consider that it changes slope at $a_1 v<v$, and that it is constant on $[a_2 v,v]$, as in Figure \ref{fig:_ex_cdf}.  We denote by $b_1= F(a_1 v)$ and $b_2 = F(a_2 v)$. For simplicity we assume that $F$ is constant on $[a_2 v, a_3 v]$ it is linear and does not change slope on $[a_3 v,1]$ with $a_3>1$.
We make the following assumptions
\begin{equation}\begin{cases}a_2 v>v/2, \\
     a_2 v \leq\frac{v+ a_1 v}{2} - \frac{a_2 v-a_1 v}{b_2-b_1} \frac{b_1}{2}.  
     \end{cases}\label{hyp:1}
\end{equation}

\begin{figure}[ht]
\centering
  \includegraphics[width=0.55\textwidth]{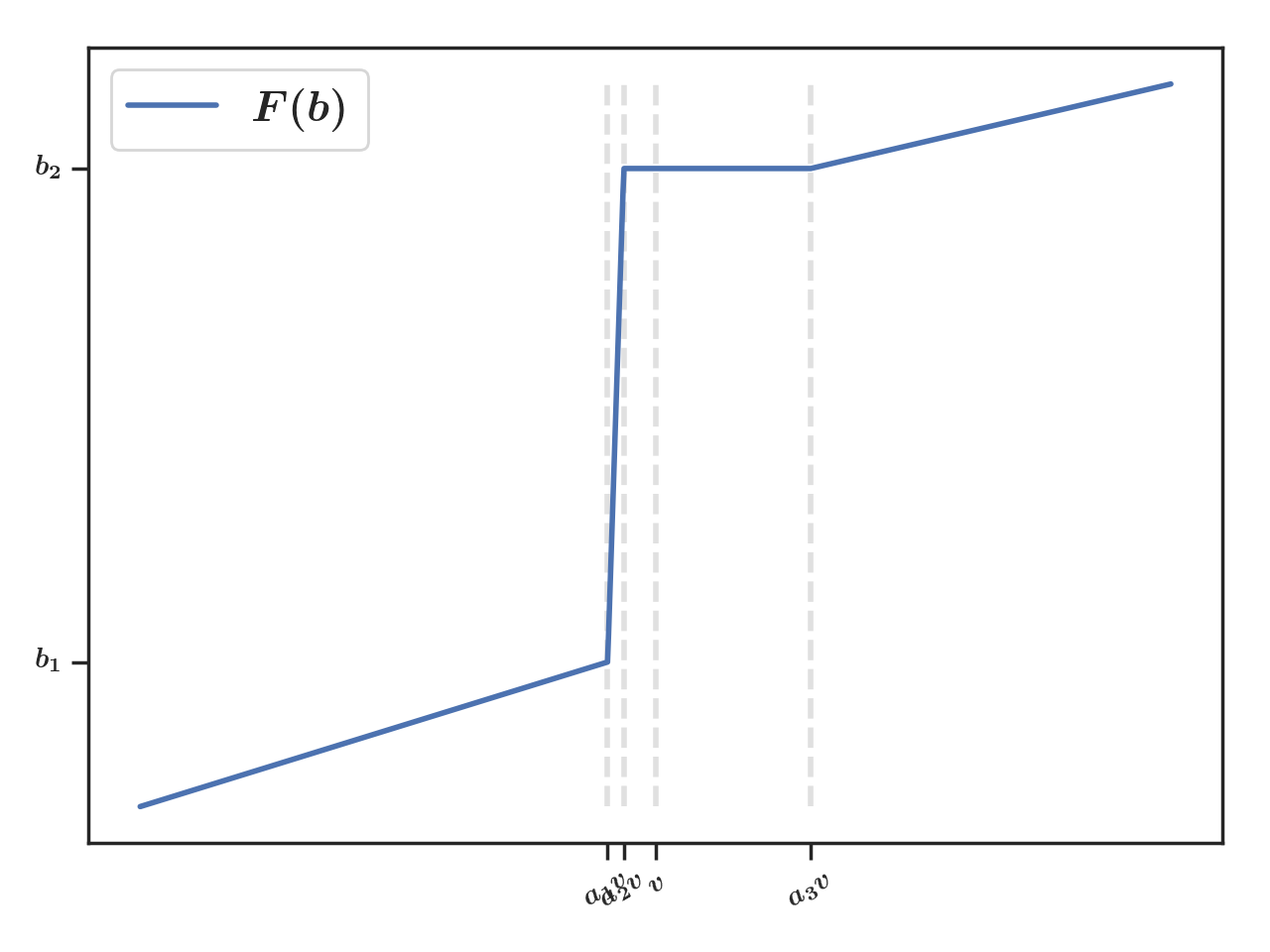}
  \caption{Example of $F$}
  \label{fig:_ex_cdf}
\end{figure}
Then 
\begin{itemize}
\item On $[0, a_1 v]$
$U_v(x)=\frac{b_1}{a_1 v}x$, 
and the optimum on  this interval is $v/2$.
The optimal value on  this interval  is $U_v(v/2)= \frac{b_1}{a_1 v}\frac{v^2}{4}$ on this interval.
\item On $[a_1 v, a_2 v]$, $U_v(x)= \left( \frac{b_2-b_1}{a_2 v-a_1 v}(x-a_1 v) + b_1 \right) (v-x)$,  and on this interval,
$U_v'(x) = \frac{b_2-b_1}{a_2 v-a_1 v}(v-2x+a_1 v) -b_1$ and 
$U'_v(x) = 0 \iff x= \frac{v+ a_1 v}{2} - \frac{a_2 v-a_1 v}{b_2-b_1} \frac{b_1}{2}.$ The optimizer on this interval is hence $a_2 v$, if $\frac{v+ a_1 v}{2} - \frac{a_2 v-a_1 v}{b_2-b_1} \frac{b_1}{2}>a_2 v$.
Under this condition, the optimal value is $U_v(a_2 v)= b_2(v-a_1 v)$ on this interval.
This can also be extended to the whole interval $[a_1 v, v]$, since U is decreasing after $a_2 v$.
\end{itemize}
Setting \begin{equation}
\frac{b_1}{a_1 }\frac{v}{4} = b_2 \label{eq:lb3}
\end{equation}
leads to the utility having two global maximizers, $v/2$ and $a_2 v$.

To summarize, the utility's argmax is $\{v/2,a_2 v\}$ if the set of Equations \ref{hyp:1} holds.

We can for example choose :

\begin{equation*}
    v = 1/2; ~ a_2 =\frac{15}{16};~ a_1 = \frac{29}{32};~b_2 = \frac{128}{29} b_1; b_1= 0.5
\end{equation*}

This choice of parameters satisfies Condition \ref{hyp:1} and Condition \ref{eq:lb3}.
Figure \ref{fig:_ex_utility} shows the corresponding utility on $[0,v]$.
\begin{figure}[h]
\centering
  \includegraphics[width=0.55\linewidth]{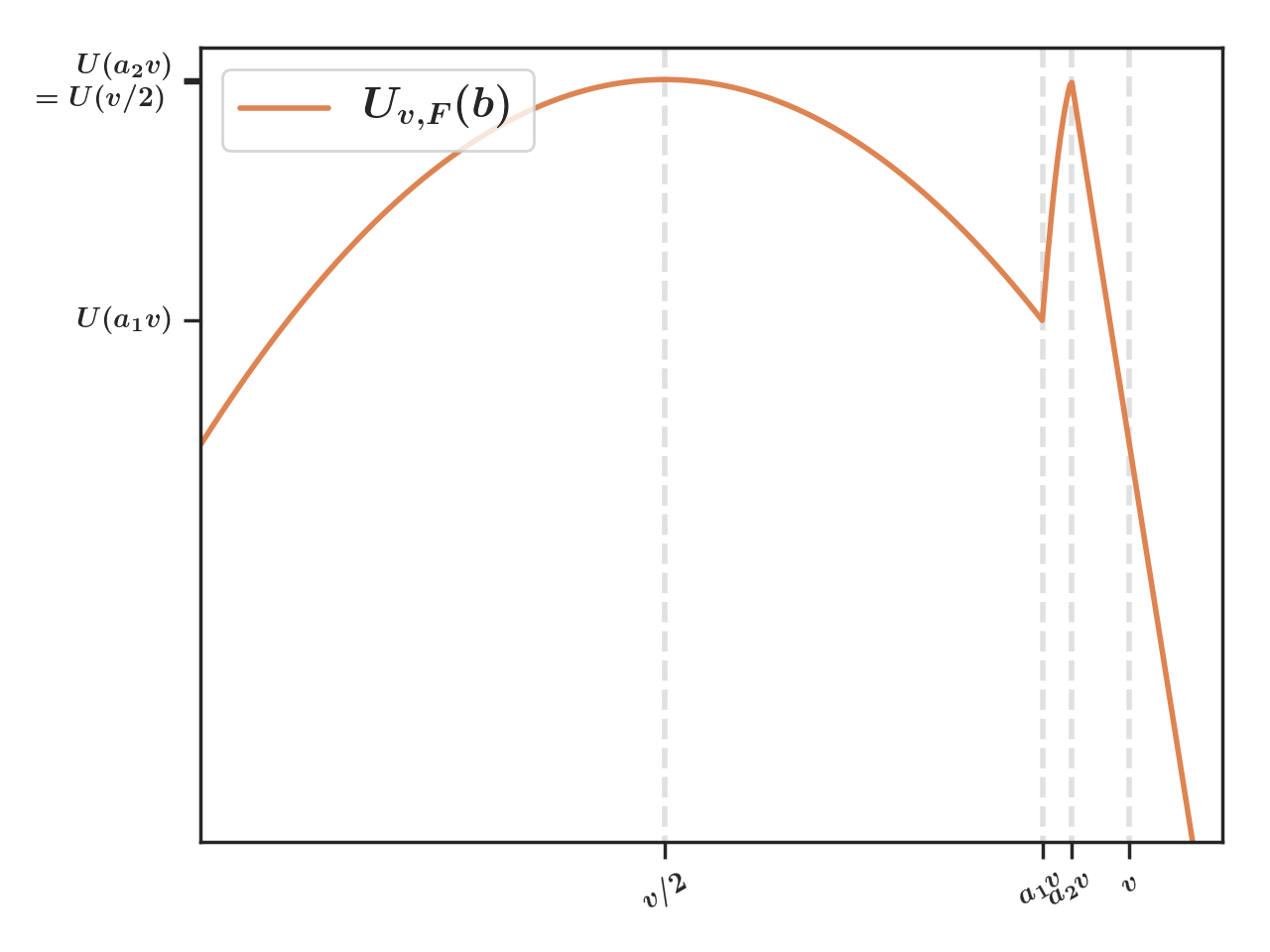}
  \caption{Associated Utility with two maximizers}
  \label{fig:_ex_utility}
\end{figure}

\section{Lower Bound}\label{sec:app_lower_bound}

\begin{repeatthm}{th:lower_bound_gen}
Let $\mathcal{C}$ denote the class of cumulative distribution functions on $[0,1]$. Any strategy, whether it assumes knowledge of $F$ or not, must satisfy
\begin{align*}
 \liminf_{T \rightarrow \infty} \frac{\max_{v \in [0,1], F \in \mathcal{C}}R_T^{v,F}}{\sqrt{T}}
 &\geq   \frac{1}{64}, 
\end{align*}
\end{repeatthm}
\begin{proof}

We exhibit a choice of $F$, and two alternative Bernoulli value distributions $Ber(v)$ and $Ber(v')$ that are difficult to distinguish but whose difference is large enough so that mistaking one for the other necessarily leads to a regret of the order of $\sqrt{T}$ when the cumulative distribution function is $F$.

Let $v<1$ and consider a discrete distribution  with support $\big\{\frac{v}{3}, \frac{2v}{3}, 1\big\}$ such that $F(\frac{v}{3}) = A$ and $F(\frac{2v}{3})= 2A + 3\frac{\Delta_T}{v}$, where $\Delta_T$ and $A$ are positive constants, that we will fix later on.  
A maximizer of the utility can only be a point of the support, since $U_{v,F}$ decreases in the intervals where $F$ is constant. It can not be $1$, because $v<1$. We have $U_{v,F}(\frac{v}{3}) = \frac{2vA}{3}$ and $U_{v,F}(\frac{2v}{3}) = \frac{2vA}{3} + \Delta_T$, while $U_{v,F}(1)\leq 0$.
Consequently,  when the value is $v$, the optimum is achieved by bidding $\frac{2v}{3}$ and bidding less than $\frac{2v}{3}$ yields a regret of at least $\Delta_T$.
Now let us consider the alternative situation in which the value is $v' = v - \delta_T$, with $\delta_T>0.$ We get $U_{v',F}(\frac{v}{3}) = \frac{2Av}{3} - \delta_T A$ and $U_{v',F}(\frac{2v}{3}) = \frac{2Av}{3} + \Delta_T - \delta_T(2A + \frac{3 \Delta_T}{v} )$.
When $\Delta_T < \delta_T(2A + \frac{3 \Delta_T}{v})$, the optimal bid is $\frac{v}{3}$ and the regret incurred by bidding more than $\frac{2v}{3}$ is at least $\delta_T(A + \frac{3 \Delta_T}{v} ) - \Delta_T$.
By setting $\Delta_T = \frac{A \delta_T}{2 - 3 \delta_T /v}$, we ensure that the regret incurred by bidding on the wrong side of $\frac{2v}{3}$ is larger than $\Delta_T$, whether the value is $v$ or $v'$.
Further, by setting $\delta_T = \sqrt{{v(1-v)}/{T}}$, we force the error $\Delta_T$ to be of the order of ${1}/{\sqrt{T}}$.

We also set $A =\frac{1}{4}$, and $v=1/2$.
We can prove that $\forall T>16$, 
 $2A +3 \frac{\Delta_T}{v}  < 1$ ; 
Indeed, if $T>16>(11/3)^2$,
$\frac{4}{3}<2\sqrt{T}-6$ hence $\frac{4}{3\sqrt{T}}<2 - \frac{6}{\sqrt{T}}$ which implies
$\frac{2}{3}\frac{\frac{1}{\sqrt{T}}}{2 - \frac{6}{\sqrt{T}}}= 6 \Delta_T<\frac{1}{2}= 1 -2A$.

We  denote by $\Po_{v,F}(\cdot)$ the probability of an event under the first configuration (respectively $\E_{v,F}(\cdot)$ the expectation of a random variable under the first configuration), and by $\Po_{v',F}(\cdot)$ the probability of an event under the second configuration (respectively $\E_{v- \delta-T, F}(\cdot)$ the expectation of a random variable under the first configuration).
We denote by $I_t$ the information collected up to time $t+1$ : $(M_t, V'_t, \ldots M_1, V'_1)$.
 $\Po_{v,F}^{I_t}$ (respectively $\Po_{v'}^{I_t}$) denotes the law of $I_t$ in the first (respectively second) configuration.

We consider the Kullback Leibler divergence between $\Po_{v,F}^{I_t}$ and $\Po_{v',F}^{I_t}$. We prove that it is equal to 
\begin{equation} KL(\Po_v^{I_t},\Po_{v',F}^{I_t})= kl(v,v') \E[N_t], \label{eq:kl_chain_rule}\end{equation}
where $kl(\cdot, \cdot)$ denotes the Kullback Leibler divergence between two Bernoulli distributions.
Indeed, thanks to  the chain rule for conditional KL, 
$$
KL(\Po_{v,F}^{I_t},\Po_{v',F}^{I_t})= KL(\Po_{v,F}^{I_t},\Po_{v',F}^{I_{t}})\\ + KL(\Po_{v,F}^{(M_t,V'_t)|I_{t}},\Po_{v',F}^{(M_t,V'_t)|I_{t}}),
$$
and
\begin{align*}
KL(\Po_{v,F}^{(M_t,V'_t)|I_{t}},\Po_{v',F}^{(M_t,V'_t)|I_{t}})&= \E[\E[KL(\nu_{I_t}\otimes \mathcal{D}_F, \nu'_{I_t}\otimes \mathcal{D}_F)|I_{t}]]\\
&= \E[kl(v, v')\1(B_t>M_t)].\\
\end{align*}
where $\nu_{I_t}$(respectively $\nu'_{I_t}$) denotes the  law of $V'_t$ knowing $I_t$ in the first configuration (respectively the second), and $\mathcal{D}_F$ the law of $M_t$.

By induction, we obtain $$ KL(\Po_{v,F}^{I_t},\Po_{v',F}^{I_t})= kl(v,v') \E_{v,F}[N_t].$$

 We stress that in either of the former configurations (under $(v,F)$ or $(v',F)$), playing on the wrong side of $\frac{2}{3} v$ yields a regret larger than $\Delta_T$.
 Using this, we get that $\forall T> 16$,
 
 \begin{align*}
 \max(R_T^{v,F}, R^{v',F}_T) & \geq \frac{1}{2}(R_T^{v,F}+ R^{v-\delta,F}_T)
 \\& \geq \frac{1}{2}\sum_{t=1}^T\left( \Delta_T \Po_{v,F}\left(B_t<\frac{2}{3} v\right) + \Delta_T \Po_{v', F}\left(B_t>\frac{2}{3} v\right)\right)\\
 & \geq \frac{1}{2}\sum_{t=1}^T\left( \Delta_T \Po_{v,F} \left(B_t<\frac{2}{3} v\right) + \Delta_T \left(1- \Po_{v',F}(B_t>\frac{2}{3} v)\right)\right)\\
  & \geq \frac{1}{2}\sum_{t=1}^T \Delta_T \left(1-  TV(\Po_{v,F}^{I_t}, \Po_{v',F}^{I_t})\right)\\
  & \geq \frac{1}{2}\sum_{t=1}^T \Delta_T \left(1-  \sqrt{\frac{1}{2}KL(\Po_{v,F}^{I_t}, \Po_{v',F}^{I_t})}\right)\\
  &\geq \frac{1}{2}\sum_{t=1}^T \Delta_T \left(1-  \sqrt{\frac{1}{2} \E_{v,F}[N_t]kl(v, v')}\right)
 \\
 &\geq \frac{1}{2}\sum_{t=2}^T \Delta_T \left(1-  \sqrt{\frac{1}{2} T kl(v, v')}\right)
 \end{align*}
 where we used Pinsker's inequality in the fifth inequality and where $TV(\cdot, \cdot)$ denotes the total variation. 
 Yet, since $kl(v,v')= \frac {(v'-v)^2}{2} \int_0^1 g''(v' + s(v' + s(v-v'))2(1-s)ds$, where $g(x) = kl(x,v')$ thanks to Taylor's inequality, 
 \begin{align*}
 kl(v,v')&\leq \frac {(v'-v)^2}{2} \int_0^1 2 \max_{u\in[v,v']} g''(u)ds\\
 &\leq (v'-v)^2 \frac{1}{\min_{u\in[v,v']}u(1-u)}\\
 &\leq \frac{(v'-v)^2}{v'(1-v')},
 \end{align*} 
 since $v= \frac{1}{2}$.
 
 Therefore, 
 
  \begin{align*}
 \max(R_T^{v,F}, R^{v',F}_T)  &
 \geq \frac{1}{2}\sum_{t=1}^T \Delta_T \left(1-  \sqrt{\frac{1}{2} T kl(v, v')}\right)\\
 &\geq \frac{1}{2}\sum_{t=1}^T \Delta_T \left(1-  \sqrt{\frac{1}{8} \frac{1}{(1/2 -\frac{1}{2\sqrt{T}})(1/2 +\frac{1}{2\sqrt{T}})}}\right)
 \\
 &\geq \frac{1}{2}\times \frac{A \delta_T}{2 - 3/2 \delta_T } T\left(1-  \sqrt{\frac{1}{8} \frac{1}{(1/2 -\frac{1}{2\sqrt{T}})(1/2 +\frac{1}{2\sqrt{T}})}}\right)
 \\
 &\geq\frac{1}{16 - 12 /\sqrt{T}} \sqrt{T} \left(1-  \sqrt{\frac{1}{8} \frac{1}{(1/2 -\frac{1}{2\sqrt{T}})(1/2 +\frac{1}{2\sqrt{T}})}}\right)
 \end{align*}
 
 Finally 
 \begin{align*}
 \liminf_{T \rightarrow \infty} \frac{\max(R_T^{v,F}, R^{v',F}_T)}{\sqrt{T}}
 &\geq   \frac{1}{16} \left(1-  \sqrt{\frac{1}{2}}\right)  \geq \frac{1}{64}
 \end{align*}
 
\end{proof}

\section{Preliminary Results}\label{sec:app_prelim}

\subsection{Concentration inequalities used for the upper bounds} \label{sec:concentration_ineq}

\subsubsection{ On the value $V_t$}\label{subsec:app_concentration_v}

\begin{lemma} \label{lem:concentration_V}The following concentration inequality on the values holds 
\begin{align*}
&\sum_{t=2}^T \Po
	\left(
		(\hat{V}_t-v)^2 \geq \frac{\gamma \log (t-1)}{2 N_t}
		\right) \leq 
\sum_{t=1}^T 2 e \sqrt{\gamma}( \log (t))t^{-\gamma}.
\end{align*}
\end{lemma}
\begin{proof}
We have, for all $\eta_{t-1}$,
\begin{align*} 
	\sum_{t=2}^T \Po
	\left(
		(\hat{V}(N_t)-v)^2 \geq \frac{\eta_{t-1}}{2 N_t}
		\right) 
		&\leq 
		\sum_{t=2}^T \Po \left(\exists m : ~1\leq m \leq t, 2m (\hat{V}(m)-v)^2 \geq 			\eta_{t-1} \right)
	 \\
	&\leq \sum_{t=1}^T 2 e \sqrt{\eta_{t-1} \log (t-1) } \exp(-\eta_{t-1}) := l_1(T)
\end{align*}
where the second inequality comes from Lemma 11 in \citep{cappe2013kullback}, and from the fact that $V_t$ is a positive random variable  bounded by 1, so $1/2-$ sub-Gaussian.

Therefore, if $\eta_t := \gamma \log t$,
\begin{align*}
l_1(T) = \sum_{t=2}^T 2 e \sqrt{\gamma}( \log(t-1) )(t-1)^{-\gamma}
\end{align*} which tends to a finite limit as soon as $\gamma> 1$.
\end{proof}
\subsubsection{On the cumulative distribution function of $M_t$}

\begin{lemma} \label{lem:concentration_F}  The following concentration inequality holds on the empirical cumulative distribution $\hat{F}_t$.
$$
		\sum_{t=2}^T \Po \left(
			\| \hat{F}_t-F\|_{\infty}\geq \frac{\gamma \log(t-1)}{2 (t-1)}
		\right) \\
		\leq 2 \sum_{t=1}^T t^{-\gamma} .$$
\end{lemma}
\begin{proof}
It holds
\begin{align*}
	&\sum_{t=2}^T \Po \left(
		(\max_{b\in[0,1]} |F_t(b)-F(b)|)^2 \geq \frac{\gamma \log(t-1)}{2(t-1)}
	\right)\\
		&\leq \sum_{t=2}^T \Po \left(
			\| \hat{F}_t-F\|_{\infty}^2\geq \frac{\gamma \log(t-1)}{2 (t-1)}
		\right)\\
		&\leq \sum_{t=1}^{T-1} 2 e^{-\frac{2\gamma  \log (t} {2t}}\\
		&\leq \sum_{t=1}^T 2 t^{-\gamma},
\end{align*}according to the Dvoretzky–Kiefer–Wolfowitz inequality (see \cite{massart1990tight}).\\
Note that this also  yields 
\begin{align*}
 \sum_{t=2}^T \Po \left(
			\| \hat{F}_t-F\|_{\infty}\geq \frac{\gamma \log(t-1)}{2 N_t}
		\right)
		&\leq \sum_{t=2}^T
		\Po \left(
			\| \hat{F}_t-F\|_{\infty}\geq \frac{\gamma \log(t-1)}{2 (t-1)}
		\right) \\
		&\leq 2 \sum_{t=1}^T t^{-\gamma} .
\end{align*}
\end{proof}

\subsubsection{Local concentration inequality}

This lemma is key for the proof of the upper bound of the regret of UCBid1+. It quantifies the variation of $\hat{F}_t$ on a small interval.
\begin{repeatlem}{lem:local_concentration_inequality}
For any $a,b\in [0,1]$, if $F$ is continuous and increasing, then 
\begin{multline}\sup_{a\leq x \leq b}|\hat{F}_t(x) - F(x) - (\hat{F}_t(a)- F(a))| \\\leq  \sqrt{\frac{2(F(b)- F(a))\log
\left( \frac{e \sqrt{t}}{\sqrt{2(F(b)- F(a))}\eta}\right)
}
{t}} + \frac{\log(\frac{t}{2(F(b)- F(a) \eta^2 })}{6 t},
\end{multline}
with probability $1-\eta$
\end{repeatlem}

Remark : it follows from the lemma that the the maximal gap between $\hat{F}_t(x) - F(x)$ and $\hat{F}_t(\frac{a+b}{2})- F(\frac{a+b}{2})$ can easily be bounded by : \begin{multline*}\sup_{a\leq x \leq b}|\hat{F}_t(x) - F(x) - (\hat{F}_t(\frac{a+b}{2})- F(\frac{a+b}{2}))|\\ \leq 2 \sqrt{\frac{2(F(b)- F(a))\log
\left( \frac{e \sqrt{t}}{\sqrt{2\eta(F(b)- F(a))}}\right)}
{t}} + 2 \frac{\log(\frac{t}{2(F(b)- F(a) \eta^2 })}{6 t}
\end{multline*} with probability $1-\eta$.

\textbf{Proof}:

Let $X_1,\dots,X_n\stackrel{iid}\sim dF$. 
Let $m>2$
For every $1\leq i\leq m$, let $x_i$ be such that
\[F(x_i) = F(a) + \frac{i}{m}\big(F(b)-F(a)\big)\;. \]
By Bernstein's inequality, since $t\big(\Fn(x_i)-\Fn(a)\big) \sim \mathcal{B}(n, F(x_i)-F(a))$ has a variance bounded by $t\big(F(b)-F(a))$, there is an event $A$ of probability at least $1- m e^{-z}$ on which 
\begin{align*} 
\max_{0\leq i\leq m} \big| \Fn(x_i)-\Fn(a) -(F(x_i)-F(a)) \big| \leq \sqrt{\frac{2\big(F(b)-F(a)\big)z}{t}}  + \frac{z}{3t} := \delta,
\end{align*}
by a union bound.
Besides, for $i=0$, $\Fn(x_i)-\Fn(a) -(F(x_i)-F(a))=0$.

On this event, for every $x_{i-1}\leq x \leq x_i$:
\begin{align*}
\Fn(x)-\Fn(a) -(F(x)-F(a))  & \leq \Fn(x_i)-\Fn(a)-(F(x_i)-F(a)) + F(x_i) - F(x) \leq \delta + \frac{1}{m},\\
\Fn(x)-\Fn(a) -(F(x)-F(a))  &\geq \Fn(x_{i-1})-\Fn(a)-(F(x_{i-1})-F(a)) + F(x_{i-1}) - F(x) \\
&\geq -\delta - \frac{1}{m}\;.
\end{align*}
and hence
\[\sup_{a\leq t \leq b} \big| \Fn(x)-\Fn(a) -(F(x)-F(a)) \big|  \leq \sqrt{\frac{2\big(F(b)-F(a)\big)z}{t}}  + \frac{z}{3t} + \frac{1}{m} \;.\]

Now, take \[m= \Big \lceil \sqrt{\frac{t}{2\big(F(b)-F(a)\big)}} \Big \rceil \]  and $z = \log(m/\eta) $: one gets that with probability at least $1-\eta$, 
\begin{align*}
\sup_{a\leq t \leq b}& \big| \Fn(x)-\Fn(a) -(F(x)-F(a)) \big| \\& \leq
 \sqrt{\frac{2\big(F(b)-F(a)\big)\log\left( \frac{\sqrt{\frac{t}{2(F(b)-F(a))}}}{\eta}\right)}{t}}  + \frac{\log\left( \frac{\sqrt{\frac{t}{2(F(b)-F(a))}}}{\eta}\right)}{3t} + \sqrt{\frac{2\big(F(b)-F(a)\big)}{t}}\\
 &\leq    \sqrt{\frac{2\big(F(b)-F(a)\big)\log\left(\frac{e\sqrt{t}}{\sqrt{2(F(b)-F(a))}\eta}\right)}{t}} + \frac{\log\left( \frac{t}{2(F(b)-F(a))\eta^2}\right)}{6t}\;.
\end{align*}

\subsection{General bound on the instantaneous regret}\label{subsec:instan_regret_app}
In the following, we will repeatedly use the following general bound on the instantaneous regret conditioned on the past and on a current victory. 
\begin{lemma}\label{lem:conditional_exp}
Let $A$ be an $\mathcal{F}_{t-1}$-measurable event. Let $S_t$ denote $(V_t-b^*)\1(M_t < b^*) - (V_t-B_t)\1(M_t < B_t) $.
The following inequality holds:
\begin{equation*}
\E\left[S_t \1(B_t>b^*)\1(A)| \mathcal{F}_{t-1} \vee \sigma(\1(B_t>M_t))\right]\leq \frac{U(b^*)- U(B_t)}{F(b^*)} \1(M_t\leq B_t)\1(A).
\end{equation*}
\end{lemma}
\begin{proof}
When $B_t>b^*$, the instantaneous regret can be decomposed as follows 
\begin{equation}
S_t \1(B_t>b^*) =(B_t-v)\1(M_t\leq b^*) \1(B_t>b^*)+ (B_t - b^*)\1 \left\{(M_t \leq b^* \leq B_t)\right\}.
\end{equation}
Note that in particular, there is no instantaneous regret when $M_t> B_t$.
Therefore
\begin{align*}
&\E\left[S_t \1(B_t>b^*) \1(A)| \mathcal{F}_{t-1} \vee \1(B_t>M_t)\right]\\
& \leq \frac{(B_t-b^*)F(b^*) + (B_t-v)(F(B_t)-F(b^*))}{F(B_t)}
\1(M_t\leq B_t) \1(B_t>b^*) \1(A)\\
&\leq \frac{U(b^*)- U(B_t)}{F(b^*)} \1(M_t\leq B_t) \1(A),
\end{align*}
since $U(b^*)-U(B_t)= (v-b^*)F(b^*)- (v-B_t)F(B_t)$, which also equals $(B_t-b^*)F(b^*) + (B_t-v)(F(B_t)-F(b^*))$.
\end{proof}

\subsection{Other lemmas}

\begin{lemma}\label{lem:sum_1_n}
The expectations $\E\left[\sum_{t=2}^T\frac{1}{N_t} \1\{M_t\leq B_t\}\right]$ and $ \E\left[\sum_{t=2}^T \sqrt{\frac{1}{N_t}} \1\{M_t\leq B_t\}\right]$  can always be bounded as follows
$$\begin{cases}\E\left[\sum_{t=2}^T \frac{1}{N_t} \1\{M_t\leq B_t\}\right]\leq 1+\log T,\\
 \E\left[\sum_{t=2}^T  \sqrt{\frac{1}{N_t}} \1\{M_t\leq B_t\}\right] \leq 1+ \sqrt{T}.\end{cases}$$
\end{lemma}

\begin{proof}
Since winning an auction increments the number of observations $N_t$ by 1,
\begin{align*}
\sum_{t=2}^T \E\Big[ \sqrt{\frac{1}{N_{t}}} \1(M_t \leq B_t)\Big]
&\leq \sum_{t=2}^T \sum_{n=1}^{T-1}  \sqrt{\frac{1}{n}}  \1\{N_t= n,~ N_{t+1} =n+1\}\\
&\leq \sum_{n=1}^{T-1}  \sqrt{\frac{1}{ n}} \sum_{t=2}^T \1\{N_{t}= n,~ N_{t} =n+1\}\\
&\leq \sum_{n=1}^{T-1}  \sqrt{\frac{1}{ n}}\\
&\leq 1 + \sum_{n=2}^{T-1} \int_{n-1}^n \sqrt{\frac{1}{ u}}du\\
&\leq 1+\sqrt{T}.
\end{align*}

Similarly, we get
\begin{align*}
\sum_{t=2}^T \E\Big[ \frac{1}{N_{t}} \1(M_t \leq B_t)\Big]
&\leq \sum_{t=2}^T \sum_{n=1}^{T-1}  \frac{1}{n}  \1\{N_t= n,~ N_{t+1} =n+1\}\\
&\leq \sum_{n=1}^{T-1}  \frac{1}{ n} \sum_{t=2}^T \1\{N_{t}= n,~ N_{t} =n+1\}\\
&\leq \sum_{n=1}^{T-1}  \frac{1}{ n}\\
&\leq 1 + \sum_{n=2}^{T-1} \int_{n-1}^n \frac{1}{ u}du\\
&\leq 1+\log{T}.
\end{align*}
\end{proof}

\begin{lemma}\label{lem:dist_U_max}
If $g_1$ and $g_2$ are two functions such that $\|g_1 - g_2\|_{\infty} \leq \delta$, then 
$$g_1(b_1^*) - g_1(b^*_2) \leq 2 \delta$$
 where $b_1^* = \max(\argmax_{b \in [0,1]} g_1(b))$
 and $b_2^* = \max(\argmax_{b \in [0,1]} g_2(b))$.
\end{lemma}

\begin{proof}
Indeed, 
\begin{align*}
0\leq g_1(b^*_1) - g_1(b^*_2) &\leq g_1(b_1^*) - g_2(b^*_2) + g_2(b^*_2) -  g_1(b^*_2) \\
 & \leq  2 \delta.
\end{align*}
\end{proof} 

\begin{lemma}\label{lem:log}
For any $a>0$, 
 $t\geq2 a \log(a)$ implies $t\geq a \log t$.
\end{lemma}
\begin{proof}
\begin{align*}
    a \log t &\geq a \left(\frac{t}{2a} + \log(2a)\right)\\
    &\geq t/2 + a \log (a),
    \end{align*}
where the first inequality follows from the fact that $\log(x/y)\leq x/y$ for any positive $x$ and $y$.
Hence when $t>2 a \log(a)$, 
$t\geq t/2+ a \log t \geq a \log t.$
\end{proof}

\section{Known $F$}\label{sec:app_known_F}
\subsection{Upper Bounds of the Regret of UCBid1}
We prove the somewhat more precise form of Theorem \ref{th:FPUCBID_general}.
\begin{repeatthm}{th:FPUCBID_general}
UCBid1 incurs a regret bounded as follows
$$R_T \leq  \frac{1 }{ F(b^*)} \sqrt{\gamma \log T }(\sqrt{T}+1) + O(1).$$
\end{repeatthm}
\begin{proof}
We denote by $U^{UCBid1}_t$ the function $b \mapsto (\hat{V}_t + \epsilon_t -b)F(b)$.
The regret can be decomposed as follows. 
\begin{align*}
R_T&\leq 1 + \sum_{t=2}^T\Po \left(|\hat{V}_{t}-v|\geq \epsilon_{t} \right) + \sum_{t=2}^T\E \left[S_t \1\left\{|\hat{V}_{t}-v|\leq \epsilon_{t} \right\}\right],
\end{align*}
Lemma \ref{lem:concentration_V} yields the following bound on the probability of over-estimating $\hat{V}_t$: $$\sum_{t=2}^T\Po(|\hat{V}_{t}-v|\geq \epsilon_{t}) \leq \sum_{t=1}^{t}2 e \sqrt{\gamma}( \log t )t^{-\gamma}.$$

Since $F(x)\leq 1, \forall x\in[0,1]$, and $\|U^{UCBid1}_{t}-U\|_{\infty} = \|(\hat{V}_{t}-v+ \epsilon_{t})F(x)\|_{\infty}\leq |\hat{V}_{t}-v+ \epsilon_{t}|$, we can bound the difference between the utility function and its (upper confidence) estimate with high probability:  $$\sum_{t=2}^T\Po(\|U^{UCBid1}_{t}- U\|_{\infty} \geq 2\epsilon_{t}) \leq \sum_{t=1}^{T}2 e \sqrt{\gamma}( \log t )t^{-\gamma}.$$

When $\|U^{UCBid1}_{t}- U\|_{\infty} \leq 2 \epsilon_{t}$,
then 
\begin{align*}
|U(b^*) - U(B_t)| \leq 4 \epsilon_{t},
\end{align*}
thanks to Lemma \ref{lem:dist_U_max}.
Additionally, using Lemma \ref{lem:psi_F}, if $\hat{V}_{t}+\epsilon_{t}-v\geq 0$ , then  $B_t\geq b^*$
Therefore, 
\begin{align*}
&\sum_{t=2}^T\frac{1}{F(b^*)}\E\left[S_t \1\left\{M_t\leq B_t\right\} \1\left\{b^*\leq B_t\right\}\1\left\{|\hat{V}_{t}-v|\leq \epsilon_{t}\right\}\right]\\
&\leq  \sum_{t=2}^T\E \left[ \frac{U(b^*) - U(B_t)}{F(b^*)} \1\left\{b^*\leq B_t\right\}\1\left\{M_t\leq B_t\right\} \1\left\{|\hat{V}_{t}-v|\leq \epsilon_{t}
\right\} \right]\\
&\leq  \sum_{t=2}^T\E \left[ \frac{U(b^*) - U(B_t)}{F(b^*)} \1\left\{b^*\leq B_t\right\}\1\left\{M_t\leq B_t\right\} \1\left\{U(b^*)-U(B_t)\leq 4 \epsilon_{t} \right\} \right]\\
&\leq \sum_{t=2}^T\frac{1}{F(b^*)}\E\left[4 \epsilon_{t} \1\left\{M_t\leq B_t\right\} \1\left\{(U(b^*)-U(B_t)\leq 4 \epsilon_{t}\right\}\right]\\
&\leq
\sum_{t=2}^T \frac{1}{F(b^*)}\sqrt{ 2\frac{\gamma \log T}{ N_{t}}}\\
&\leq
\frac{1}{F(b^*)} \sqrt{2 \gamma \log T} (1+\sqrt{T}),
\end{align*}
where the second inequality comes from Lemma \ref{lem:conditional_exp} (in fact $\left\{|\hat{V}_{t}-v|\leq \epsilon_{t}\right\}$ is $\mathcal{F}_{t-1}$-measurable) and the last inequality comes from Lemma \ref{lem:sum_1_n}.

Using Lemma \ref{lem:concentration_V} yields $$\sum_{t=2}^T\Po(|\hat{V}_{t}-v|\geq \epsilon_{t}) \leq  \sum_{t=1}^T2 e \sqrt{\gamma}( \log t )t^{-\gamma} .$$
Combining  this with the above decomposition of the regret yields 
$$R_T\leq1 + \sum_{t=1}^T2 e \sqrt{\gamma}( \log t )t^{-\gamma} +
\frac{1}{F(b^*)} \sqrt{2\log T} (1+\sqrt{T}),$$
When $\gamma>1$, $\sum_{t=1}^T2 e \sqrt{\gamma}( \log t )t^{-\gamma}$ tends to a  constant, and
$$R_T \leq \frac{1}{F(b^*)} \sqrt{2 \gamma \log T} (1+\sqrt{T}) + O(1),$$ which concludes the proof.
\begin{repeatthm}{th:FPUCBID_pseudo_mhr}
If   $F$ satisfies Assumption \ref{ass:pseudo-mhr} and \ref{ass:lambda_pseudo-mhr}, then
$$
 R_T\leq  \frac{2\gamma \lambda C_f^2}{F(b^*) c_f}\log^2(T)  + O(\log T), $$
 when $\gamma>1.$
\end{repeatthm}

\begin{proof} 

Thanks to Lemma \ref{lem:psi_F}, if $\hat{V}_{t}+\epsilon_{t}-v\geq 0$ , then  $B_t\geq b^*$. Additionally,
$$B_t-b^*\leq (\hat{V}_{t}+ \epsilon_{t} - v),$$
thanks to Lemma \ref{lem:lipschitz}.
In particular, if $ \hat{V}_{t}+\epsilon_{t} -v< 2 \epsilon_{t}$, 
$$B_t-b^*\leq  2 \epsilon_{t}.$$

The regret can therefore be decomposed as follows :
\begin{multline}\label{eq:regret_decomposition} R_T \leq 1 +\sum_{t=2}^T \Po(\hat{V}_{t}+\epsilon_{t} -v \leq 0) +\sum_{t=2}^T \Po(\hat{V}_{t}-\epsilon_{t} -v \geq 0)\\ +\E\left[\sum_{t=2}^T  S_t \1(B_t\in\left[b^*, b^*+ \min(2 \epsilon_{t}, \Delta )\right]\right] + \sum_{t=2}^T   \E\left[S_t \1(B_t\in\left[b^*+ \min(2 \epsilon_{t}, \Delta), b^* + \Delta \right])\right]
\end{multline}

Let us bound the third term of this inequality.
Thanks to Lemma \ref{lem:conditional_exp} ,
\begin{multline}
 \E\left[S_t \1(B_t\in\left[b^*, b^*+\epsilon_{t} \right])|\mathcal{F}_{t-1}\vee \sigma( \1\left\{M_t\leq B_t\right\}) \right]\\
 \leq \frac{U(b^*) - U(B_t )}{F(b^*)}
 \times\1\left\{M_t\leq B_t\right\} \1\left\{b^*\leq B_t \leq b^* + 2  \epsilon_{t}\right\},
\end{multline}
because $(B_t\in\left[b^*, b^*+\epsilon_{t}\right])$ is $\mathcal{F}_{t-1}$- measurable.
This is why 
\begin{align*}
&\sum_{t=2}^T\E\left[ \E\left[ S_t \1(B_t\in[b^*, b^*+ \min(2 \epsilon_{t}, \Delta  ])|\mathcal{F}_{t-1}\vee \sigma (\left\{\1\left\{M_t\leq B_t\right\})\right\} \right]\right]\\
&\leq \sum_{t=2}^T \E\left[\frac{U(b^*) - U(B_t )}{F(b^*)}
 \times\1\left\{M_t\leq B_t\right\} \1\left\{b^*\leq B_t \leq b^* + \min(2 \epsilon_{t}, \Delta \right\}\right]\\
 &\leq\sum_{t=2}^T \E\left[\frac{W(q^*) - W(Q_t )}{F(b^*)}
 \times\1\left\{M_t\leq B_t\right\} \1\left\{q^*\leq Q_t \leq b^* + 2 C_f  \epsilon_{t}\right\}\right]\\
  &\leq\sum_{t=2}^T \E\left[\frac{\lambda(q^* - Q_t )^2}{c_fF(b^*)}
 \times\1\left\{M_t\leq B_t\right\} \1\left\{q^*\leq Q_t \leq b^* + 2 C_f  \epsilon_{t}\right\}\right]\\
&\leq \E\left[ \frac{\lambda (2  C_f)^2}{ c_fF(b^*)} \sum_{t=2}^T \left(\frac{\gamma \log T }{2N_t}\right) \1\left\{M_t \leq B_t\right\}\right]\\
  &\leq \frac{ 2 \lambda  \gamma \bar{C_f}}{ c_f F(b^*)}  \log T(\log T +1),
\end{align*}
where the third inequality comes from Lemma \ref{lem:sub_quadratic} and the last one follows from Lemma \ref{lem:sum_1_n}.

Thanks to Lemma \ref{lem:concentration_V}, the sum of the first  term  and the second term  of  Equation (\ref{eq:regret_decomposition}) can be bounded by $\sum_{t=2}^T \Po(\hat{V}_{t} -v< \epsilon_{t})+ \sum_{t=2}^T \Po(\hat{V}_{t}-\epsilon_{t} -v \geq 0) \leq \sum_{t=1}^T  e \sqrt{\gamma} \frac{\log t}{t^{\gamma }}$ which is bounded by a constant when $\gamma >1$.

The last term of Equation (\ref{eq:regret_decomposition}) can be bounded as follows:
\begin{align*}
    \sum_{t=2}^T   \E\left[S_t\1(B_t\in\left[b^*+ \min(2 \epsilon_{t}, \Delta), b^* + \Delta \right])\right]
    &\leq \sum_{t=2}^T   \Po\left[(\Delta > 4 \epsilon_{t}, M_t \leq B_t,B_t>b^* \right]\\
    &\leq \sum_{t=2}^T   \Po\left[\Delta^2 > 4 \frac{\gamma \log T}{2 N_{t}}, M_t \leq B_t, B_t>b^* \right]\\
    &\leq \sum_{t=2}^T \sum_{n=1}^{T-1}  \Po\left[\Delta^2 > 2 \frac{\gamma \log T}{2 N_{t}}\right]\1\left[N_{t}= n,~ N_{t+1} =n+1 \right]\\
    &\leq \sum_{n=1}^{T-1}  \1\left[n<4\frac{\gamma \log T}{2 \Delta^2}\right]  \sum_{t=2}^T\1\left\{N_{t}= n,~ N_{t+1} =n+1 \right\}\\
    &\leq \sum_{n=1}^{T-1}  \1\left\{n<4\frac{\gamma \log T}{2 \Delta^2}\right\}\\
    & \leq  4\frac{\gamma \log T}{2 \Delta^2}
\end{align*}
where the first inequality comes from the fact that when $B_t>b^*$, a positive instantaneous regret can  only occur if $M_t \leq B_t$.
By summing all components of the regret, 

$$R_T\leq 1 + 4\frac{\gamma \log T}{2 \Delta^2} + \frac{2\gamma \lambda C_f^2}{F(b^*) c_f}(\log^2(T)+ \log T).$$
In conclusion,
\begin{align*} R_T\leq  \frac{2\gamma \lambda C_f^2}{F(b^*) c_f}\log^2(T)  + O(\log T)  \\
\end{align*} when $\gamma>1.$
\end{proof}


\subsection{Lower bound of the regret of optimistic strategies}\label{subsec:app_lower_bound_known_F}

\begin{repeatlem}{th:parametric_lower_bound}
Consider all environments where $V_t$ follows a Bernoulli distribution with expectation $v$ and $F$ satisfies Assumption \ref{ass:pseudo-mhr} and is such that $\phi' \leq \lambda$, and there exists $c_f$ and $C_f$ such that $0<c_f<f(b)<C_f, ~ \forall b \in [0,1]$.
If a strategy is such that, for all such environments,
$R_T^{v,F}\leq O(T^{a})$, for all $a>0$, 
and there exists $\gamma >0$ such that $\Po(B_t< b^*)<t^{-\gamma}$,
 then this strategy must satisfy:
\begin{align*}
& \liminf_{T \rightarrow \infty}  \frac{R_T^{v,F}}{\log T}\geq c_f^2 \lambda ^2\left(\frac{v(1-v)(v- b^*_{v,F})}{32} \right).
\end{align*}
\end{repeatlem}
Note that this proof is an adaptation of the proof of the parametric lower bound of \citep{achddou2021efficient}.

\begin{lemma}\label{lem:limit_Nt}
If $~ R_T \leq O(T^{a}), ~ \forall a>0,$ and $F$ admits a density which is lower bounded by a positive constant and upper bounded. 
Then, $$ \lim_{t \rightarrow \infty} \E \left[\frac{N_t}{t}\right] = F(b^*).$$ 
\end{lemma}

\begin{proof}
The fraction of won auctions is  $\E\left[\frac{N_t}{t}\right]= \E[\frac{1}{t}\sum_{s=1}^t F(B_s]$, by the tower rule.
Since $F$ admits a density $f$, upper bounded by a constant  $C_f$, 
\\
$$\E[(F(B_t) - F(b^*))^2]]\leq C_f^2 \E[(B_t-b^*)^2] .$$
The consistency assumption implies $\sum_{t=1}^{T} \E[(B_t- b^*)^2]\leq O(T^{a}), ~ \forall a>0,$ because of Lemma \ref{lem:quadratic}. In particular  $\lim_{t \rightarrow \infty}\E[(B_t- b^*)^2] = 0 $.
Combining the two previous arguments yields $\lim_{t \rightarrow \infty}\E[(F(B_t) - F(b^*))^2]=0$. Then, because $L_2$-convergence implies $L_1$-convergence, $\lim_{t \rightarrow \infty}\E[F(B_t)] = F(b^*)$. \\
Together with  the equality $\E\left[\frac{N_t}{t}\right]= \E[\frac{1}{t}\sum_{s=1}^t F(B_s)]$, and with the Cesaro theorem, this result proves  suffices to prove the lemma.
\end{proof} 

We set a time step $t\in [1,T]$.	We consider two alternative configurations with identical distributions for $M_t$ but that differ by the distribution of $V_t$. The value $V_t$ is distributed according to a Bernoulli distribution of expectation $v$ in the first configuration, respectively $v'_t = v + \sqrt{\frac{v(1-v)}{F(b^*)t}}$, in the second configuration.

\paragraph{Notation.}
We let $\Po_v(\cdot)$ denote the probability of an event under the first configuration (respectively $\E_v(\cdot)$ the expectation of a random variable under the first configuration), whereas $\Po_{v'_t}(\cdot)$ denotes the probability of an event under the second configuration (respectively $\E_{v'_t}(\cdot)$ the expectation of a random variable under the first configuration).
The information collected up to time $t+1$ is denoted $I_t$ : $(M_t, V'_t, \ldots M_1, V'_1)$.
Finally, $\Po_v^{I_t}$ (respectively $\Po_{v'_t}^{I_t}$) is the law of $I_t$ in the first (respectively second) configuration.

The Kullback Leibler divergence between $\Po_v^{I_t}$ and $\Po_{v'_t}^{I_t}$ can be proved to satisfy 
$$ KL(\Po_v^{I_t},\Po_{v'_t}^{I_t})= kl(v,v'_t) \E[N_t],$$
exactly like in Equation \ref{eq:kl_chain_rule}.

Using Lemma \ref{lem:limit_Nt},
$\forall \epsilon>0, \exists t_1(\epsilon), \forall t\geq t_1(\epsilon) $, 
$$ KL(\Po_v^{I_t},\Po_{v'_t}^{I_t}) \leq kl(v,v'_t)(1+\epsilon)F(b^*).$$
Using the data processing inequality (see for example \citet{garivier2019explore}), we get
\begin{align*}
 KL(\Po_v^{I_t},\Po_{v'_t}^{I_t})
& \geq kl \left(\Po_v \left( B_t > \frac{b^*_{v,F} + b^{*}_{v'_t,F}}{2}\right), \Po_{v'_t} \left( B_t > \frac{v+ b^{*}_{v'_t,F}}{2}\right)\right)\\
& \geq  2 \left(\Po_v \left( B_t > \frac{b^*_{v,F} + b^{*}_{v'_t,F}}{2}\right)- \Po_{v'_t} \left( B_t > \frac{b^*_{v,F} + b^{*}_{v'_t,F}}{2}\right)\right)^2\\
& \geq  2 \left(\Po_v \left( B_t > \frac{b^*_{v,F} + b^{*}_{v'_t,F}}{2}\right)+ \Po_{v'_t} \left( B_t \leq \frac{b^*_{v,F} + b^{*}_{v'_t,F}}{2}\right)-1\right)^2,
\end{align*}
where the  second inequality comes from Pinsker inequality.
Consequently, we get 
\begin{align*}
\Po_v \left( B_t > \frac{b^*_{v,F}+ b^{*}_{v'_t,F}}{2}\right)+ \Po_{v'_t} \left( B_t \leq \frac{b^*_{v,F} + b^{*}_{v'_t,F}}{2}\right)
 \geq  1 - \sqrt{\frac{1}{2} KL(\Po_v^{I_t},\Po_{v'_t}^{I_t})}.
\end{align*}
Specifically, $\forall t>t_0(\epsilon)$,
\begin{align*}
\Po_v \left( B_t > \frac{b^*_{v,F} + b^{*}_{v'_t,F}}{2}\right)+ \Po_{v'_t} \left( B_t \leq \frac{b^*_{v,F} + b^{*}_{v'_t,F}}{2}\right)
\geq  1 - \sqrt{\frac{1}{2}kl(v,v'_t)(1+\epsilon)F(b^*_{v,F})t}.
\end{align*}

Using  the fact that $\E_v[(B_t-b^*_{v,F})^2]\geq  \left(b^*_{v,F}-\frac{b^*_{v,F}+ b^{*}_{v'_t,F}}{2}\right)^2\Po_v \left( B_t > \frac{b^*_{v,F}+ b^{*}_{v'_t,F}}{2}\right)$ yields
\begin{align*}
\E_v[(B_t-b^*_{v,F})^2]&\geq  \left(\frac{b^*_{v,F}- b^{*}_{v'_t,F}}{2}\right)^2 \Po_v \left( B_t > \frac{b^*_{v,F}+ b^{*}_{v'_t,F}}{2}\right)\\
&\geq  \left(\lambda \frac{v- v'_t}{2}\right)^2 \Po_v \left( B_t > \frac{b^*_{v,F}+ b^{*}_{v'_t,F}}{2}\right)\\
& \geq \lambda^2 \frac{v(1-v)}{4F(b^*_{v,F})t} \left( 1- \sqrt{ \frac{1}{2}(1 + \epsilon)  kl(v, v'_t)F(b^*_{v,F})t}   - 1/ {t^{\gamma}}\right),
\end{align*}
where the second inequality comes from the fact that $v = \phi_F(b^{*}_{v,F})$ (resp. $v'_t = \phi_F(b^{*}_{v'_t,F})$) and that $\phi_F'\leq \lambda$
and the the second inequality stems from the assumption that the algorithm outputs a bid that does not underestimate $b^{*}_{v'_t,F}$ with high probability:  $\Po_{v'_t}(B_t< b^{*}_{v'_t,F})<\frac{1}{t^{\gamma}}$.

We use the fact that
$\forall \epsilon>0, ~\exists t_2(v, \epsilon), ~\forall t\geq t_2(v,\epsilon),~  kl\left(v, v + \sqrt{\frac{v(1-v)}{F(b^*_{v,F})t}}\right) \leq \frac{1 + \epsilon}{2F(b^*_{v,F})t}  $
which is proved by observing that $kl(v,v')= \frac {(v'-v)^2}{2} \int_0^1 g''(v' + s(v' + s(v-v'))2(1-s)ds$, where $g(x) = kl(x,v')$; and that thanks to Taylor's inequality, 
 \begin{align*}
 kl(v,v')&\leq \frac {(v'-v)^2}{2} \int_0^1 2 \max_{u\in[v,v']} g''(u)ds\\
 &\leq (v'-v)^2 \frac{1}{\min_{u\in[v,v']}u(1-u)}
\end{align*}

and that $\forall \epsilon>0, ~\exists t_2(v, \epsilon)$, such that $\min_{u\in[v,v']}u(1-u)<\frac{1+\epsilon}{v(1-v)}$.
Putting all the pieces together yields \\
$\forall t\geq \max( t_1(\epsilon),t_2(v,\epsilon)), $

\begin{align*}
\E_v[(B_t-b^*_{v,F})^2] \geq \frac{v(1-v)}{4F(b^*_{v,F})t} \left( 1- \sqrt{ \frac{1}{4}(1 + \epsilon)^2  }   - 1/ t^{\gamma}\right).
\end{align*}
Let $t_0(v,\epsilon) =\max( t_1(\epsilon),t_2(v,\epsilon)).$
We obtain
\begin{align*}
\sum_{t= 1}^T \E_v[(B_t-b^*_{v,F})^2] \geq \sum_{t= t_0(v,\epsilon)}^T \lambda^2 \frac{v(1-v)}{4F(b^*_{v,F})t} \left( 1-  \frac{1}{2}(1 + \epsilon)   - 1/ t^{\gamma}\right).
\end{align*}

Recall that, according to Lemma \ref{lem:quadratic}, 
$$R_T(v) = \sum_{t=1}^T \E\left[ U(b^*_{v,F}) - U(B_t) \right]\geq \frac{U(b^*_{v,F})}{4} \sum_{t= 1}^T \E_v[(Q_t-q^*)^2] \geq \frac{c_f^2 U(b^*_{v,F})}{4} \sum_{t= 1}^T \E_v[(B_t-b^*_{v,F})^2] .$$

Hence, $\forall \epsilon>0,$

$$
 R_T(v)
 \geq \lambda^2  \frac{c_f^2 U(b^*_{v,F})}{4} \left(\frac{v(1-v)}{4} \left( 1-  \frac{1}{2}(1 + \epsilon)\right)\right) \log \frac{T}{t_0(v,\epsilon)}  - O(1).
$$

And $\forall \epsilon>0,$
\begin{align*}
& \liminf_{T \rightarrow \infty}  \frac{R_T(v)}{\log T}\geq \frac{c_f^2 \lambda^2  U(b^*_{v,F})}{4}  \left(\frac{v(1-v)}{4F(b^*_{v,F})} \left( 1-  \frac{1}{2}(1 + \epsilon)\right)\right) .
\end{align*}
Since this holds for all $\epsilon$,
\begin{align*}
 \liminf_{T \rightarrow \infty}  \frac{R_T(v)}{\log T}
\geq \lambda^2  c_f^2 \left(\frac{v(1-v)(v- b^*_{v,F})}{32} \right).
\end{align*}

\end{proof}
\section{Unknown $F$}\label{sec:app_unknown_F}
\subsection{Upper Bound of the Regret of O-UCBid1}
\begin{theorem}\label{th:optiBID}
O-UCBid1 incurs a regret bounded by 
$$R_T \leq\frac{4 \sqrt 2}{F(b^*)} \sqrt{ \gamma \log T} (\sqrt{T}+1) + O(1).$$
\end{theorem}

 We first observe that the  algorithm overbids ($B_t>b^*$) when $F$ and $v$ belong to their confidence regions $\mathbb{F}_t = \{\tilde{F}, \|F-\hat{F}_t\|\leq \epsilon_t \}$ and $\mathbb{V}_t = [v- \epsilon_t, v+\epsilon_t]$.

\begin{lemma}\label{lem:implication}The bid submitted by O-UCBid1 is an upper bound of $b^*$ when $\|\hat{U}_{t} - U\|_\infty \leq 2 \epsilon_{t}$.
$$\left\{~\|\hat{U}_{t} - U\|_\infty \leq 2 \epsilon_{t}~\right\} \text{ implies } b^* \leq B_t.$$
\end{lemma}
\begin{proof}
Let us pick $\underline{b}\in \argmax\hat{U}_{t}$.\\
$$
	\hat{U}_{t}(\underline{b}) -\hat{U}_{t}(b^*)= \hat{U}_{t}(\underline{b})-U(b^*) +U(b^*)-\hat{U}_{t}(b^*) \leq 4 \epsilon_{t}.
$$
We deduce that $\hat{U}_{t}(b^*)\geq \hat{U}_{t}(\underline{b}) - 4 \epsilon_{t} \geq \max\hat{U}_{t} - 4 \epsilon_{t}$. \\
Hence, $b^*\in \left\{  b\in [0,1], \hat{U}_{t}(b)\geq \max\hat{U}_{t} - 2 \epsilon_{t} \right\}.$
By definition of $B_t$, this yields $B_t\geq b^*$. 
\end{proof}
 
Next we observe that if $F$ and $v$ lie in their confidence regions $\mathbb{F}_{t}$ and $\mathbb{V}_{t}$, then $\|\hat{U}_{t} - U\|_{\infty}\leq 2 \epsilon_t $. (Recall that $\hat{U}_t(b)=(\hat{V}_t-b)\hat{F}_t(b)$.)
Indeed, we have 
\begin{align*}
	  \hat{U}_{t}(b) - U(b) &= (\hat{V}_{t}-b)\hat{F}_{t}(b) -(v-b) F(b)\\
	 &= (\hat{V}_{t}-v)F(b) + \hat{V}_t(\hat{F}_{t}(b) -F(b)) + b(F(b)-\hat{F}_b) \\
	 &= (\hat{V}_{t}-v)F(b) + (\hat{V}_{t}-b)(\hat{F}_{t}(b) -F(b))
\end{align*}
which yields
\begin{equation}\label{eq:U_norm}
	 |\hat{U}_{t}(b) - U(b)|  \leq |\hat{V}_{t}-v| + \|F(b) -\hat{F}_{t}(b)\|_\infty.
	\end{equation}
We then decompose the regret into 
\begin{align}
&E(R_T ) = \sum_{t=1}^T \E(U(b^*)-U(B_t))\notag \\
&\leq 1+ \sum_{t=}^T  \Po ( F \notin  \mathbb{F}_{t} \text{ or } v \notin \mathbb{V}_{t}) + \sum_{t=2}^T \E\left( S_t \1(B_t>b^*) \1(\|\hat{U}_{t}-U\|_\infty\leq 2 \epsilon_{t},~ F \in  \mathbb{F}_{t} ,~v \in \mathbb{V}_{t}\right) .\label{ineq:optiBID_FPUCBID}
\end{align}
The second term of the second hand side of Equation \ref{ineq:optiBID_FPUCBID} is easily bounded thanks to the concentration inequalities in Lemmas \ref{lem:concentration_V} and \ref{lem:concentration_F}. In fact, combining these latter lemmas yields the following bound.
\begin{lemma}\label{lem:first_term}
$$ \sum_{t=2}^T  \Po (F \notin  \mathbb{F}_{t} \text{ or } v \notin \mathbb{V}_{t})  \leq  2 \sum_{t=1}^T2 e \sqrt{\gamma}( \log t )t^{-\gamma}$$
\end{lemma}

We apply Lemma \ref{lem:conditional_exp} to bound the third term of the second hand side of Equation \ref{ineq:optiBID_FPUCBID} as follows:
\begin{multline}\label{eq:S_unknown_F}
\E\Big[] S_t \1(B_t>b^*) \1(\|\hat{U}_{t}-U\|_\infty\leq 2 \epsilon_{t}, ~ F \in  \mathbb{F}_{t} ,~v \in \mathbb{V}_{t})\Big] \\
\leq \frac{1}{F(b^*)}\E\Big[U(b^*)-U(B_t))\times\1(M_t \leq B_t)\1(\|U-\hat{U}_{t}\|_{\infty}\leq 2 \epsilon_{t}, ~ F \in  \mathbb{F}_{t} ,~v \in \mathbb{V}_{t})\1(B_t>b^*)\Big],
\end{multline}
because  $\1(B_t>b^*) \1(\|\hat{U}_{t}-U\|_\infty\leq 2 \epsilon_{t}, ~ F \in  \mathbb{F}_{t} ,~v \in \mathbb{V}_{t})$ is $\mathcal{F}_{t-1}$-measurable.
We then bound the deviation $(U(b^*)-U(B_t))\1(M_t\leq B_t)$ by $8\epsilon_t$ by using Lemma \ref{lem:dist_U_max}. 
\begin{lemma}\label{lem:secimplication} When applying the O-UCBid1 strategy, if $\|U - \hat{U}_{t}\|_{\infty}\leq 2 \epsilon_{t} $, then
 $$|U(B_t)-U(b^*)|\leq 8 \epsilon_{t}.$$
\end{lemma}
\begin{proof}
Assume $\|U - \hat{U}_{t}\|_{\infty}\leq 2 \epsilon_{t} $.
Note that
$\hat{U}_{t}(B_t) - \hat{U}_{t}(b^*)=  \hat{U}_{t}(B_t) - \hat{U}_{t}(\hat{b})+ \hat{U}_{t}(\hat{b}) - \hat{U}_{t}(b^*) $, where $\hat{b}= \max \argmax_{b \in [0,1]} (\hat{V}_{t}-b)\hat{F}_{t}(b)$.

By design , we have $\hat{U}_{t}(B_t) - \hat{U}_{t}(\hat{b})= - 2 \epsilon_{t} $.
Thanks to Lemma \ref{lem:dist_U_max}, and because $\|U - \hat{U}_{t}\|_{\infty}\leq 2 \epsilon_{t}$ we know that $0 \leq \hat{U}_{t}(\hat{b}) - \hat{U}_{t}(b^*) \leq  4 \epsilon_{t}$. 
This yields $|\hat{U}_{t}(B_t) - \hat{U}_{t}(b^*)|\leq 4 \epsilon_{t}$.

Finally $$|U(B_t)- U(b^*)|\leq 8 \epsilon_{t}.$$
\end{proof}

Then, by summing, we get

\begin{align*}
&\sum_{t=2}^T\E\Big[] S_t \1(B_t>b^*) \1(\|\hat{U}_{t}-U\|_\infty\leq 2 \epsilon_{t}, ~ F \in  \mathbb{F}_{t} ,~v \in \mathbb{V}_{t})\Big] \\
&\leq \sum_{t=2}^T \frac{1}{F(b^*)}\E\Big(U(b^*)-U(B_t))\times\1(M_t \leq B_t)\1(B_t>b^*) \1(\|U-\hat{U}_{t}\|_{\infty}\leq 2 \epsilon_{t}, ~ F \in  \mathbb{F}_{t} ,~v \in \mathbb{V}_{t})\Big)\\
&\leq \sum_{t=2}^T \frac{1}{F(b^*)}\E\Big[8 \epsilon_{t} \times\1(M_t \leq B_t)\1(\|U-\hat{U}_{t}\|_{\infty}\leq 2 \epsilon_{t})\1(B_t>b^*)\Big]\\
&\leq \sum_{t=2}^T \frac{1}{F(b^*)}\E\Big[8 \sqrt{\frac{\log T}{2 N_{t}}} \1(M_t \leq B_t)\Big]\\
&\leq\frac{1}{F(b^*)}4 \sqrt{2\log T}(\sqrt{T}+1),
\end{align*}
where the last inequality comes from Lemma \ref{lem:sum_1_n}.
Using  Equation \ref{ineq:optiBID_FPUCBID} and Lemma \ref{lem:first_term} yields 
$$R_T \leq\frac{1}{F(b^*)}4 \sqrt{2\log T}(\sqrt{T}+1)  + \sum_{t=2}^T2 e \sqrt{\gamma}( \log t )t^{-\gamma}.$$
Consequently, when $\gamma>1$,
$$R_T \leq\frac{1}{F(b^*)}4 \sqrt{2\log T}(\sqrt{T}+1)  + O(1).$$

\subsection{General Upper Bound of the Regret of UCBid1+}
We prove a slightly different version of Theorem \ref{th:V-opt} than that of the main paper.
\begin{repeatthm}{th:V-opt}
UCBid1+ incurs a regret bounded by
\begin{align*}R_T&\leq 12 \sqrt{\frac{\gamma \alpha }{F(b^*)}} \sqrt{\log T} \sqrt{T}  + O(\log T)\\
&\leq 12 \frac{1}{U(b^*)}\sqrt{v \gamma}  \sqrt{\log T} \sqrt{T}  + O(\log T)
,
\end{align*}
where $\alpha := \frac{v}{v-b^*} $, provided that $\gamma>2$.
\end{repeatthm}

\begin{proof}
We denote by $\mathcal{E}$ the event $\{ \forall t_0<t<T, \;|\hat{V}_{t}-v|\leq \epsilon_{t}, \| F- \hat{F}_{t}\|_{\infty} \leq \sqrt{\frac{\gamma \log (t-1)}{2(t-1)}} \}$, where $t_0:= \min(3,1+ 8\frac{  \gamma(\alpha+1)^2}{\alpha (F(b^*))^2} \log\left( 4\frac{  \gamma(\alpha+1)^2}{\alpha (F(b^*))^2}\right)).$

Using Lemmas \ref{lem:concentration_V} and \ref{lem:concentration_F}, this event happens with high probability, when $\gamma>2$.

\begin{lemma}
The probability of the complementary of $\mathcal{E}$ is bounded as follows
$$
	\Po\left(\mathcal{E}^C
		\right) 
		  \leq  4 e(\gamma -1) (\log T ) (T)^{1-\gamma}.
$$
provided that $\gamma>2.$
\end{lemma}
\begin{proof}

We have
\begin{align*} \Po
	\left(\exists t\in [t_0,T], ~
		(\hat{V}(N_{t})-v)^2 \geq \frac{\gamma \log (t-1)}{2 N_{t}}
		\right) &\leq\Po
	\left(\exists t\in [2,T], ~
		(\hat{V}(N_{t})-v)^2 \geq \frac{\gamma \log (t-1)}{2 N_{t}}
		\right) \\
	&	\leq \sum_{t=2}^T\Po 
	\left(
		(\hat{V}(N_{t})-v)^2 \geq \frac{\gamma \log (t-1)}{2 N_{t}},
		\right) \\
	 &\leq  \sum_{t=1}^{T} \ 2 e \log(t)  t^{-\gamma} \\
	& \leq  \int_{u=1}^{T} 2 e \log(t)  u^{-\gamma} du
	\\ 
	&  \leq  2 e(\gamma -1) \log(T) (T)^{1-\gamma},
\end{align*}
thanks to Lemma \ref{lem:concentration_V}.
Similarly,
\begin{align*} \Po
	\left(\exists t\in [t_0,T], ~\| F- \hat{F}\|_{\infty} \geq \sqrt{\frac{\gamma \log (t-1)}{2 N_{t}}}\right)
	&\leq\sum_{t=t_0}^T \Po \left(~\| F- \hat{F}\|_{\infty} \geq \sqrt{\frac{\gamma \log (t-1)}{2 N_{t}}}\right)\\
	&\leq 2 \sum_{t=t_0}^T t^{-\gamma}\\
	&\leq  \int_{u=2}^{T}  2 u^{-\gamma} du\\
	&\leq  2 (\gamma -1) (T)^{1-\gamma} 
\end{align*}
thanks to Lemma \ref{lem:concentration_F}.
\end{proof}

When $\mathcal{E}$ occurs, it is possible  to prove that $F(B_t)$ is lower-bounded by a positive constant as soon as $t$ is large enough. 
\begin{lemma}\label{lem:bound_F}
On $\mathcal{E}$, provided that $t>t_0:= \min \left(3,1+ 8\frac{  \gamma(\alpha+1)^2}{\alpha (F(b^*))^2} \log\left( 4\frac{  \gamma(\alpha+1)^2}{\alpha (F(b^*))^2}\right)\right) $, $F(B_t)$ is lower bounded  by
$$F(B_t)>\frac{F(b^*)}{2 \alpha},$$
where $\alpha = \frac{v}{v-b^*}$.
\end{lemma}
\begin{proof}
$b^*= \frac{\alpha -1 }{\alpha} v$.
Since we are on $\mathcal{E}$, 
$$b^*\leq \frac{\alpha-1}{\alpha}(\hat{V}_{t} + \epsilon_{t}).$$
Hence $$\hat{V}_{t} + \epsilon_{t} \leq \alpha(\hat{V}_{t} + \epsilon_{t} - b^*).$$
Since $B_t >0$, 
$$\hat{V}_{t} + \epsilon_{t} - B_t \leq \alpha(\hat{V}_{t} + \epsilon_{t} - b^*).$$
And $$\frac{\hat{V}_{t} + \epsilon_{t} - B_t}{\hat{V}_{t} + \epsilon_{t} - b^*}\leq  \alpha.$$
By definition of $B_t$,
$$(\hat{V}_{t} + \epsilon_{t} - B_t )\hat{F}_t(B_t)\geq (\hat{V}_{t} + \epsilon_{t} - b^*) \hat{F}_t(b^*) $$ which implies
$$ \hat{F}_t(B_t)\geq \frac{\hat{V}_{t} + \epsilon_{t} - b^*} {\hat{V}_{t} + \epsilon_{t} - B_t}\hat{F}_t(b^*) \geq \frac{1}{\alpha} \hat{F}_t(b^*) $$
Now, 
\begin{align*}
F(B_t)&\geq  \hat{F}_t(B_t) - \sqrt{\frac{\gamma \log (t-1)}{2 (t-1)}} \\&
\geq  \frac{1}{\alpha} \hat{F}_t(b^*)   - \sqrt{\frac{\gamma \log (t-1)}{2(t-1)}} \\
&\geq  \frac{1}{\alpha} F(b^*)   -\left( \frac{1}{\alpha} +1 \right) \sqrt{\frac{\gamma \log (t-1)}{2 (t-1)}},
\end{align*}
because we assume that we are on $\mathcal{E}$.
Note that if $t>t_0$,  then
$$\frac{4\gamma (\alpha +1)^2}{F(b^*)^2} <\frac{(t-1)}{\log (t-1)},$$
thanks to Lemma \ref{lem:log}, and
$$\left( \frac{1}{\alpha} +1\right) \sqrt{\frac{\gamma \log (t-1)}{2 (t-1)}}< \frac{1}{2\alpha} F(b^*),$$
so that
$$F(B_t) \geq \frac{ F(b^*)}{2\alpha},$$
which concludes the proof.
\end{proof}

\begin{lemma}\label{lem:rate_Nt}
$\forall t>t_0,$
$$\Po \left(N_{t}< \frac{1}{4\alpha }F(b^*) (t- t_0), \mathcal{E}  \right) \leq \exp \left(-\frac{2((\frac{1}{2\alpha}F(b^*))^2}{4}(t-t_0)\right).$$
\end{lemma}
\begin{proof}
Indeed if $t \geq t_0$,  then $N_{t}$ is larger than the sum $N'_{t}$ of $t-t_0$ samples from a Bernoulli distribution  with average $\frac{1}{2\alpha}F(b^*)$ , hence
the probability that $N_{t}<  \frac{1}{4\alpha }F(b^*) (t- t_0)$ intersected with 
$\mathcal{E}  $ can be bounded as follows.
\begin{align*} &\Po \left(N_{t}< \frac{1}{4\alpha }F(b^*) (t- t_0), \mathcal{E}  \right)\\
&\leq \Po \left(N'_{t}<+  \frac{1}{4\alpha }F(b^*) (t- t_0)  \right)\\
& \leq \Po \left( \frac{1}{2\alpha}F(b^*) (t- t_0) -(N'_{t} - t_0) > \frac{1}{4\alpha }F(b^*) (t- t_0)\right) \\
& \leq  \exp \left(-\frac{2((\frac{1}{2\alpha}F(b^*))^2}{4}(t- t_0)\right)\\
& \leq \exp \left(-\frac{2((\frac{1}{2\alpha}F(b^*))^2}{4}(t-t_0)\right),
\end{align*}
where we used Hoeffding's inequality for the third inequality.
\end{proof}

Finally, we can prove that the expected instantaneous regret conditioned on $B_t$ is bounded by a multiple of $\epsilon_{t}$.
\begin{lemma}\label{lem:eps}
$$U(B_t)-U(b^*)\leq 6 \epsilon_{t}$$
\end{lemma}
\begin{proof}

Thanks to Equation \ref{eq:U_norm}, we have 
$\|\hat{U}_{t}-U\|_{\infty}\leq \ 2 \epsilon_{t}.$
Very similarly we have 
$$\| U^{UCBid1+}_{t} - \hat{U}\|_{\infty} = \max_{b \in[0,1]}|\epsilon_{t}\hat{F}_{t}(b)|\\
\leq  \epsilon_{t},
$$
where $U^{UCBid1+}: b \mapsto (\hat{V}_{t} + \epsilon_{t} - b) \hat{F}_{t}(b)$.
Hence, $$\| U^{UCBid1+}_{t} - U\|_{\infty} \leq 3 \epsilon_{t}.$$

By Lemma \ref{lem:dist_U_max}, this yields  $$U(B_t)-U(b^*)\leq 6 \epsilon_{t}$$
\end{proof}

\paragraph{Proof of the Theorem}
We use the following decomposition
\begin{align*}
    R_T &\leq T \times \Po(\mathcal{E}^c) +\sum_{t=1}^T \E[S_t \1\{\mathcal{E}\}]\\
    &\leq T \times \Po(\mathcal{E}^c) + t_0 + \sum_{t=t_0}^T \E[S_t \1\{\mathcal{E}\}]
\end{align*}
Thanks to Lemma \ref{lem:bound_F}, and when $t>t_0$,
$F(B_t)\geq \frac{1}{2\alpha}F(b^*)$.
Using this, we get $N_{t}> \frac{1}{4\alpha}F(b^*)(t-t_0) , \forall t>t_0$ with high probability.

Thanks to Lemma \ref{lem:rate_Nt},
\begin{align*}
 \E[S_t \1\{\mathcal{E}\}]  &\leq  \exp \left(-\frac{2((\frac{1}{2\alpha}F(b^*))^2}{4}(t-t_0)\right) + \E\left[S_t \1\{N_{t} \geq \frac{1}{4\alpha}F(b^*) (t- t_0) \}\right]\\
 &\leq  \exp \left(-\frac{2((\frac{1}{2\alpha}F(b^*))^2}{4}(t-t_0)\right) + \E\left[6 \sqrt{\frac{ 4 \alpha \gamma \log T}{ F(b^*) (t- t_0)}}\1\{N_{t} \geq \frac{1}{4\alpha}F(b^*) (t- t_0) \}\right];
\end{align*}

By summing, 
\begin{align*}
 \sum_{t=t_0}^T\E[S_{t} \1\{\mathcal{E}\}] 
 &\leq \sum_{t=t_0}^T \exp \left(-\frac{2((\frac{1}{2\alpha}F(b^*))^2}{4}(t-t_0)\right) + \sum_{t=t_0}^T 6 \sqrt{\frac{4 \alpha \gamma \log T}{ F(b^*) (t- t_0)}}\\
 &\leq \frac{1}{1- \exp(-\frac{2(\frac{1}{2\alpha}F(b^*))^2}{4})}+  6 \sqrt{\frac{4 \alpha \gamma}{ F(b^*)}} \sqrt{\log T} \sqrt{T} \\
 &\leq \frac{4}{\frac{1}{2\alpha}F(b^*)} +  6 \sqrt{\frac{4 \alpha \gamma}{ F(b^*)}} \sqrt{\log T} (\sqrt{T}) ,
\end{align*}
where the last inequality comes from $1-\exp(-u)\geq 2/u$, for any positive $u$.
Using the decomposition of the regret yields 
\begin{align*}
 &R_T
  \leq t_0 + T \Po(\mathcal{E}^C) + 
\frac{4}{\frac{1}{2\alpha}F(b^*)}+  6 \sqrt{\frac{4 \alpha}{ F(b^*)}} \sqrt{\log T } \sqrt{T} 
\\
&\leq 4+ 8\frac{  \gamma(\alpha+1)^2}{\alpha (F(b^*))^2} \log\left( 4\frac{  \gamma(\alpha+1)^2}{\alpha (F(b^*))^2}\right) + 4 e(\gamma -1) \log T  (T)^{2-\gamma} +\frac{8 \alpha}{ F(b^*)} +  12 \sqrt{\frac{\alpha \gamma}{ F(b^*)}} \sqrt{\log T} \sqrt{T} \\
&\leq 4+ \frac{8\alpha}{ F(b^*)}+  8\frac{  \gamma(\alpha+1)^2}{\alpha (F(b^*))^2} \log\left( 4\frac{  \gamma(\alpha+1)^2}{\alpha (F(b^*))^2}\right) + 4 e(\gamma -1) \log T +  12 \sqrt{\frac{\alpha \gamma}{F(b^*)}} \sqrt{\log T} (\sqrt{T}),
\end{align*}
which concludes the proof.
\end{proof}

\subsection{Proof of an Intermediary Regret Rate under Assumptions \ref{ass:pseudo-mhr} and \ref{ass:lambda_pseudo-mhr}}
 In this section, we prove an easier version of Theorem \ref{th:fast_rate}.
 We will use lemmas of the previous subsection for this version as well as for the more complex version.
In particular we have already proven that $\mathcal{E} \cap \{N_t\geq \frac{1}{4 \alpha}F(b^*)t\}$, occurs with high probability.  
Under Assumptions \ref{ass:pseudo-mhr} and \ref{ass:lambda_pseudo-mhr} and on this event, we prove the following result.

\begin{lemma}\label{lem:init_UCBid1}
Under Assumptions  \ref{ass:pseudo-mhr} and \ref{ass:lambda_pseudo-mhr} and if $t>\max(t_0, t_1)$,
\begin{itemize}
\item $\|F- \hat{F}_t\|_{\infty} \leq \epsilon_t^+$ 
and $|v- \hat{V}_t| \leq  \epsilon_t^+$,
\item $ |U(b^*) - U(B_t)| \leq  6 \epsilon_t^+$
\item $|b^* - B_t| \leq \Delta$,
\item $|b^* - B_t| \leq 1/\sqrt{c_U} \sqrt{ 6 \epsilon_t^+}$.
\item  $ |U(b^*) - U(B_t)| \leq  C_U(b^*- B_t)^2$
\end{itemize} on
$\mathcal{E} \cap \{N_t\geq \frac{1}{4 \alpha }F(b^*)t\}$, where
$\begin{cases}t_0 = \min \left(3,1+ 8\frac{  \gamma(\alpha+1)^2}{\alpha (F(b^*))^2} \log\left( 4\frac{  \gamma(\alpha+1)^2}{\alpha (F(b^*))^2}\right)\right)\\
 t_1 = 2 \sqrt{C_u} \Delta^{1/4}  \frac{\gamma \alpha}{F(b^*)} \log T,\\
 \epsilon_t^+ = \sqrt{\frac{2 \alpha \gamma \log t}{ F(b^*)t}},\\
 c_U = c_f \frac{1}{4} U(b^*),\\
 C_U = \frac{C_f}{c_f} \lambda .
 \end{cases}$
\end{lemma}
\begin{proof}
On the event $\mathcal{E} \cap \{N_t\geq \frac{1}{4 \alpha}F(b^*)t\}$, $\|F- \hat{F}_t\|_{\infty} \leq \epsilon_t^+$ 
and $|v- \hat{V}_t| \leq  \epsilon_t^+$ where $\epsilon_t^+ = \sqrt{\frac{2 \alpha \gamma \log t}{ F(b^*)t}}$  from Lemmas, \ref{lem:concentration_V},\ref{lem:concentration_F} \ref{lem:rate_Nt} and 
$ |U(b^*) - U(B_t)| \leq 6 \epsilon_t \leq  6 \epsilon_t^+$  from Lemmas \ref{lem:eps} and \ref{lem:rate_Nt}. 

Under Assumptions \ref{ass:pseudo-mhr} and \ref{ass:lambda_pseudo-mhr}, we prove that after $t_1$, we have 
$|B_t - b^*| \leq \Delta $ on $\mathcal{E} \cap \{N_t\geq \frac{1}{4 \alpha}F(b^*)t\}$, so that we will be able to use the boundedness of the density after this time step.

When $F$ satisfies assumption  \ref{ass:pseudo-mhr}, $U$ is unimodal, as shown in the proof of Lemma \ref{lem:unique_max}, and so if $$U(b^*) - U(b) \leq \min(U(b^*)- U(b^*- \Delta),U(b^*)- U(b^*+ \Delta)) ,$$ then $$b \in [b^*- \Delta, b^* + \Delta].$$
It follows that if $$6 \epsilon_t^+ \leq \min(U(b^*)- U(b^* - \Delta),U(b^*)- U(b^* + \Delta)) $$
and therefore $6 \epsilon_t^+ \leq C_u \Delta $
where $C_u := \lambda C_f/ c_f$ (see Lemma 7), then  $$|b^* - B_t| \leq \Delta$$ on 
$\mathcal{E} \cap \{N_t\geq \frac{1}{4 \alpha}F(b^*)t\}$. Then, for all $t> 2 \sqrt{c_u} \Delta^{1/4}  \frac{\gamma \alpha}{F(b^*)} \log T := t_1$, we have $|B_t - b^*| \leq \Delta $ on $\mathcal{E} \cap \{N_t\geq \frac{1}{4 \alpha}F(b^*)t\}$.

Under Assumption \ref{ass:pseudo-mhr}, for any $q \in [0,1]$, $W_{v,F}(q^*_{v,F}) - W_{v,F}(q) \geq \frac{1}{4}(q^*_{v,F} - q)^2 W_{v,F}(q^*_{v,F}).$
We have $U = W \circ F$, so that if $t>t_1$, then $B_t \in [b^*- \Delta, b^* +\Delta]$ and $U(b^*) -U(B_t) \geq c_f \frac{1}{4}(b^* -B_t)^2 U(b*) := c_U (b^* -B_t)^2$.
In this case, we can also prove that  $|b^* - B_t| \leq 1/\sqrt{c_U} \sqrt{ 6\epsilon_t^+}$, under $\mathcal{E} \cap \{N_t\geq \frac{1}{4 \alpha}F(b^*)t\}$.
\end{proof}

\begin{proposition}\label{prop:one_iteration}
Under Assumptions  \ref{ass:pseudo-mhr} and \ref{ass:lambda_pseudo-mhr} and if $t>\max(t_0, t_1)$, $\delta_t< \Delta$ , $|B_t - b^*|\leq \delta_t$,
and $\epsilon_t^+\leq M \delta_t$,\\
Then $$|B_t - b^*|^2 \leq \frac{6}{c_U} \sqrt{\frac{ C_f \delta_t \log
\left( \frac{M e^2 t\sqrt{2t}}{2 c_f \eta^2}\right)
}{t}} + \frac{2\log(\frac{M t\sqrt{2t}}{2c_f \eta^2 })}{c_U t} + \frac{2}{c_U} (2 C_f +1) \delta_t \sqrt{\frac{2 \alpha \gamma\log T}{F(b^*)t}},$$ with probability $1- \eta$ on $\mathcal{E} \cap \{N_t\geq \frac{1}{4 \alpha}F(b^*)t\}$.
\end{proposition}

\begin{proof}
It is clear from Lemma \ref{lem:local_concentration_inequality} that 
$$\sup_{b^* - \delta_t \leq b \leq b^* + \delta_t}|\hat{F}_t(b) - F(b) - (\hat{F}_t(b^*)- F(b^*))| \leq 2 \sqrt{\frac{2 C_f \delta_t \log
\left( \frac{e \sqrt{t}}{\sqrt{2c_f \delta_t}\eta}\right)
}
{t}} + 2 \frac{\log(\frac{t}{2c_f \delta_t \eta^2 })}{6 t} := \beta_t,$$
with probability $1- \eta$.
We can also decompose $U(b)-U^{UCBid1+}_t(b) - (U^{UCBid1+}_t(b^*)- U(b^*))$ into
\begin{align*} 
&U(b)-U^{UCBid1+}_t(b) - (U^{UCBid1+}_t(b^*)- U(b^*))\\
&=(v-b)F(b) - (\hat{V_t}+ \epsilon_t-b)\hat{F}_t(b) - \left((v-b^*)F(b^*) - (\hat{V_t}+ \epsilon_t-b^*)\hat{F^*}_t(b)\right)\\
&= (v-b)F(b) - (v-b)\hat{F}_t(b) - \left((v-b^*)F(b^*) - (v-b^*)\hat{F^*}_t(b)\right) - (\hat{V_t}+ \epsilon_t - v) \left(\hat{F}_t(b) - \hat{F}_t(b^*)\right)\\
&= (v-b^*)\left(F(b) - \hat{F}_t(b) - \left(F(b^*) -\hat{F^*}_t(b)\right)\right) - (\hat{V_t}+ \epsilon_t - v) \left(\hat{F}_t(b) - \hat{F}_t(b^*)\right)\\
&~~ + (b^* - b) (\hat{F}(b) - \hat{F}_t(b))
\end{align*}
which in turn proves that 
\begin{align*}|U(b)-U^{UCBid1+}_t(b) - (U^{UCBid1+}_t(b^*)- U(b^*))| &\leq \beta_t+ 2 \epsilon_t |\hat{F}_t(b) - \hat{F}_t(b^*)| + \delta_t |\hat{F}_t(b) - \hat{F}_t(b)| \\
&\leq \beta_t+ 2 \epsilon_t^+ (C_f \delta_t + \beta_t) + \delta_t \epsilon_t^+ \\
& \leq \beta_t + 2 \epsilon_t^+ \beta_t + (2 C_f +1) \delta_t \epsilon_t^+\\
& \leq 3\beta_t + (2 C_f +1) \delta_t \epsilon_t^+ := \gamma_t,
\end{align*}
for all $b$ in $[b^* - \delta_t, b^* + \delta_t]$.\\

Now, we know that $U(b^*)- U(b)$ is lower bounded by $c_U(b^*-b)^2$, on this interval\\
and $\|U^{UCBid1+}_t(b)- U(b) + U^{UCBid1+}_t(b^*)- U(b^*) \|_{\infty} \leq \gamma_t $ on $[b^* - \delta_t, b^* + \delta_t]$. We call $G$ the shifted version of $U$ defined by $G(b) = U(b) + U^{UCBid1+}_t(b^*)- U(b^*)$. Its argmax is $b^*$ and $G(b^*)- G(b)$ is lower bounded by $c_U(b^*-b)^2$\\
then $c_U(B_t- b^*)^2\leq G(b^*)- G(B_t) \leq 2 \gamma_t $ (see Lemma \ref{lem:dist_U_max}).

Then , by definition of $\gamma_t$ and $\beta_t$:
\begin{align*} (B_t- b^*)^2 &\leq \frac{6}{c_U} \sqrt{\frac{ C_f \delta_t \log
\left( \frac{e^2 t}{2c_f \delta_t \eta^2}\right)
}{t}} + \frac{2\log(\frac{t}{2c_f \delta_t \eta^2 })}{ c_U t} + \frac{2}{c_U}(2 C_f +1) \delta_t \epsilon_t^+\\
&\leq \frac{6}{c_U} \sqrt{\frac{ C_f \delta_t \log
\left(M \frac{e^2 t}{2c_f \epsilon_t^+\eta^2}\right)
}{t}} + \frac{2\log(\frac{Mt}{2c_f \epsilon_t^+ \eta^2 })}{ c_U t} +\frac{2}{c_U} (2 C_f +1) \delta_t \epsilon_t^+\\
&\leq \frac{6}{c_U} \sqrt{\frac{ C_f \delta_t \log
\left( \frac{M e^2 t\sqrt{t}}{2 c_f \eta^2}\right)
}{t}} + \frac{2\log(\frac{M t\sqrt{t}}{2c_f \eta^2 })}{ c_U t} + \frac{2}{c_U} (2 C_f +1) \delta_t \sqrt{\frac{2 \alpha \gamma\log T}{F(b^*)t}}.
\end{align*}
where the last inequality stems from that fact that  $1/\epsilon_t^+ = \sqrt{\frac{F(b^*) t}{2\alpha \gamma \log t}}\leq \sqrt{t}$ since $\alpha,\gamma\geq 1$.
\end{proof}

\begin{theorem}\label{th:3_8}
Under Assumptions \ref{ass:pseudo-mhr} and \ref{ass:lambda_pseudo-mhr}, $$R_T \leq O(T^{3/8} \log T).$$
\end{theorem}

\begin{proof}

From Lemma \ref{lem:init_UCBid1}, we have that $|b^* - B_t| \leq 1/\sqrt{c_U} \sqrt{ 6 \epsilon_t^+}$, on $\mathcal{E}\cap \{N_t\geq \frac{1}{4 \alpha}F(b^*)t\}\}$. Therefore, we can apply Proposition \ref{prop:one_iteration} with $\delta_t = \frac{1}{\sqrt{c_U}} \sqrt{ 6\epsilon_t^+}$ with $M= \frac{\sqrt{c_U}}{\sqrt{6}}$, and $\eta= \frac{1}{t}$ .

We use the general fact that  $\log(At^{\alpha})\leq 2 \alpha \log t$ as soon as $t^{\alpha}>A$, for all $A,a>0$, to derive the following two inequalities :

$\forall t\geq \left(\frac{M e^2}{2c_f}\right)^{\frac{1}{4}}$, $$ \frac{6}{c_U} \sqrt{\frac{ C_f \delta_t \log
\left( \frac{M e^2 t\sqrt{t}}{2 c_f \eta^2}\right)
}{t}} \leq   \frac{6 \sqrt{8} \sqrt{C_f}}{c_U^{\frac{5}{4}}} \sqrt{\frac{\delta_t\log t}{t}}= \frac{24  (72\alpha \gamma)^{\frac{1}{8}} \sqrt{C_f}}{c_U^{\frac{5}{4}}  {F(b^*)}^{\frac{1}{8}}}  \sqrt{\frac{\log ^2 t }{t^{\frac{5}{4}}}}.$$

$\forall t\geq (\frac{M}{2 c_f})^{\frac{1}{4}}$, $$  \frac{2\log(\frac{Mt^2 t\sqrt{t}}{2c_f })}{ c_U t}\leq \frac{16}{c_U} \frac{\log t}{t} .$$

We also have, for all t, 
\begin{align*}
 \frac{2}{c_U} (2 C_f +1) \delta_t \sqrt{\frac{ 2 \gamma \alpha \log T}{F(b^*)t}} &\leq \frac{2}{c_U}(2 C_f + 1)\sqrt{\frac{2\gamma \alpha }{F(b^*)}} \delta_t \sqrt{\frac{\log t}{t}}\\
 & = \frac{2(72\alpha \gamma)^{\frac{1}{4}} }{c_U^{\frac{3}{2}} {F(b^*)}^{\frac{1}{4}}}(2 C_f + 1)\sqrt{\frac{2\gamma \alpha }{F(b^*)}}  \frac{(\log t)^{\frac{1}{4}}}{t^\frac{1}{4}} \sqrt{\frac{\log t}{t}} 
\end{align*}
Therefore $|B_t - b^*|^2 \leq   \left( \frac{24  (72\alpha \gamma)^{\frac{1}{8}} \sqrt{C_f}}{c_U^{\frac{5}{4}}  {F(b^*)^{\frac{1}{8}}}}+ \frac{16}{c_U} +\frac{2(72\alpha \gamma)^{\frac{1}{4}} }{c_U^{\frac{3}{2}} {F(b^*)}^{\frac{1}{4}}}(2 C_f + 1)\sqrt{\frac{2\gamma \alpha}{F(b^*)^{\frac{1}{8}}}} \right)
 \frac{\log  t }{t^{\frac{5}{8}}} $ 
with probability $1- \frac{1}{t}$,
 for $t \geq \max(\left(\frac{M e^2}{2c_f}\right)^{\frac{1}{4}}, (\frac{M}{2c_f})^{\frac{1}{4}}):= t_2$ 
 on $\mathcal{E} \cap \{N_t\geq \frac{1}{4 \alpha}F(b^*)t\}$.
 On this event, $U(b^*) - U(B_t) \leq C_U (b^* - B_t)^2$

 We use the following decomposition
\begin{align*}
    R_T &\leq T \times \Po(\mathcal{E}^c) +\sum_{t=1}^T \E[S_t \1\{\mathcal{E}\}]\\
    &\leq T \times \Po(\mathcal{E}^c) + \max(t_0,t_1, t_2) + \sum_{t=\max(t_0,t_1, t_2)}^T \E[S_t \1\{\mathcal{E}\}]\\
    &\leq T \times \Po(\mathcal{E}^c) + \max(t_0,t_1, t_2)  + \sum_{t=\max(t_0,t_1,t_2 )}^T \Po(\mathcal{E} \cap \{N_t< \frac{1}{4 \alpha}F(b^*)t\}) \\
    &~~~ +\sum_{t=\max(t_0,t_1, t_2 )}^T \E[S_t \1\{\mathcal{E}\cap \{N_t\geq \frac{1}{4 \alpha}F(b^*)t\}\}]\\
   &\leq  T \times \Po(\mathcal{E}^c) + \max(t_0,t_1,t_2 )  + \sum_{t=\max(t_0,t_1, t_2 )}^T \Po(\mathcal{E} \cap \{N_t< \frac{1}{4 \alpha}F(b^*)t\}) \\
    &~~~ + \sum_{t=\max(t_0,t_1, t_2 )}^T C_U \E[(b^* - B_t)^2]\\
    &\leq  T \times \Po(\mathcal{E}^c) + \max(t_0,t_1, t_2 )  + \sum_{t=\max(t_0,t_1,t_2 )}^T \Po(\mathcal{E} \cap \{N_t< \frac{1}{4 \alpha}F(b^*)t\}) \\
    &~~~ + \sum_{t=\max(t_0,t_1, t_2 )}^T C_0 \frac{\log t }{t^{\frac{5}{8}}} + \sum_{t=\max(t_0,t_1, t_2 )}^T  \frac{1}{t}
    \\
    &\leq  T \times \Po(\mathcal{E}^c) + \max(t_0,t_1, t_2 )  + \sum_{t=\max(t_0,t_1, t_2)}^T \Po(\mathcal{E} \cap \{N_t< \frac{1}{4 \alpha}F(b^*)t\}) \\
    &~~~ +C_0 \frac{8}{3} T^{\frac{3}{8}} \log T + \log T \\
    &\leq  \log T+ 4 e(\gamma -1) \log T  (T)^{2-\gamma} +  \max(t_0,t_1, t_2 )+  \frac{8 \alpha}{ F(b^*)} +  \frac{8}{3}C_0 T^{\frac{3}{8}} \log T.
\end{align*}
 where
$\begin{cases}t_0 = \min \left(3,1+ 8\frac{  \gamma(\alpha+1)^2}{\alpha (F(b^*))^2} \log\left( 4\frac{  \gamma(\alpha+1)^2}{\alpha (F(b^*))^2}\right)\right)\\
 t_1 = 2 \sqrt{C_u} \Delta^{1/4}  \frac{\gamma \alpha}{F(b^*)} \log T,\\
 t_2 = \max(\left(\frac{\sqrt{c_U} e^2}{2c_f\sqrt{6}}\right)^{\frac{1}{4}}, (\frac{\sqrt{c_U}}{2c_f\sqrt{6}})^{\frac{1}{4}}) = (\frac{\sqrt{c_U} e^2}{2c_f\sqrt{6}})^{\frac{1}{4}},\\
C_0 =\left( \frac{24  (72\alpha \gamma)^{\frac{1}{8}} \sqrt{C_f}}{c_U^{\frac{5}{4}}  {F(b^*)^{\frac{1}{8}}}}+ \frac{16}{c_U} +\frac{2(72\alpha \gamma)^{\frac{1}{4}} }{c_U^{\frac{3}{2}} {F(b^*)}^{\frac{1}{4}}}(2 C_f + 1)\sqrt{\frac{2\gamma \alpha }{F(b^*)}} \right)C_U.
 \end{cases}$

Therefore $$R_T \leq \frac{8}{3}C_0 T^{\frac{3}{8}}\log T + o(T^{\frac{3}{8}} \log T).$$

\end{proof}

\subsection{Proof of Theorem \ref{th:fast_rate}}
Theorem \ref{th:3_8} is proved by applying Proposition \ref{prop:one_iteration} once.  By iterating the argument, we can actually achieve a regret of the order of $T^a$, for any $a>\frac{1}{3}$.
The proof involves an induction argument. The following lemma is the main element of the proof of the induction.

\begin{lemma}\label{lem:one_iteration} Assume that $t$ and $F$ satisfy the assumptions of Proposition \ref{prop:one_iteration}.
Assume that $|B_t - b^*|$ is bounded by $\delta_t^{(k)}$ such that $\delta_t^{(k)}= \min(1,C^{(k)}\log(t) t^{- u_k})$ with probability $1-\eta^{(k)}$,  and $u_k <2/3$, $C^{(k)} \geq 1$
.
Then $|B_t - b^*|$ is bounded by $\delta_t^{(k+1)}$ such that $\delta_t^{(k+1)}=\min(1, C^{(k+1)}\log(t) t^{-\frac{1}{4}(1+u_{k})})$ with probability $1-\eta^{(k)}- \frac{1}{Kt}$,\\where $C^{(k+1)} = C \left(C^{(k)}\right)^{\frac{1}{4}}$  and where $C=\max\left(1, \frac{12 \sqrt {2 C_f}}{c_U}+ \frac{16}{c_U}+\frac{2}{c_U}(2 C_f + 1)\sqrt{\frac{2\gamma\alpha}{F(b^*}}\right).$
\end{lemma}
\begin{proof}

We use Proposition \ref{prop:one_iteration}, and the fact that $\epsilon_t^+ \leq \sqrt{2 \alpha  \gamma / F(b^*) }\frac{\log t}{t^{-u_k}}  \leq \sqrt{2 \alpha  \gamma / F(b^*) }\delta^{(k)}_t$ to prove that 
$$|B_t - b^*|^2\leq \frac{6}{c_U} \sqrt{\frac{ C_f \delta_t^{(k)} \log
\left( \frac{M e^2 t\sqrt{2t} K^2 t^2}{2 c_f }\right)
}{t}}
 + \frac{2\log(\frac{MK^2 t^2 t\sqrt{2t}}{2c_f })}{ c_U t} + \frac{2}{c_U} (2 C_f +1) \delta_t^{(k)} \sqrt{\frac{2\alpha \gamma\log t}{F(b^*)t}} ,$$ with probability $(1-\eta^{(k)})(1- \frac{1}{Kt})$ and with $M = \sqrt{2 \alpha  \gamma / F(b^*) }$\\

We use the general fact that  $\log(At^{\alpha})\leq 2 \alpha \log t$ as soon as $t^{\alpha}>A$, for all $A,a>0$, to derive the following two inequalities :

$\forall t\geq \left(\frac{M e^2 K^2}{2c_f}\right)^{\frac{1}{4}}$, $$ \frac{6}{c_U} \sqrt{\frac{ C_f \delta_t^{(k)} \log
\left( \frac{M e^2 t\sqrt{2t} K^2 t^2}{2 c_f }\right)
}{t}} \leq   \frac{6 \sqrt{8 C_f}}{c_U} \sqrt{\frac{\delta_t^{(k)}\log t}{t}}:= C_1\sqrt{\frac{\delta_t^{(k)}\log t}{t}}:= C_1 \beta_{1,t}.$$

$\forall t\geq (\frac{MK^2}{2c_f})^{\frac{1}{4}}$, $$  \frac{2\log(\frac{MK^2 t^2 t\sqrt{2t}}{2c_f })}{ c_U t}\leq \frac{16}{ c_U} \frac{\log t}{t} := C_2 \frac{\log t}{t} := C_2  \beta_{2,t} .$$
We also have, for all t, $$ \frac{2}{c_U} (2 C_f +1) \delta_t^{(k)} \sqrt{\frac{2\alpha \gamma\log T}{F(b^*)t}} \leq \frac{2}{c_U}(2 C_f + 1)\sqrt{\frac{2\gamma\alpha}{F(b^*}} \delta_t^{(k)}\sqrt{\frac{\log t}{t}} := C_3 \delta_t^{(k)}\sqrt{\frac{\log t}{t}}:= C_3 \beta_{3,t} $$
We can derive the following bounds
\begin{itemize}
\item $\beta_{3,t} \leq \beta_{1,t}$ since $\delta_t^{(k)}\leq 1$.
\item $\beta_{2,t} \leq \beta_{1,t}$ since $\delta_t^{(k)}= \min(1,C^{(k)}\log(t) t^{- u_k})\geq \frac{\log t}{t}$.
\end{itemize}
Hence 

$$|B_t - b^*|^2\leq (C_1+ C_2+C_3) \beta_{1,t}= (C_1+ C_2+C_3)  \sqrt{\frac{\delta_t^{(k)}\log t}{t}},$$ with probability $1-\eta^{(k)} \frac{1}{Kt}.$
This yields \begin{align*}|B_t - b^*|&\leq \sqrt{(C_1+ C_2+C_3)}  \left(\frac{\delta_t^{(k)}\log t}{t}\right)^{\frac{1}{4}}\\
&\leq \sqrt{(C_1+ C_2+C_3)}  \left(\frac{\min(1,C^{(k)}\log^2(t) t^{- u_k})}{t}\right)^{\frac{1}{4}}\\
&\leq  \sqrt{(C_1+ C_2+C_3)} (C^{(k)})^{1/4} t^{-\frac{1}{4}(1+u_k)} \log t\\
&\leq C \left(C^{(k)}\right)^{1/4} t^{-\frac{1}{4}(1+u_k)}  \log t,
\end{align*}

\end{proof}

\begin{proposition}\label{prop:mult_iteration} Assume that $t$ and $F$ satisfy the assumptions of Proposition \ref{prop:one_iteration}.
If $t>t_3= max \left((\frac {\sqrt{2 \alpha  \gamma / F(b^*) }K^2}{2c_f})^{\frac{1}{4}} ,(\frac{\sqrt{2 \alpha  \gamma / F(b^*) } e^2 K^2}{2c_f})^{\frac{1}{4}} \right)$, then on $\mathcal{E} \cap \{N_t\geq \frac{1}{4 \alpha}F(b^*)t\}$,
$$|B_t - b^*|\leq   C^{(0)} C^{\frac{1}{3}} \log(t) t^{-\frac{1}{3}+ \frac{1}{3 \times4^K} + \frac{1}{4^{K+1}}},$$
with probability $1-\frac{1}{t}$ where $C=\max\left(1, \frac{12 \sqrt{2 C_f}}{c_U}+ \frac{16
}{c_U}+\frac{2}{c_U}(2 C_f + 1)\sqrt{\frac{2 \gamma \alpha}{F(b^*)}}\right)$, and $C^{(0)} =\max \left(1,\sqrt{\frac{1}{c_U}}\left(\frac{72\gamma \alpha }{F(b^*) }\right)^{\frac{1}{4}} \right)$
\end{proposition}
\begin{proof}
The proposition follows from using an induction argument based on Lemma \ref{lem:one_iteration}. We  can initiate an induction argument  with  $\delta_t^{(0)}$ such that $$\delta_t^{(0)}= \min(1,C^{(0)}\log(t) t^{- u_k}),$$ writing $u_0= \frac{1}{4}$ and $C^{(0)} = \max(1,\sqrt{\frac{1}{c_U}}\left(\frac{72 \alpha \gamma}{ F(b^*)}\right)^{1/4} )$, thanks to Lemma \ref{lem:init_UCBid1}. The fact that $u_k$ and $C^{(k)}$ as defined as in Lemma \ref{lem:one_iteration} satisfy
$u_{k+1} = \frac{1}{4}(1+u_k)$ which yields $$u_K = \left(\frac{1}{4}\right)^K u_0 + \sum_{i=1}^K \frac{1}{4^i} =  \left(\frac{1}{4}\right)^K u_0 + 4\frac{1/4- (1/4)^{K+1}}{3}$$
and $C^{(k+1)} = C \times (C^{(k)})^{\frac{1}{4}}$ which yields $$C^{(K)}= \left(C^{(0)}\right)^{\frac{1}{4^K}} C^{\sum_{i=1}^K \frac{1}{4^i}} \leq C^{\frac{1}{3}}, $$ suffices to complete the induction.

\end{proof}

We recall Theorem \ref{th:fast_rate}.
\begin{repeatthm}{th:fast_rate}
Under Assumptions \ref{ass:pseudo-mhr} and \ref{ass:lambda_pseudo-mhr}, $$R_T \leq O(T^{1/3 + \epsilon} ),$$
for any $\epsilon>0$
as long as $\gamma >2$.
\end{repeatthm}

We choose $K$ such that $\frac{1}{3}+ \frac{2}{3 \times 4^K} + \frac{2}{4^{K+1}} < \frac{1}{3} + \epsilon $. (We can choose $K = \big \lceil \log_4 \left(\frac{3}{14} \frac{1}{\epsilon}\right) \big\rceil) +1$ for example).
Then, thanks to proposition \ref{prop:mult_iteration}, 
for all $t> t_3$,
on $\mathcal{E} \cap \{N_t\geq \frac{1}{4 \alpha}F(b^*)t\}$, $$|B_t - b^*|\leq   C^{(0)} C^{\frac{1}{3}} \log(t) t^{-\frac{1}{3}+ \frac{1}{3 \times4^K} + \frac{1}{4^{K+1}}},$$ with probability $1-\frac{1}{t}$.
We can therefore do the same decomposition as in the proof of Theorem \ref{th:3_8}. 

\begin{align*}
    R_T 
    &\leq T \times \Po(\mathcal{E}^c) + \max(t_0,t_1, t_3)  + \sum_{t=\max(t_0,t_1,t_3 )}^T \Po(\mathcal{E} \cap \{N_t< \frac{1}{4 \alpha}F(b^*)t\}) \\
    &~~~ +\sum_{t=\max(t_0,t_1, t_3 )}^T \E[S_t \1\{\mathcal{E}\cap \{N_t\geq \frac{1}{4 \alpha}F(b^*)t\}\}]\\
   &\leq  T \times \Po(\mathcal{E}^c) + \max(t_0,t_1,t_3 )  + \sum_{t=\max(t_0,t_1, t_3 )}^T \Po(\mathcal{E} \cap \{N_t< \frac{1}{4 \alpha}F(b^*)t\}) \\
    &~~~ + \sum_{t=\max(t_0,t_1, t_3 )}^T C_U \E[(b^* - B_t)^2]\\
    &\leq  T \times \Po(\mathcal{E}^c) + \max(t_0,t_1, t_3)  + \sum_{t=\max(t_0,t_1,t_3 )}^T \Po(\mathcal{E} \cap \{N_t< \frac{1}{4 \alpha}F(b^*)t\}) \\
    &~~~ + \sum_{t=\max(t_0,t_1, t_3 )}^T  C^{(0)} C_U C^{\frac{1}{3}}(\log t) t^{-\frac{2}{3}+ \frac{2}{3 \times 4^K} + \frac{2}{4^{K+1}}} \\
    & + \sum_{t=\max(t_0,t_1, t_3 )}^T  \frac{1}{t}
    \\
    &\leq  T \times \Po(\mathcal{E}^c) + \max(t_0,t_1, t_3 )  + \sum_{t=\max(t_0,t_1, t_3)}^T \Po(\mathcal{E} \cap \{N_t< \frac{1}{4 \alpha}F(b^*)t\}) \\
    &~~~ + C^{(0)}C_U C^{\frac{1}{3}} \frac{1}{\frac{1}{3}+ \frac{2}{3 \times4^K} + \frac{2}{4^{K+1}}} T^{\frac{1}{3}+ \frac{2}{3 \times 4^K} + \frac{2}{4^{K+1}}} \log T + \log T \\
    &\leq  \log T + 4 e(\gamma -1) \log T  (T)^{2-\gamma} +  \max(t_0,t_1, t_3 )\\
    &~~~+  \frac{8 \alpha}{ F(b^*)} +   3 C^{(0)} C_U C^{\frac{1}{3}} T^{\frac{1}{3}+ \epsilon}.
\end{align*}
where
$\begin{cases}t_0 = \min \left(3,1+ 8\frac{  \gamma(\alpha+1)^2}{\alpha (F(b^*))^2} \log\left( 4\frac{  \gamma(\alpha+1)^2}{\alpha (F(b^*))^2}\right)\right)\\
 t_1 = 2 \sqrt{c_u} \Delta^{1/4}  \frac{\gamma \alpha}{F(b^*)} \log T,\\
 t_3= (\frac{\sqrt{2 \alpha  \gamma / F(b^*) }e^2 K^2}{2c_f})^{\frac{1}{4}}  ,\\
 C^{(0)}= \max \left(1,\sqrt{\frac{1}{c_U}}\left(\frac{72\gamma \alpha }{F(b^*) }\right)^{\frac{1}{4}} \right)\\
C=\max\left(1, \frac{12 \sqrt{ 2C_f}}{c_U}+ \frac{16}{c_U}+\frac{2}{c_U}(2 C_f + 1)\sqrt{\frac{2\gamma\alpha}{F(b^*)}}\right)
 \end{cases}$.
 
 Hence $$R_T  \leq O(T^{1/3 + \epsilon}).$$

\section{Further figures}\label{sec:app_expe}
We present in Figure \ref{fig:rw_histo} the histogram of the normalized data used to simulate the real-world experiment.
\begin{figure}[ht]
\centering
  \includegraphics[width=0.55\textwidth]{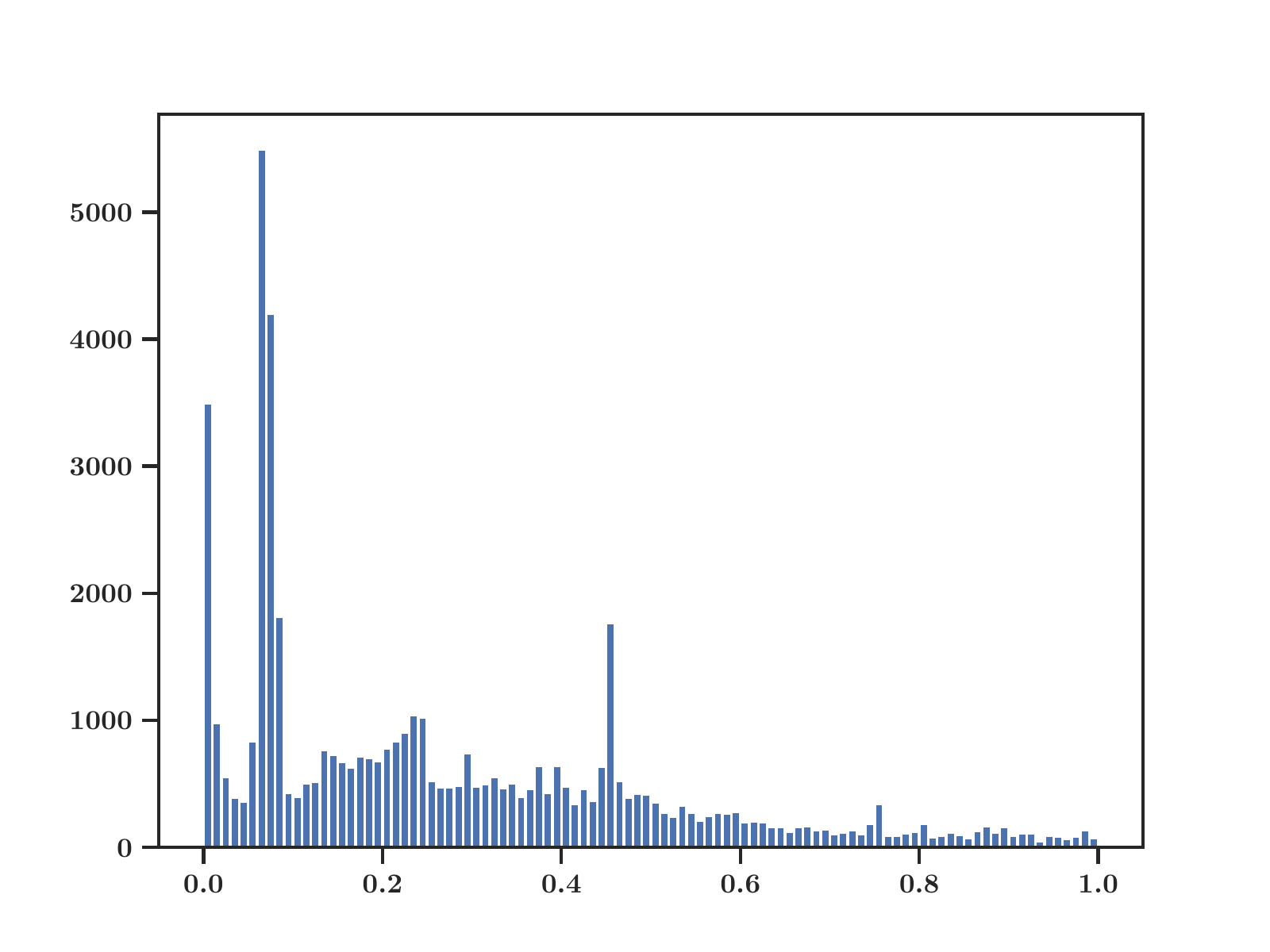}
  \caption{Bidding Data histogram}
  \label{fig:rw_histo}
\end{figure}